\documentclass{article}

\usepackage[preprint]{myarxivstyle}

\usepackage{bm}
\usepackage[utf8]{inputenc} 
\usepackage[T1]{fontenc}    
\usepackage{hyperref}       
\usepackage{url}            
\usepackage{booktabs}       
\usepackage{amsfonts}       
\usepackage{nicefrac}       
\usepackage{microtype}      
\usepackage{xcolor}         

\usepackage{subfigure}
\usepackage{graphicx} 
\usepackage{amsmath}
\usepackage{xcolor}
\usepackage{comment}

\usepackage{hyperref}

\usepackage{amsthm}

\usepackage{wrapfig}

\usepackage{enumitem}

\theoremstyle{plain}
\newtheorem{theorem}{Theorem}[section]

\newtheorem{lemma}[theorem]{Lemma}
\newtheorem{corollary}[theorem]{Corollary}
\theoremstyle{definition}
\newtheorem{definition}[theorem]{Definition}
\newtheorem{example}[theorem]{Example}
\newtheorem{assumption}[theorem]{Assumption}
\theoremstyle{remark}
\newtheorem{remark}[theorem]{Remark}

\usepackage{natbib}

\usepackage{thmtools} 

\makeatletter
\newcommand{\listofappendices}{%
  \section*{Appendix Contents}%
  \@starttoc{apx}%
}
\makeatother

\title{Implicit Bias of Mirror Flow in Homogeneous Neural Networks: Sparse and Dense Feature Learning}

\author{
  Tom Jacobs \\
  CISPA Helmholtz Center \\
  \texttt{tom.jacobs@cispa.de} \\
  \And
  Guido Mont\'ufar\\
  UCLA and MPI MiS\\ 
  \texttt{montufar@math.ucla.edu}%
}

\begin{document}

\maketitle

\begin{abstract}
    We study the max-margin solutions reached by mirror flow in deep neural networks with homogeneous activation functions. 
    Extending classical results on gradient flow, we derive a novel balance equation for mirror flow from convex duality, enabling a characterization of the horizon function governing the induced margin. 
    We further establish max-margin characterizations together with convergence rates and norm growth estimates. 
    Finally, we support our theory through experiments on synthetic datasets and standard vision tasks. 
    Concretely, we show that: (1) distinct non-homogeneous mirror maps can induce the same max-margin solution; 
    (2) convergence can be extremely slow, including exponentially slow regimes; and 
    (3) although all considered mirror maps exhibit feature learning, they can produce markedly different representations, ranging from sparse to dense neuron activations. 
    Together, these results provide a unified perspective on sparse and dense feature learning in homogeneous neural networks, highlighting how mirror maps shape both optimization dynamics and the geometry of the learned classifiers. 
\end{abstract}

\section{Introduction} 
Understanding why gradient-based optimization can generalize well is one of the central questions in deep learning theory. 
A key insight from the implicit regularization literature is that gradient descent training does not simply return a minimizer of the loss function but instead is implicitly biased towards particular types of minimizers. 
One of the most general results presently available in this context states that for homogeneous neural network classifiers trained on data that is separable by the network, 
gradient flow approaches an $L_2$ max-margin classifier 
\citep{Lyu2020Gradient}. 
This result connects the optimization dynamics induced by the algorithm and the network parameterization to the geometric properties of the resulting decision boundary, which directly influence the generalization behavior.

However, gradient flow is only one algorithm within a much broader family of optimization algorithms. 
Mirror flows generalize gradient flows by replacing the Euclidean geometry of the search space with the geometry induced by a strictly convex mirror potential $R$. The resulting parameter updates take place in the dual space induced by the mirror map $\nabla R$. 
Mirror flows arise naturally in several contexts. In particular, they arise as reparameterized gradient flows of overparameterized models \citep{Li2022ImplicitBO, jacobs2025mirror} and they are used to implement structural parameter constraints such as sparsity or non-negativity. 
In spite of the importance of these methods, a unified theory of implicit regularization for mirror flows on neural networks is still missing.

Prior work on implicit regularization of mirror flows has largely focused on linear models. 
In this context, \citep{Sun2023AUA} characterized the max-margin for homogeneous mirror maps, and more recently, \citep{pesme2024implicit} showed how to deal with non-homogeneous separable mirror maps by introducing a horizon function that captures the limiting direction of the iterates. 
This is particularly useful for mirror maps like the hyperbolic entropy used for sparse training \citep{jacobs2025mirror, jacobs2025mirror2}. 
To extend this line of results to the case of non-linear homogeneous networks, we identify and resolve two key difficulties: 
(i) defining an appropriate notion of margin, 
and (ii) establishing convergence of the iterates in direction to a KKT point of the corresponding max-margin. 
Note that for homogeneous networks with an exponential-tail loss, driving the loss to zero requires the parameters to diverge. 
One therefore cannot study convergence of the parameters directly, but rather convergence in direction.

\begin{figure}
    \centering
\includegraphics[scale=.28]{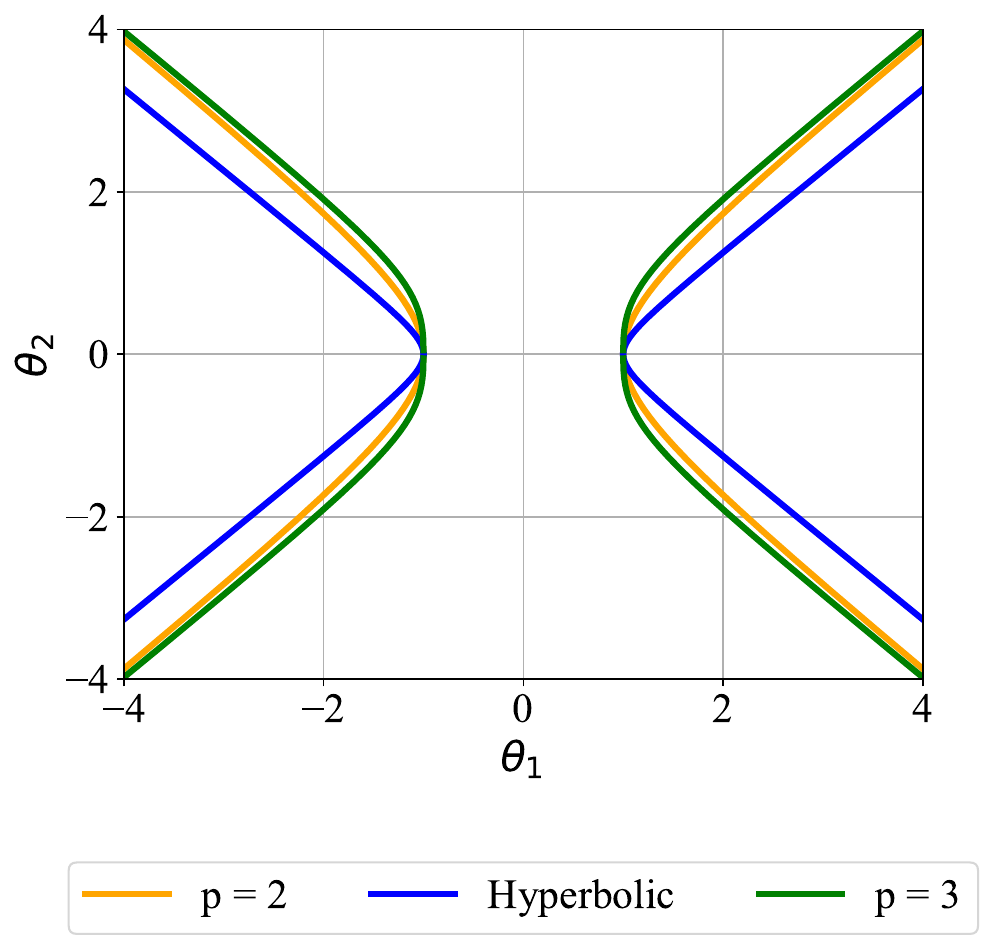}
\qquad 
\includegraphics[scale=.28]{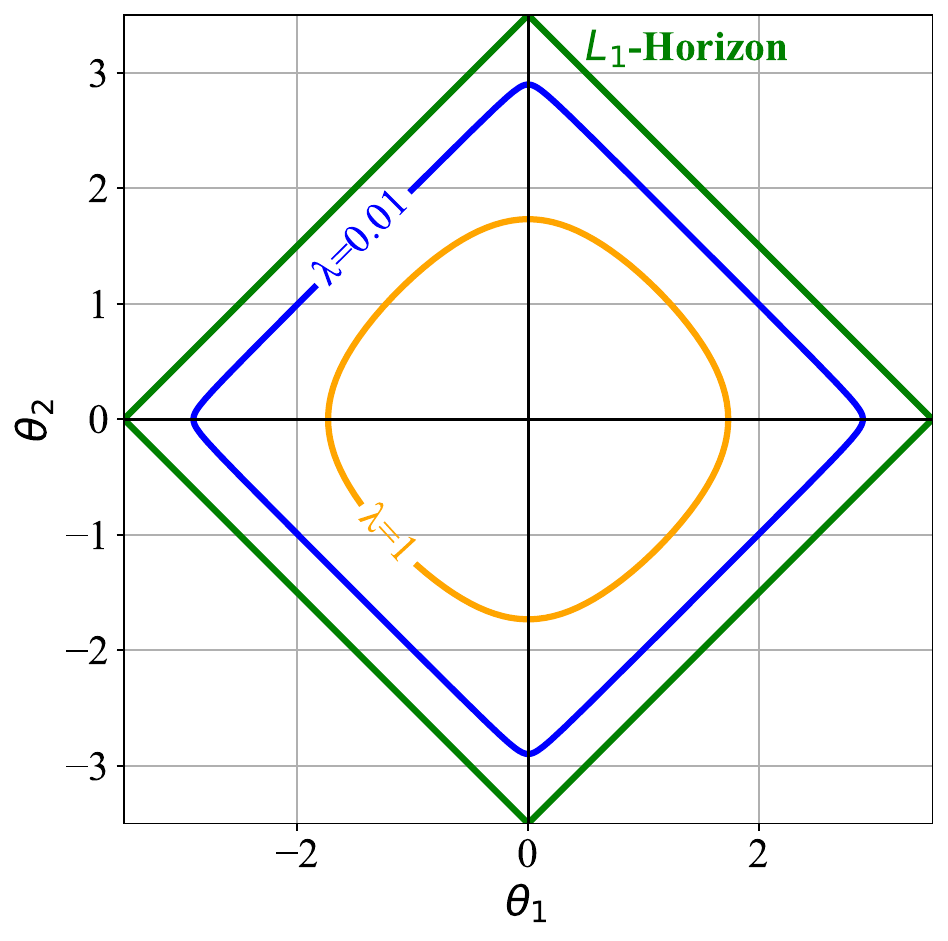}  
    \caption{
    Left: Balance equations. Shown are the solution sets of the balance equations for distinct mirror maps, characterized by $Q(\nabla R)$.
    Right: Horizon functions. 
    Shown are the level sets $Q(\nabla R)=c$ for the hyperbolic entropy with different values of the hyperparameter $\lambda$. 
    Although all choices of $\lambda$ share the same horizon function, given by the $L_1$-norm,  
    for larger $\lambda$ substantially larger parameter magnitudes $\|\theta\|$ are required for these level sets to approximate the limiting shape.
    }
\label{fig:new_balance_equation}
\label{fig:horizon apporox}
\end{figure}

Our starting point is the identification 
of a new \emph{balance equation} for mirror flow, illustrated in Figure~\ref{fig:new_balance_equation}. 
This provides an analog of the layer-wise norm balance equation well known for gradient flow. 
Building on this, 

we introduce the \emph{$Q$-margin} as a suitable notion of margin for mirror flows, defined via the dual potential $Q(\nabla R(\theta))$, where $Q := R^*$ denotes the 
convex conjugate of $R$.
This is the direct mirror analog of the $L_2$-margin for gradient descent. 
Using these ingredients, 
we show that mirror flow converges in direction to a KKT point of a constrained optimization problem whose objective is the square of a \emph{horizon function} of the form 
$\phi_{\alpha}(\theta) := \lim_{\eta\rightarrow0} \eta(\alpha Q (\nabla R(\theta/\eta)))^{1/\alpha}$, corresponding to a max-margin solution in the induced geometry. 
The horizon function 
captures the asymptotic geometry of the mirror potential at large parameter scales and governs the limiting direction of the dynamics. 
Here, $\alpha$ is the asymptotic homogeneity degree of the mirror potential. 
In the homogeneous case $R(\theta) = \frac{1}{p}\|\theta\|_p^p$, this reduces to $\|\theta\|_p$, with $\alpha=p$.

Next, to gain further insight into the role of non-homogeneity in the potential function, 
we consider a family of potentials with a hyperparameter $\lambda \geq 0$ which controls the level of non-homogeneity. 
We show that although distinct non-homogeneous mirror potentials can induce the same horizon function, and thus converge to the same max-margin solution, 
they can have vastly different convergence rates, in some cases requiring exponentially longer time. 
Figure~\ref{fig:horizon apporox} illustrates that larger values of $\lambda$ require disproportionately larger parameter magnitudes for the horizon function to closely approximate its asymptotic limit. 
These results provide concrete guidance for selecting the degree of inhomogeneity to balance convergence rate and other factors such as training stability. 

Finally, we show that mirror descent induces different feature representations through its geometry. In Theorem~\ref{theorem : reform KKT main}, we show for a two-layer network, hyperbolic entropy leads to fewer active neurons, whereas a smoothed homogeneous mirror potential activates more, resulting in sparse versus dense feature learning (Figure~\ref{fig: feature learning}).

\begin{figure}[ht]
    \centering
    \includegraphics[width=0.95\linewidth]{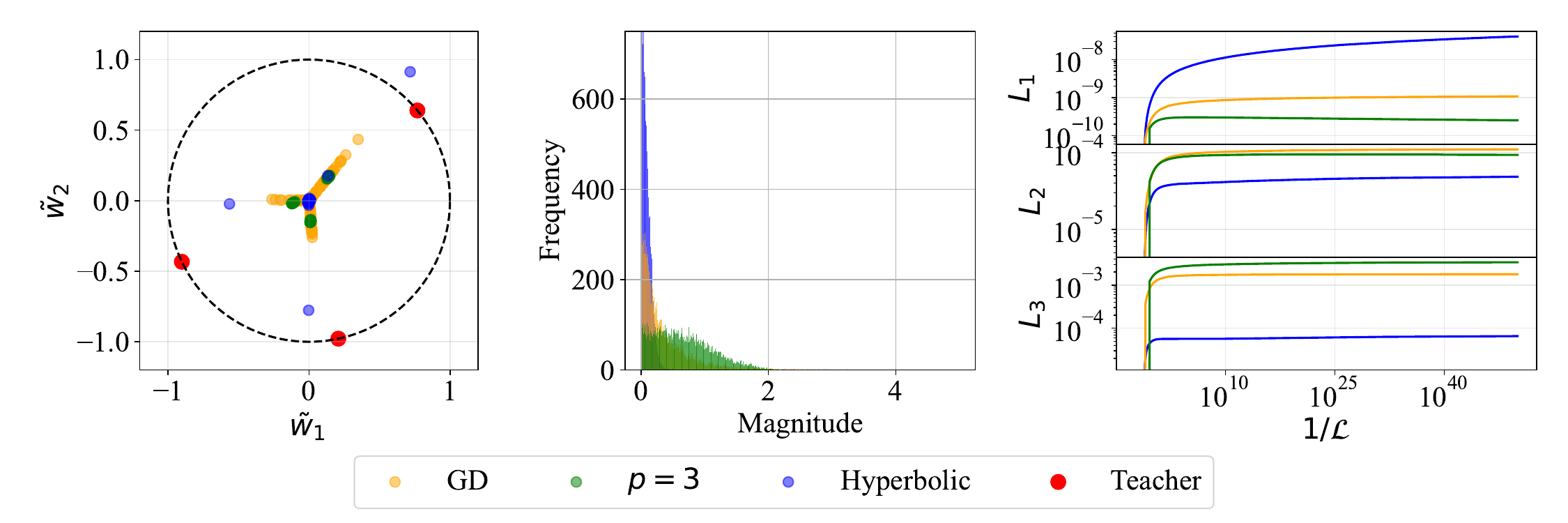}
    \caption{
   \textbf{Different types of feature learning under mirror flows.} The hyperbolic entropy induces sparse feature learning with fewer active neurons, whereas the smoothed homogeneous potential ($p =3$) induces dense feature learning with more active neurons. \textbf{(Left)} Input weight representations $\tilde{w}=|a_j|w_j$ of a two-layer student network in a student-teacher setup: GD (orange) produces diffuse weights near the origin, $p =3$ (green) learns teacher directions with all active neurons, and the hyperbolic map (blue) concentrates mass onto only a few neurons aligned with the teacher (red), reflecting sparsity. \textbf{(Middle)} Weight magnitude distribution of the hidden layer in a three-layer network: GD and $p=3$ yield spread out distributions with many active weights, while the hyperbolic entropy produces a sharply peaked distribution near zero with few large-magnitude weights. \textbf{(Right)} The $L_k$-margin, for $k = 1,2,3$ as a function of $1/\mathcal{L}$ for a three-layer network: each mirror flow converges in direction to their corresponding max-margin.}
    \label{fig: feature learning}
\end{figure}

\paragraph{Main contributions.}
In this work, we characterize the implicit bias of mirror flow in training homogeneous neural network classifiers, with the following contributions:
\begin{itemize}[leftmargin=*]
    \item \textbf{Balance equation.} 
    For layerwise separable mirror maps, we derive a novel balance equation, a conserved quantity of the mirror flow (Lemma~\ref{lemma : invariance balance}), using the Fenchel-Young identity (Lemma~\ref{Lemma : convex duality}). 
    
    \item \textbf{Max-margin.} 
    Based on the balance equation, we introduce a novel $Q$-margin for mirror flow in Theorem~\ref{theorem : main result of margin convergence}
    and derive the corresponding constrained optimization problem in Theorem~\ref{theorem : KKT max margin result}.

    \item \textbf{Reaching the max-margin solution.} 
    We derive tight growth rates for $Q$ and decrease rates for the loss in Theorem~\ref{theorem : convergence and growth rates}. 
These highlight for non-homogeneous mirrors the difficulty of reaching the corresponding max-margin solution. 
    We leverage this analysis to provide guidance for selecting hyperparameters for mirror maps in Lemma~\ref{lemma : hyperbolic entropy alignment horizon}.  

    \item \textbf{Sparse and dense feature learning.} We show that mirror descent can learn both sparse and dense representations, as established in Theorem~\ref{theorem : reform KKT main}, and we empirically verify the different margins reached and the difficulty of reaching them depending on the smoothening parameter $\lambda$. 
    
    \item \textbf{NTK and modularity.} 
    In Appendix~\ref{section : additional implications}, we show that the constrained optimization problem admits an SVM formulation that cannot be captured within an RKHS framework, but instead requires a Reproducing Kernel Banach Space (RKBS). 
We further formulate the constrained optimization problem for layer-wise distinct mirror maps, highlighting its modular structure. 
    
\end{itemize}

\paragraph{Related work.}  
We highlight key related works here and 
defer broader discussion to Appendix~\ref{appendix : extended related work}, including connections to general implicit bias literature and hyperparameter transfer. 

\begin{itemize}[leftmargin=*] 
\item \textbf{Max-margin results.} 
A substantial body of work characterizes the implicit bias of gradient-based optimization in terms of margin maximization. 
Classical results establish that gradient descent on (linearly) separable data converges in direction to the $L_2$ max-margin classifier \citep{soudry2017implicitbiasgradientdescent, Lyu2020Gradient, Jin2017HowTE}. 
These results have recently been extended to steepest-descent \citep{tsilivis2024flavors},
inhomogeneous neural networks \citep{cai2025implicit}, simplicity bias \citep{pmlr-v235-tsoy24a}, and the edge-of-stability regime \citep{NEURIPS2023_eb189151}. 
Moreover, benign overfitting has been studied under approximate max-margin convergence \citep{Karhadkar2024BenignOI}. 
However, a max-margin characterization of mirror flow for neural networks beyond the linear setting has remained open \citep{Sun2023AUA, pesme2024implicit}. 

\item 
\textbf{Mirror flow.}
Mirror descent and its continuous-time limit, mirror flow, are classical tools in optimization that enforce %
constraints on iterates via a Bregman divergence \citep{Wu2021ImplicitRI, marcotte2024momentumconservationlawseuclidean, 10.5555/902667, Majidi2021ExponentiatedGR}. 
Mirror flow also arises naturally as the dynamics of reparametrized gradient descent \citep{Li2022ImplicitBO, jacobs2025mirror, jacobs2025mirror2, Chou2023RobustIR, woodworth2020kernel, 
epub126385}, 
and has been used to promote sparsity \citep{jacobs2025mirror, Lunk2026SparseTO}. 
Its implicit bias has been characterized in restricted settings such as linear classification \citep{Sun2023AUA, pesme2024implicit}, matrix factorization \citep{gunasekar2017implicitregularizationmatrixfactorization},  univariate regression with two-layer neural networks \citep{liang2025implicit}, and single layer attention \citep{Julistiono2024OptimizingAW}. 
We provide the first max-margin result for mirror flow in the broader setting of homogeneous neural networks. 

\item 
\textbf{Balance equation.} 
The balance equation describes a conserved quantity of gradient flow that appears in deep linear networks \citep{domin2024from, Marcotte2023AbideBT}, 
Riemannian gradient flow \citep{marcotte2024momentumconservationlawseuclidean}, 
and structured architectures such as transformers and ResNets \citep{marcotte2025transformativeconservativeconservationlaws}. 
Beyond standard gradient descent, analogous balance equations are known only 
for diagonal linear networks \citep{jacobs2026never}. 
In contrast to prior derivations relying on architectural homogeneity or Riemannian structure, we derive a new balance equation via convex duality, which is central to our max-margin analysis for homogeneous networks.

\end{itemize}

\section{Preliminaries}\label{section : preliminaries}

We consider homogeneous neural networks trained for binary classification. 
In this section, we first recall the balance equation for gradient flow and then introduce the definitions and assumptions from convex analysis needed to define mirror flow. Informally, given an objective function $\mathcal{L} : \mathbb{R}^n \rightarrow \mathbb{R}$, the  mirror flow associated with a strictly convex differentiable potential $R : \mathbb{R}^n \rightarrow \mathbb{R}$ is defined by 
\begin{equation*}
    d \nabla R(\theta_t) = - \nabla\mathcal{L}(\theta_t) dt, \qquad \theta_0 = \theta_{\text{init}}. 
\end{equation*}

\paragraph{Neural network and loss.}  
Let $f$ be a neural network that, given an input $x \in \mathbb{R}^d$ and  parameters $\theta$, outputs a real value $f(\theta, x)$. 
The sign of $f(\theta, x)$ determines the predicted class.  
We denote the training dataset by $Z := \{ (x_i,y_i) : i \in [K] \}$, where $x_i \in \mathbb{R}^d$ and $y_i \in \{\pm 1\}$ denote the input and label, respectively. 
For a loss function $\ell : \mathbb{R} \rightarrow \mathbb{R}$, the empirical training loss is given by 
$\mathcal{L}(\theta):= \sum_{i =1}^K \ell(y_i f(\theta, x_i))$. 

\begin{assumption}\label{assumptions : setting nn}
We impose the following assumption on the classifier $f$ and loss function $\ell$, following \citep{Lyu2020Gradient}: 
\begin{itemize}[leftmargin=*]
    \item (Regularity). For any fixed $x$, the map $\theta\mapsto f(\theta, x)$ is locally Lipschitz and satisfies the chain rule. 
    \item (Homogeneity). There exists $L > 0$ such that for all $c > 0$, $f(c \theta, x) = c^L f(\theta, x)$.  
    \item (Exponential Loss). $\ell(q) = \exp{(-q)}$. 
    \item (Separability). There exists $t_0$ such that $\mathcal{L}({\theta(t_0)}) < 1$. 
\end{itemize} 
\end{assumption} 

Since $f$ is only assumed to be locally Lipschitz, we work with Clarke subdifferentials. 
The subdifferential of $f$ at a point $\theta \in A$, where $A$ is a convex set, is defined as 
\begin{equation*}
    \partial^{\circ}f(\theta) := \text{conv}\left\{ \lim_{k \rightarrow \infty} \nabla f(\theta_k): \theta_k \rightarrow \theta, f \text{ is differentiable at } \theta_k\right\}.
\end{equation*}
Following \citep{Lyu2020Gradient, Davis2018StochasticSM}, 
we say that a locally Lipschitz function $f: \mathbb{R}^n \rightarrow \mathbb{R}$ admits the chain rule if, for any arc $z : [0, \infty) \rightarrow \mathbb{R}^n$, 
$\frac{d}{dt}(f \circ z)(t) = \langle h, \dot{z}(t)\rangle$ for all 
$h \in \partial^{\circ} f(z(t))$, for a.e.\ $t> 0$. 
Concretely, we consider neural networks of the form 
\begin{equation}
\label{equation : homogeneous neural network} 
    f(\theta, x) : =  W_1 \sigma \left( W_2 \cdots \sigma\left( W_L x\right) \right) , 
\end{equation}
where $\theta = [W_1, \ldots , W_L]$ and $\sigma$ is a homogeneous activation function, such as the ReLU or linear. %

\paragraph{Balance equation and margin for gradient flow.} 
For homogeneous neural networks as in Eq.~(\ref{equation : homogeneous neural network}), standard gradient flow with $R(\theta) = \frac{1}{2} \| \theta \|_2^2$ satisfies the following \emph{balance equation} for all $t\geq 0$ and $i \in [L-1]$: 
\begin{equation}
\label{equation : balance equation gradient descent}
   \frac{1}{2} \| W_{i,t} \|_2^2 - \frac{1}{2} \|W_{i+1,t} \|_2^2 = \frac{1}{2}\lambda_{0,i} , 
\end{equation}
where $\lambda_{0,i} = \| W_{i,0} \|_2^2 - \|W_{i+1,0} \|_2^2 $. 
Eq.~(\ref{equation : balance equation gradient descent}) shows that the squared Euclidean norms of the weight matrices across adjacent layers grow proportionally over time. 
This property was exploited in \citep{Lyu2020Gradient} to derive the 
$L_2$ max-margin characterization of the implicit bias of gradient flow. 

We recall the definition of a generalized margin with respect to the $L_2$-norm, following \citep{Lyu2020Gradient}, and later extend it to more general settings. 
This notion captures the minimal signed distance of the training input data to the decision boundary, measured in the geometry induced by the parameter norm. 

\begin{definition}[$L_2$-margin]
\label{definition : norm margin}  
Given a dataset $\{(x_i,y_i) \colon i\in[K]\}$ and a function $f$ that is $L$-homogeneous in its parameters, %
the generalized $L_2$-margin is a function of $\theta$ defined as  
    \begin{equation*}
        \gamma := \min_i \, y_i f\left(\frac{\theta}{\|\theta \|_2}, x_i\right) =  \min_i \frac{y_i f(\theta, x_i)}{\|\theta \|_2^{L}} . 
    \end{equation*}
Moreover, we denote $q_i : =y_i f(\theta, x_i)$ and $q_{\text{min}} := \min_i q_i$. 
\end{definition}

\paragraph{Convex analysis.} 
To define mirror flow, we first introduce the mirror potential 
$R : \mathbb{R}^n \rightarrow \mathbb{R}$. 

\begin{assumption}\label{assumption : Mirror properties}
For our analysis, we impose the following assumptions on $R$: (1) strict convexity (2) twice differentiable (3) coercivity, i.e., $\|\theta\| \rightarrow \infty$ implies $\| \nabla R\| \rightarrow \infty$, so that $\nabla R$ is surjective onto $\mathbb{R}^n$, and (4) $\nabla^2 R$ is locally Lipschitz.
\end{assumption}

We recall the convex conjugate and Fenchel-Young duality in Definition~\ref{definition : convex conjugate} and Lemma~\ref{Lemma : convex duality}. 

\begin{definition}
\label{definition : convex conjugate}
    The convex conjugate of $R$ is the function $Q : \mathbb{R}^n \rightarrow \mathbb{R}$ defined by $ Q(\theta) := \sup_{\xi \in \mathbb{R}^n} \{ \langle\theta, \xi \rangle - R(\xi) \} . $
\end{definition} 

\begin{lemma}[Fenchel-Young identity]
\label{Lemma : convex duality} 
    Let $R$ satisfy Assumption~\ref{assumption : Mirror properties} and let $Q$ denote its convex conjugate. 
    Then for all $\theta \in \mathbb{R}^n$, 
    \begin{equation*}
         \langle \theta, \nabla R(\theta) \rangle = R(\theta) + Q(\nabla R(\theta)). 
    \end{equation*}
\end{lemma}

Applying Lemma~\ref{Lemma : convex duality} to the Euclidean case $R(\theta):= \frac{1}{2} \|\theta \|_2^2$, we obtain $Q(\nabla R(\theta)) =\frac{1}{2}\|\theta \|_2^2$. 
This matches the layer-wise quantities appearing in the gradient flow balance equation in Eq.~(\ref{equation : balance equation gradient descent}). 
As we will show in Section~\ref{section : balance equation starting point}, this identity allows us to define an analogous balance equation for mirror flow, where $R(\theta)$ and $Q(\nabla R(\theta))$ need not coincide. 

\paragraph{Mirror flow.} 
In homogeneous neural networks, mirror flow is defined by the differential inclusion  
$$
\frac{d\nabla R(\theta_t)}{dt} \in \, -\partial^{\circ} \mathcal{L}(\theta_t),\quad \text{ or equivalently }\quad
\frac{d \theta_t}{dt} \in - (\nabla^2 R)^{-1}(\theta_t) \, \partial^{\circ} \mathcal{L}(\theta_t).
$$
When $f$ is differentiable, this reduces to 
$\frac{d \theta_t}{dt} = -(\nabla^2 R)^{-1}(\theta_t) \nabla \mathcal{L}(\theta_t)$. 
Under our regularity assumptions, the training loss satisfies the chain rule, yielding 
\begin{equation*}
    \frac{d}{dt} \mathcal{L}(\theta_t)
    = - \left\langle h_t, \frac{d \theta_t}{dt} \right\rangle
    = - \left\| \frac{d \theta_t}{dt} \right\|_{\nabla^2 R(\theta_t)}^2,
    \quad \text{for some } h_t \in \partial^{\circ} \mathcal{L}(\theta_t).
\end{equation*}
Here $\| \cdot \|_{\nabla^2 R(\theta)} := \sqrt{\langle \cdot , \nabla^2 R(\theta) \cdot \rangle}$ denotes the local norm 
induced by the mirror potential.

\section{Balance equation for mirror flow}
\label{section : balance equation starting point}

For homogeneous networks, the direction of convergence is governed by the growth of the parameter norm. 
We derive a balance equation for the mirror flow of a homogeneous neural network $f$ of the form given in Eq.~(\ref{equation : homogeneous neural network}). 
To this end, we assume that the mirror potential is layer-wise separable, i.e., $R(\theta) = \sum_{i = 1}^L R_i(W_i)$. 
Then the mirror flow then takes the form 
\begin{equation}
    d \nabla R_i(W_{i,t}) = - \nabla_i \mathcal{L}(\theta_t) dt, \qquad W_{i,0} = W_{i, \text{init}}, 
    \label{eq : mirror balance layer}
\end{equation}
for  $i \in [L]$, where $\nabla_i$ denotes the gradient with respect to the parameters of the $i$th layer. 
Since the network is not assumed to be $C^1$, this should be interpreted as a differential inclusion. 
In particular, solutions need not be unique and are understood in a set-valued sense.

\begin{lemma}
\label{lemma : invariance balance} 
    The mirror flow in Eq.~(\ref{eq : mirror balance layer}) satisfies the following balance equation for all $t \geq 0$: 
    \begin{equation*}
       Q_i (\nabla_{i} R_i(W_{i,t})) - Q_{i + 1}(\nabla_{{i + 1}} R_{i + 1}(W_{i + 1,t})) = \lambda_{0,i} 
       , 
    \end{equation*}
    where $\lambda_{0,i} =  Q_i (\nabla_{i} R_i(W_{i,0})) - Q_{i + 1}(\nabla_{{i + 1}} R_{i + 1}(W_{i + 1,0}))$, for all $i\in[L-1]$. 
\end{lemma}

\begin{proof}[Proof sketch] 
It follows from homogeneity of $f$ and convex duality, with details in 
Appendix~\ref{app:proof-lemma-31}. 
\end{proof} 

Lemma~\ref{lemma : invariance balance} recovers the gradient flow balance equation in the special case $R(\theta) = \frac{1}{2} \| \theta \|_2^2$. 
Moreover, it identifies $Q(\nabla R(\theta))$ as a natural normalization for the margin, as both the balance equation and the growth of the margin are governed by the homogeneity of $f$. 
To relate the dynamics to the margin, we invoke Euler's identity: 
if $f$ is homogeneous of degree $L$, then 
    $\langle \theta_t , h \rangle_2 = L f(\theta_t; x), 
    \text{ for all } h \in \partial^{\circ} f(\theta_t; x)$. 
This, together with Lemma~\ref{Lemma : convex duality}, implies 
\begin{equation}\label{equation : growth Q and loss}
    \frac{d}{dt} Q(\nabla R(\theta_t)) 
    = \left\langle \theta_t , \frac{d}{dt}\nabla R(\theta_t) \right\rangle_2 
    = -\langle \theta_t , g_t \rangle_2 
    \geq L \mathcal{L}(\theta_t)\, \log\!\bigl(1/\mathcal{L}(\theta_t)\bigr),
\end{equation}
for some $g_t \in \partial^{\circ} \mathcal{L}(\theta_t)$. 
Here the first equation follows from Lemma~\ref{Lemma : convex duality}, the second from mirror flow dynamics, and the third from Euler's identity (using the homogeneity of $f$) 
and the structure of the exponential loss, the full calculation is in Eq.~(\ref{equation : growth bound exp loss log}) in the appendix. 
This shows that the evolution of $Q(\nabla R(\theta))$ is controlled by the loss. 
This extends the analysis of the gradient flow setting in \citep{Lyu2020Gradient}, corresponding to $R(\theta): = \frac{1}{2} \|\theta \|_2^2$, to the mirror flow setting with general potentials. 

\paragraph{Mirror potentials.}
We focus on two classes of mirror potentials: 
the hyperbolic entropy and smoothened homogeneous potentials, summarized in Table~\ref{table:mirror-potentials}. 
The hyperbolic entropy is known to promote sparsity \citep{Wu2021ImplicitRI, jacobs2026hyperbolic}. 
The horizon function $\phi(\theta)$ characterizes the asymptotic behavior of the potential, capturing the shape of the level sets of $R(\theta)$ as $\| \theta\| \rightarrow \infty$, as illustrated in Figure~\ref{fig:horizon apporox}. 
The corresponding balance equation is illustrated in Figure~\ref{fig:new_balance_equation}. 

\begin{table}[h!]
\caption{Examples of separable mirror potentials $R$ and their corresponding dual potentials $Q(\nabla R(\theta))$ and horizon functions $\phi(\theta)$. Moreover we have that $ \frac{1}{p} + \frac{1}{q} = 1 $. 
}
\label{table:mirror-potentials}
\centering
\small 
\renewcommand{\arraystretch}{2} %
\begin{tabular}{|l|c|c|c|}
\hline
\textbf{Mirror Potential} & \(R(\theta)\) & \(Q(\nabla R(\theta))\) & \( \phi(\theta) \)\\
\hline
Hyperbolic Entropy & 
\(\displaystyle \sum_{i=1}^n \theta_i \, \text{arcsinh}\Big(\frac{\theta_i}{\sqrt{\lambda}}\Big) - \sqrt{\theta_i^2 + \lambda} \) &
\(\displaystyle \sum_{i=1}^n \sqrt{\theta_i^2 + \lambda} \) & \( \|\theta \|_1\) \\
\hline
Smoothened Hom. & 
\(\frac{1}{p} \|\theta\|_p^p + \frac{\lambda}{2} \|\theta\|_2^2, \ p \geq 2\) &
\(\frac{1}{q} \|\theta\|_p^p + \frac{\lambda}{2} \|\theta\|_2^2 \) & \((p-1)^{1/p}\|\theta \|_{p}\) \\
\hline
\end{tabular}
\end{table}

\section{Late-phase max-margin characterization}
\label{section : late phase max margin} 

We are now ready to characterize the implicit bias of mirror flow towards max-margin solutions. 
First we introduce the $Q$-margin and the corresponding $Q$-soft margin inspired by the balance equation presented in Section~\ref{section : balance equation starting point} which allows us to relate the dynamics to the loss. 

\paragraph{The $Q$-margin and a soft approximation.}
Analogous to the $L_2$-margin we define the $Q$-margin for asymptotically homogeneous mirror maps. 

\begin{definition}[$Q$-margin]\label{definition : q margin}
    For an $L$-homogeneous function $f$
    and asymptotically $\alpha$-homogeneous mirror map $R$, 
    the $Q$-margin is a function of $\theta$ and defined as 
    \begin{equation*}
        \gamma_Q := \min_i \, y_i f\left(\frac{\theta}{\left(\alpha Q(\nabla(R(\theta))\right)^{1/\alpha}}, x_i\right) =  \min_i \frac{y_i f(\theta, x_i)}{\left(\alpha Q(\nabla(R(\theta))\right)^{L/\alpha}} . 
    \end{equation*}
\end{definition} 

Definition~\ref{definition : q margin} generalizes the $L_2$-margin for GD which is recovered by setting $R(\theta) = \frac{1}{2} \|\theta \|_2^2$.
Next, we introduce the \emph{$Q$-soft margin} as a tractable surrogate to track the  $Q$-margin during training. 
Similar to \citep{Lyu2020Gradient} and \citep{tsilivis2024flavors}, 
we define the soft margin using the log-sum-exp (LSE). 
The term soft refers to replacing the exact minimum   with a smooth approximation, in which the margin is directly controlled by the dynamics via Eq.~(\ref{equation : growth Q and loss}).

\begin{definition}[$Q$-soft margin]
\label{definition : Q soft margin}  
    The $Q$-soft margin is defined as 
    \begin{equation*}
        \tilde{\gamma}_Q:=  \frac{\text{log}(\frac{1}{\mathcal{L}(\theta)})}{\left(\alpha Q(\nabla(R(\theta))\right)^{L/\alpha}} , 
    \end{equation*}
where $\mathcal{L}(\theta)$ is the loss. 
\end{definition}

Note that when the training set is correctly classified we have that $\mathcal{L}(t_0) < 1$ and  $\tilde\gamma_Q>0$ by Definition~\ref{definition : Q soft margin}.  
This is a key property we will use to show the $Q$-soft margin continues to grow once the data is correctly classified. 
Lemma~\ref{lemma : soft margin equivalence} shows that the $Q$-soft margin approximates the $Q$-margin. 

\begin{lemma}
\label{lemma : soft margin equivalence}
    The $Q$-soft margin is a $\mathcal{O}(\left(\alpha Q\nabla(R(\theta))\right)^{-L/\alpha})$ additive approximation of the $Q$-margin.
\end{lemma} 

\begin{proof}[Proof sketch] 
This follows from using log-sum-exp, 
with details in Appendix \ref{app:proof-Theorem45}. 
\end{proof} 

\begin{assumption}
\label{assumption : semihomogeneous}
    The mirror potential $R : \mathbb{R}^n \rightarrow \mathbb{R}$ is asymptotically $\alpha$-homogeneous with $\alpha \geq1$ and satisfies:
    \begin{equation*}
        \alpha Q(\nabla R(\theta)) \geq \| \theta \|_{\nabla^2 R}^2 . 
    \end{equation*}
\end{assumption}

Assumption~\ref{assumption : semihomogeneous} is satisfied by all mirror potentials listed in Table~ \ref{table:mirror-potentials}, as shown in Appendix~\ref{appendix : assumption mirrors}. 

\paragraph{Alignment of the $Q$-margin.}

We now show that mirror flow naturally controls the growth of the $Q$-soft margin, which in turn leads to an increase in the $Q$-margin. 

\begin{theorem} 
\label{theorem : main result of margin convergence}
    Under Assumptions \ref{assumptions : setting nn} and \ref{assumption : Mirror properties} %
    and Assumption~\ref{assumption : semihomogeneous}, 
    the $Q$-soft margin is an $\mathcal{O}(\left(\alpha Q\nabla(R(\theta))\right)^{-L/\alpha}))$ additive approximation for the $Q$-margin and the following hold: 
    \begin{itemize}
        \item For a.e.\ $t > t_0$, 
        $\frac{d}{dt} \tilde{\gamma}_Q(\theta_t)) \geq 0$. 
        \item $\mathcal{L}(\theta_t) \rightarrow 0$ and $Q(\nabla R (\theta_t)) \rightarrow \infty$ as $t \rightarrow \infty$, therefore $| \gamma_Q - \tilde{\gamma}_Q| \rightarrow 0$. 
    \end{itemize}
\end{theorem}

\begin{proof}[Proof sketch] We use Eq.~(\ref{equation : growth Q and loss}) to control the $Q$-soft margin growth and use Lemma \ref{lemma : soft margin equivalence}, for details see Appendix \ref{app:proof-Theorem45}. 
\end{proof} 

Theorem~\ref{theorem : main result of margin convergence} shows that 
mirror flow naturally drives the $Q$-soft margin to grow, and that as the loss converges to zero, the $Q$-soft margin closely tracks the $Q$-margin.

\paragraph{Convergence rate} 
We obtain the following convergence rates for mirror flow, which depend on the degree of homogeneity $\alpha \geq 1$ of the mirror map. 

\begin{theorem}
\label{theorem : convergence and growth rates}
    Under the same assumptions as in Theorem~\ref{theorem : main result of margin convergence}, we have the following. 
    \newline 
For $\alpha \in [1,2]$: $ \mathcal{L}(\theta_t) = \Theta \left(\frac{1}{t \ \log (t)^{2- \frac{\alpha}{L}}}\right) \quad \text{and} \quad (\alpha Q(\nabla R(\theta_t)))^{L/\alpha} = \Theta (\log \ t).$
\newline
    For $\alpha > 2$: $  \mathcal{L}(\theta_t) = O \left(\frac{1}{t \ \log (t)^{2- \frac{\alpha}{L}}}\right) \quad \text{and}\quad  (\alpha Q(\nabla R(\theta_t)))^{L/\alpha} = \Omega (\log \ t).$
\end{theorem}
\begin{proof}
    See Appendix \ref{section : growth rates}.
\end{proof}
Theorem~\ref{theorem : convergence and growth rates} shows that a larger $\alpha$ may slow down convergence up to a multiplicative logarithmic factor. However, for large depth $L$ this may become negligible. 

\paragraph{Implicit bias description.} 
We obtain an implicit bias description of mirror flow in terms of a corresponding constrained optimization problem and its approximate KKT conditions. 
For simplicity of presentation, the following Theorem~\ref{theorem : KKT max margin result} specializes to the smoothened homogeneous mirror potentials listed in Table~\ref{table:mirror-potentials}. We provide a more general result in Theorem~ \ref{theorem : appendix KKT max margin result}. 

\begin{theorem}(Smoothened homogeneous)
\label{theorem : KKT max margin result}
     Assume $R$ is a smoothened homogeneous potential with $\alpha = p \geq 2$. 
    Then, under the same assumptions as in Theorem~\ref{theorem : main result of margin convergence}, 
    the parameter direction $\bar{\theta}_t := \frac{\theta_t}{\| \theta_t \|_2}$ converges to a KKT point of the following optimization problem: 
    \begin{equation*}
        \min_{\theta \in \mathbb{R}^n} \, \frac{1}{2}\phi^2_{\alpha}(\theta) \qquad \text{such that} \qquad y_i f(\theta, x_i) \geq 1, \text{ for all } i \in [K] , 
    \end{equation*}
    where $\phi_{\alpha} : = \lim_{\eta \rightarrow 0} \eta (\alpha Q(\nabla R(\theta /\eta)))^{1/\alpha}$ is the horizon function.
\end{theorem}

\begin{proof}[Proof sketch] 
The proof follows a similar strategy as used in \citep{Lyu2020Gradient} and \citep{tsilivis2024flavors}. The main differences are: (1) Approximating the horizon function by the iterates of $\nabla (\alpha Q(\nabla R(\theta)))^{2/\alpha}$, which requires $\alpha \geq 2$. 
(2) Showing that $\nabla (\alpha Q(\nabla R(\theta)))^{1/\alpha}/q_{\text{min}}^{1/L}$ tends to an approximate KKT point. 
(3) Showing convergence of the normalized dual iterates to a limit point, this implies that the normalized primal iterates converge as well. 
The full proof is given in Appendix~\ref{theorem : appendix KKT max margin result}. \end{proof}

 Theorem~\ref{theorem : KKT max margin result} relies on $\alpha \geq 2$. 
 For the hyperbolic entropy, which does not satisfy that assumption (see details in Appendix~\ref{kkt assumption}), we can exploit a connection to homogeneous reparameterizations trained with gradient flow from \citep{Li2022ImplicitBO, jacobs2025mirror, jacobs2025mirror2}. 
 This allows us to apply results for homogeneous neural networks trained with gradient flow and show the following.

\begin{corollary}(Hyperbolic entropy)
\label{theorem : KKT max margin result hyperbolic}
    For the hyperbolic entropy $R_{\lambda}(\theta)$ with hyperparameter $\lambda > 0$ under same assumptions as Theorem~\ref{theorem : main result of margin convergence}, the 
    parameter direction $\bar{\theta}_t := \frac{\theta_t}{\| \theta_t \|_1}$ converges 
    to a KKT point of the following optimization problem: 
    \begin{equation*}
        \min_{\theta \in \mathbb{R}^n} \, \frac{1}{2}\phi_{1}^2(\theta) \qquad \text{such that} \qquad y_i f(\theta, x_i) \geq 1 , \text{ for all } i \in [K] , 
    \end{equation*}
    where $\phi_{1} : = \| \theta \|_1$ is the corresponding horizon function. 
\end{corollary}
\begin{proof}[Proof sketch] 
    We use the relationship between 2-homogeneous reparameterization and the hyperbolic mirror map, which allows us to apply Theorem~4.4 in \citep{Lyu2020Gradient}. 
    See Appendix~\ref{appendix : hyperbolic} for details. 
\end{proof}

\paragraph{From sparse to dense feature learning.} 
Consider now a two layer neural network $f((a,w), x): = \sum_{j =1}^N a_j \sigma (w_j^T x)$. 
The previous two results can be reformulated as follows: 

\begin{theorem}
\label{theorem : reform KKT main} 
    Let $R$ be the hyperbolic entropy ($\alpha =1$) or a smoothened homogeneous mirror potential ($\alpha =p\geq 2$). 
    Then the constrained optimization problems from Theorems \ref{theorem : KKT max margin result} and \ref{theorem : KKT max margin result hyperbolic} can reformulated as:
    \begin{equation}
    \label{equation : reformulation of KKT problem}
    \min_{\tilde{a} \in \mathbb{R}^N, \tilde{w} \in \mathbb{R}^{Nd}} \sum_{j =1} |\tilde{a}_j|^{\alpha/2}  \qquad \text{such that} \qquad \|\tilde{w}_j \|_{L_{\alpha}} =1, \text{ for all } j \in [N] , 
    \end{equation}
    where $\tilde{a}_j = a_j \| w_j \|_{L_{\alpha}}$ and $\tilde{w}_j = w_j /\|w_j \|_{L_{\alpha}}$, and such that $y_i f(\theta, x_i) \geq 1, \text{ for all } i \in [K]$.
\end{theorem}

\begin{proof}[Proof sketch] 
    It follows from the balance equation for the %
    parameter iterates and the rescale invariance of the two-layer neural network. 
    The full proof is presented in Appendix \ref{subsection : reform}. 
\end{proof}

Theorem~\ref{theorem : reform KKT main} reveals that different values of the homogeneity degree $\alpha$ lead to different forms of feature learning. 
Specifically, for the hyperbolic entropy ($\alpha =1$) mirror flow is biased towards networks with fewer active neurons, whereas for smoothened homogeneous potentials ($\alpha \geq 2$) it is biased towards networks with more active neurons. 
This is illustrated numerically in Figure~\ref{fig: feature learning}. Note in particular that all the considered mirrors exhibit feature learning in the sense that the input weights of the neurons move significantly from their random initial positions. 

\begin{wrapfigure}{r}{0.4\textwidth}
 \vspace{-.75cm}
  \centering
  \includegraphics[width=0.4\textwidth]{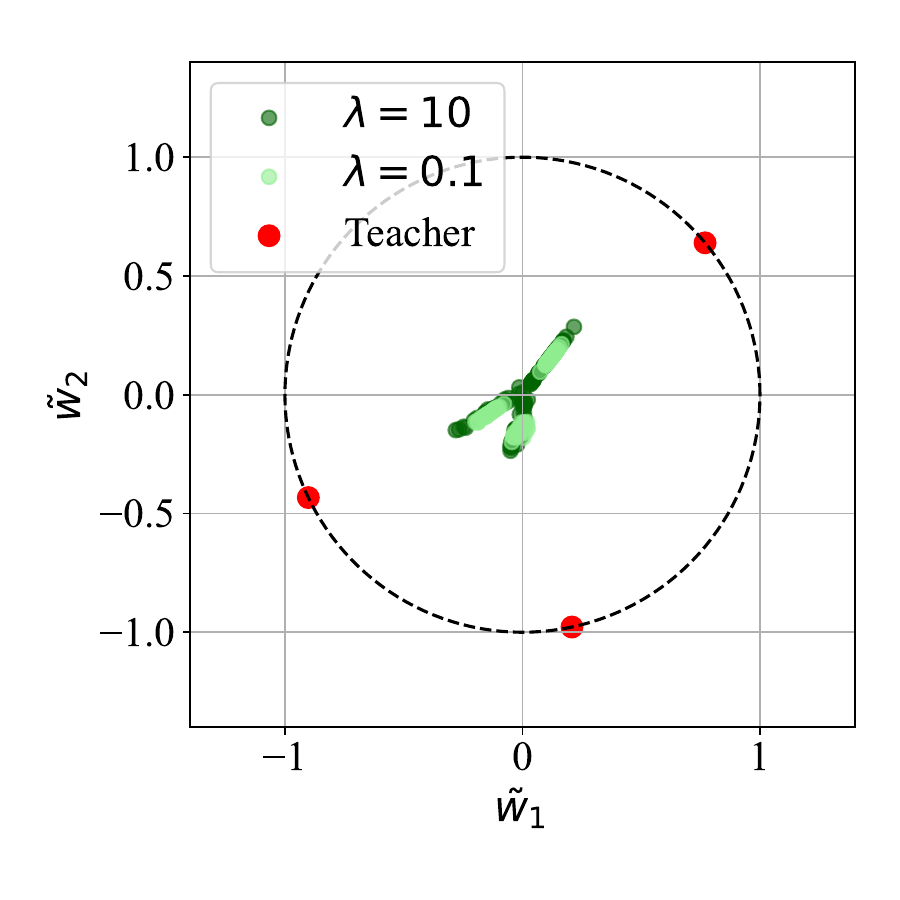}
 \vspace{-.9cm}  
  \caption{Training with smoothened homogeneous potential ($p=3$).
  This shows that increasing $\lambda >0$ slows convergence toward the max-margin solution, so that for finite training time the learned representation remains closer to that of gradient descent.
  }\label{fig : L3 finding different representation}
 \vspace{-1.2cm}
\end{wrapfigure}

\paragraph{Reaching the margin.} 
The results presented so far characterize the implicit bias of mirror flow through the introduction of a $Q$-margin and its maximization in the late phase of training. 
Reaching the max-margin solution, determined by the (homogeneous) horizon function, critically depends on the rate of parameter growth. 
In particular, convergence to the max-margin direction can be extremely slow, in some cases even exponentially slow, as illustrated in Figure~\ref{fig:horizon apporox}. 
As established in Theorems~\ref{theorem : KKT max margin result} and \ref{theorem : KKT max margin result hyperbolic}, the limiting direction is characterized by the horizon function $\phi_{\alpha}$. However, along the dynamics, it is the $Q$-margin that governs parameter growth. 
Thus it is of interest how quickly the $Q$-margin approximates the $\phi_{\alpha}$-margin. We characterize this next.

\begin{lemma} 
\label{lemma : hyperbolic entropy alignment horizon} 
    For the hyperbolic entropy and smoothened homogeneous mirror flow $R_{\lambda} : \mathbb{R}^n \rightarrow \mathbb{R}$, we have that once the time surpasses $\Omega(\exp ((\sqrt{\lambda} n)^L))$ and $\Omega(\exp (\left(\tfrac{1}{2} {p \lambda n} \right)^{L/p}))$, respectively, 
    the relative difference between $Q_{\lambda}(\nabla R_{\lambda}(\theta))$ and $\phi_1(\theta)$ is $O(1)$.  
\end{lemma} 
\begin{proof}
    See Appendix \ref{appendix : alignment}. %
\end{proof}

Lemma~\ref{lemma : hyperbolic entropy alignment horizon} can be used to select training hyperparameters mitigating slow margin alignment.  
The takeaway is that the hyperparameter $\lambda$ should to be relatively small compared to $n$ if we want to reach the max-margin solution in sub-exponential time. 
For example, it is sufficient 
to choose $\lambda =1/ n^2$ for the hyperbolic entropy and $\lambda = 1/pn$ for the smoothened potentials. 
Note that otherwise it can %
take an exponential time in $n$ to reach the max-margin solution, effectively making it unreachable. This is shown in Figure \ref{fig : L3 finding different representation}, for smoothened homogeneous potential in a two layer student-teacher setting. 

\section{Experimental validation}
We validate here our theoretical results. First we confirm Theorem~\ref{theorem : reform KKT main} by considering a two-layer network trained in a student-teacher setup. 
Next, we test the mirror descent algorithms on a standard vision task, CIFAR10. 
The full training details and additional ablations are provided in Appendix~\ref{appendix : experiments}.

\paragraph{Sparse and dense feature learning.}
To illustrate Theorem~\ref{theorem : reform KKT main}, we consider a teacher network with $3$ neurons together with two- and three-layer student networks of hidden dimension $100$. 
The final parameters are shown in Figure~\ref{fig: feature learning}. 
We train using (i) hyperbolic entropy, (ii) gradient descent, and (iii) a smoothened homogeneous mirror potential ($p=3$). 
Observe that all $3$ optimization procedures exhibit feature learning. 
However, as predicted by the theorem, smaller homogeneity degrees ($\alpha=1$) produce solutions with fewer active neurons or weights, 
whereas larger values ($\alpha=3$) produce solutions with more active neurons or weights. 
We refer to these regimes as sparse and dense feature learning. 
Furthermore, Figure~\ref{fig : L3 finding different representation} validates the dependence on $\lambda > 0$, confirming Lemma~\ref{lemma : hyperbolic entropy alignment horizon}. 

\paragraph{Shaping representation structure in a vision task.} 
We train a VGG-16 \citep{Simonyan2014VeryDC} on 
CIFAR-10 \cite{cifar}, initialized in the lazy regime, using a hyperparameter sweep over $\lambda$\footnote{Here $\lambda$ is layerwise rescaled correcting for the width of each layer as detailed in Appendix~\ref{appendix : experiments}.} and the 
learning rate (details in Appendix~ \ref{appendix : experiments}). 
Hyperbolic entropy achieves the best validation 
accuracy (Table~\ref{tab:results} in Appendix~\ref{appendix : experiments}). 
The weight distributions in Figure~\ref{fig : weight dists} confirm the sparse and dense feature learning patterns observed in the student-teacher setting, and show that $\lambda$ must be small to alter the weight representation, consistent with Lemma~\ref{lemma : hyperbolic entropy alignment horizon}. 
Furthermore, as a consequence of the changed weight distribution, pruning the network leads to reduced performance degradation for the hyperbolic entropy and increased degradation for the smoothened homogeneous potential. 
Moreover, we find that the weight distribution of the last layer remains unchanged (Figure~\ref{fig: hist last layer} in Appendix~\ref{appendix : experiments}), consistent with training under standard parameterization (SP) \citep{pmlr-v139-yang21c}. 

\begin{figure}[ht!]
    \centering
    \subfigure
    {
    \includegraphics[width=0.32\textwidth]{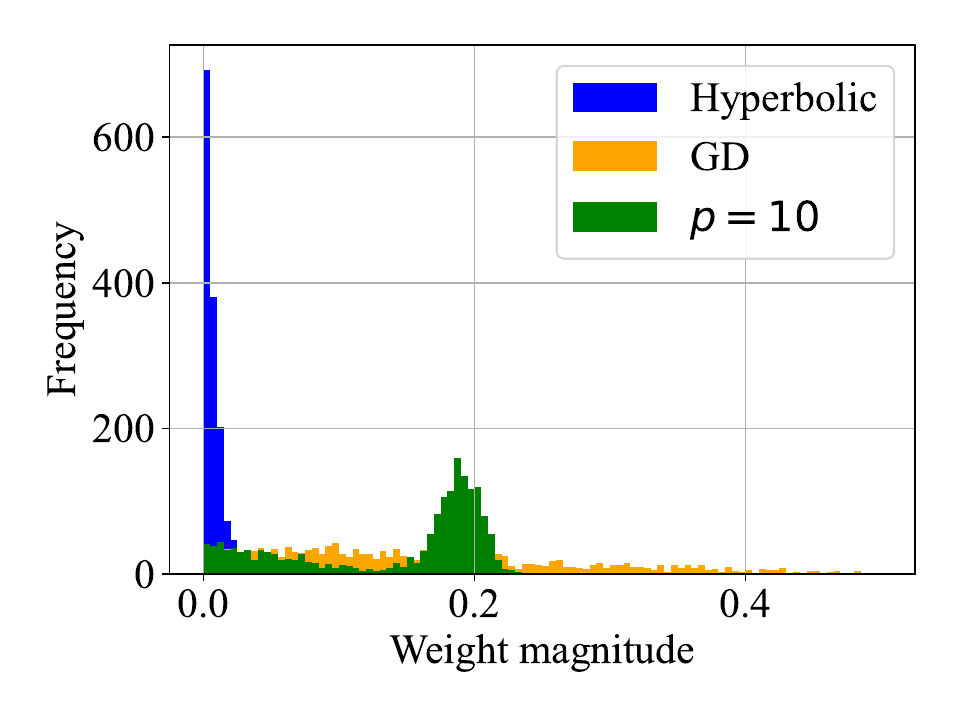}
        \label{fig: dense vs sparse}
    }
    \subfigure
    {\includegraphics[width=0.32\textwidth]{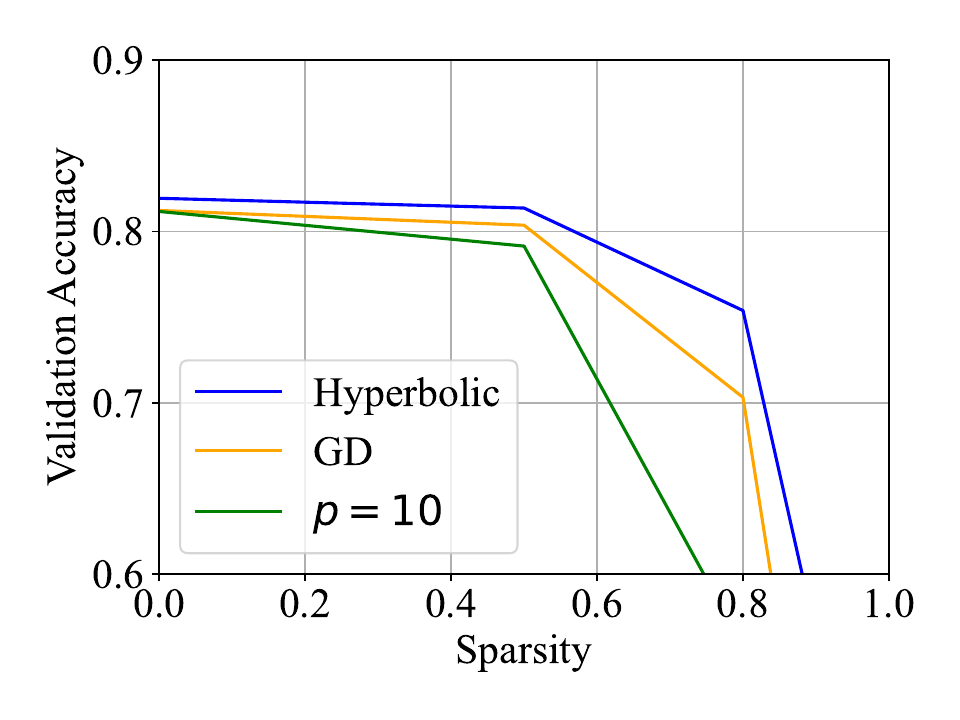}
        \label{fig: sparsity}}
    \subfigure
    {\includegraphics[width=0.32\textwidth]{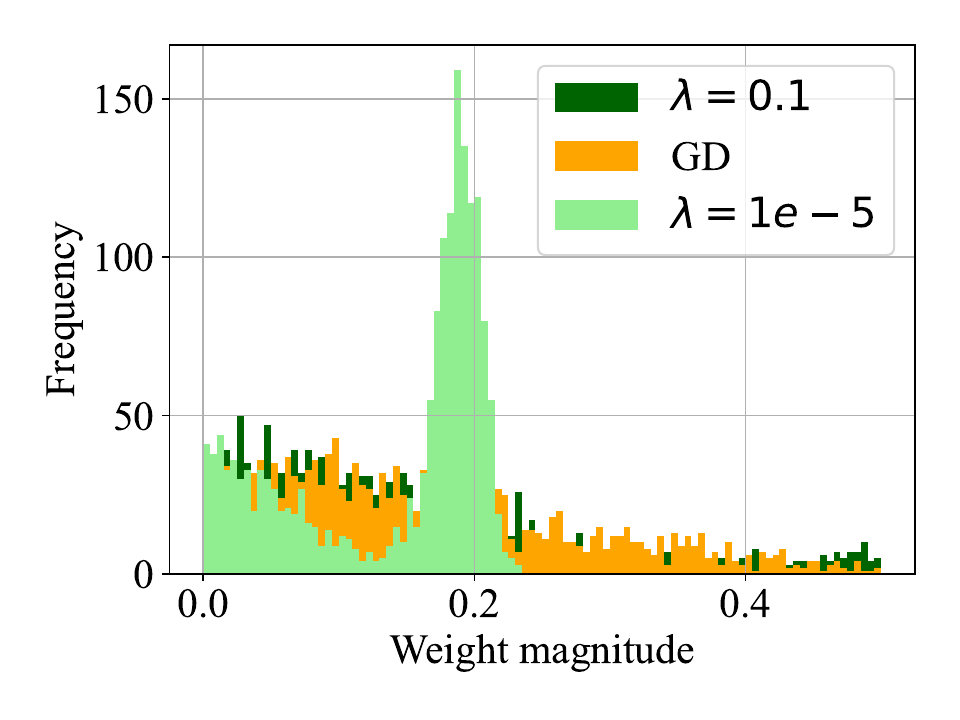}
        \label{fig: lambda large small}}

    \caption{Weight distributions for the first layer of a VGG-16 and weight pruning. (Left) We show the distribution for $3$ mirror maps leading to dense and sparse representations. (Middle) Validation accuracy versus sparsity, we can prune using layer-wise magnitude pruning more weights when training with the hyperbolic entropy and less with the homogeneous potential ($p =10$) due to the change in weight distribution. (Right) Homogeneous mirror potential with $p=10$ both for large and small $\lambda$, large $\lambda$ leads to a similar representation as gradient descent (GD), confirming Lemma~\ref{lemma : hyperbolic entropy alignment horizon}.}
    \label{fig : weight dists}
\end{figure}

\section{Conclusion and future work}
We established directional late-phase implicit bias results for a class of mirror descent algorithms in homogeneous neural networks. Our analysis is based on a novel balance equation derived via the Fenchel-Young identity. 
In addition, we show that the hyperparameter $\lambda > 0$ plays a critical role in whether the implicit bias induced by mirror flow is realized in practice, as large values of $\lambda$ can substantially slow convergence to the corresponding max-margin solution. 
Together, these results %
show how %
mirror %
geometry steers optimization toward distinct max-margin solutions and induces either 
sparse or dense feature learning depending on the chosen geometry. 

In particular, the $L_1$ max-margin promotes sparsity by selecting a small subset of active parameters. In contrast, the $L_p$ max-margin for $p \geq 2$ encourages dense but homogeneous weight distributions, where weights concentrate around similar magnitudes. 
While the former aligns naturally with sparse feature selection, the latter may benefit quantization, as more uniform weight scales facilitate mapping parameters to discrete levels. 
Overall, our results suggest that mirror geometry provides a unified mechanism for steering models toward either sparse or quantization-friendly representations. 

Finally, this raises several questions beyond the scope of the present work, including extending Theorem~\ref{theorem : KKT max margin result} to the $\alpha < 2$ regime (discussed in Appendix \ref{kkt assumption}), as well as understanding the effects of finite learning rate, early stopping, and the generalization properties of these max-margin solutions.

\newpage

\subsection*{Acknowledgments} 
This work was supported in part by the DFG project 464109215 within the Priority Programme SPP 2298 ``Theoretical Foundations of Deep Learning''. 
TJ has been supported by funding from the European
Research Council (ERC) under the Horizon Europe Framework Programme (HORIZON) for proposal
number 101116395 SPARSE-ML. 
GM has been supported in part by 
the DARPA AIQ grant HR00112520014, 
NSF grants 
DMS-2522495, 
DMS-2145630, 
CCF-2212520,  
and the BMFTR in DAAD project 57616814 (SECAI).


\bibliography{refs}

@inproceedings{
pesme2024implicit,
title={Implicit Bias of Mirror Flow on Separable Data},
author={Scott Pesme and Radu-Alexandru Dragomir and Nicolas Flammarion},
booktitle={The Thirty-eighth Annual Conference on Neural Information Processing Systems},
year={2024},
url={https://openreview.net/forum?id=wiMaws0FWB}
}

@inproceedings{
liang2025implicit,
title={Implicit Bias of Mirror Flow for Shallow Neural Networks in Univariate Regression},
author={Shuang Liang and Guido Mont\'ufar},
booktitle={The Thirteenth International Conference on Learning Representations},
year={2025},
url={https://openreview.net/forum?id=IF0Q9KY3p2}
}

@inproceedings{
jacobs2025mirror,
title={Mask in the Mirror: Implicit Sparsification},
author={Tom Jacobs and Rebekka Burkholz},
booktitle={The Thirteenth International Conference on Learning Representations},
year={2025},
url={https://openreview.net/forum?id=U47ymTS3ut}
}

@inproceedings{
jacobs2025mirror2,
title={Mirror, Mirror of the Flow: How Does Regularization Shape Implicit Bias?},
author={Tom Jacobs and Chao Zhou and Rebekka Burkholz},
booktitle={Forty-second International Conference on Machine Learning},
year={2025},
url={https://openreview.net/forum?id=MLiR9LS5PW}
}

@InProceedings{woodworth2020kernel,
  title = 	 {Kernel and Rich Regimes in Overparametrized Models},
  author =       {Woodworth, Blake and Gunasekar, Suriya and Lee, Jason D. and Moroshko, Edward and Savarese, Pedro and Golan, Itay and Soudry, Daniel and Srebro, Nathan},
  booktitle = 	 {Proceedings of Thirty Third Conference on Learning Theory},
  pages = 	 {3635--3673},
  year = 	 {2020},
  noeditor = 	 {Abernethy, Jacob and Agarwal, Shivani},
  volume = 	 {125},
  series = 	 {Proceedings of Machine Learning Research},
  month = 	 {09--12 Jul},
  publisher =    {PMLR},
  url = 	 {https://proceedings.mlr.press/v125/woodworth20a.html}
}

@InProceedings{pmlr-v139-yang21c,
  title = 	 {Tensor Programs {IV}: Feature Learning in Infinite-Width Neural Networks},
  author =       {Yang, Greg and Hu, Edward J.},
  booktitle = 	 {Proceedings of the 38th International Conference on Machine Learning},
  pages = 	 {11727--11737},
  year = 	 {2021},
  noeditor = 	 {Meila, Marina and Zhang, Tong},
  volume = 	 {139},
  series = 	 {Proceedings of Machine Learning Research},
  month = 	 {18--24 Jul},
  publisher =    {PMLR},
  url = 	 {https://proceedings.mlr.press/v139/yang21c.html}
}

@article{Sun2023AUA,
  title={A Unified Approach to Controlling Implicit Regularization via Mirror Descent},
  author={Haoyuan Sun and Khashayar Gatmiry and Kwangjun Ahn and Navid Azizan},
  journal={ArXiv},
  year={2023},
  volume={abs/2306.13853},
  url={https://api.semanticscholar.org/CorpusID:259251984}
}

@inproceedings{
Li2022ImplicitBO,
title={Implicit Bias of Gradient Descent on Reparametrized Models: On Equivalence to Mirror Descent},
author={Zhiyuan Li and Tianhao Wang and Jason D. Lee and Sanjeev Arora},
booktitle={Advances in Neural Information Processing Systems},
noeditor={Alice H. Oh and Alekh Agarwal and Danielle Belgrave and Kyunghyun Cho},
year={2022},
url={https://openreview.net/forum?id=k4KHXS6_zOV}
}

@inproceedings{
Lyu2020Gradient,
title={Gradient Descent Maximizes the Margin of Homogeneous Neural Networks},
author={Kaifeng Lyu and Jian Li},
booktitle={International Conference on Learning Representations},
year={2020},
url={https://openreview.net/forum?id=SJeLIgBKPS}
}

@article{Heeringa2025DeepNA,
  title={Deep Networks are Reproducing Kernel Chains},
  author={Tjeerd Jan Heeringa and Len Spek and Christoph Brune},
  journal={ArXiv},
  year={2025},
  volume={abs/2501.03697},
  url={https://api.semanticscholar.org/CorpusID:275342340}
}

@misc{soudry2017implicitbiasgradientdescent,
      title={The Implicit Bias of Gradient Descent on Separable Data}, 
      author={Daniel Soudry and Elad Hoffer and Mor Shpigel Nacson and Suriya Gunasekar and Nathan Srebro},
      year={2017},
      eprint={1710.10345},
      archivePrefix={arXiv},
      primaryClass={stat.ML},
      url={https://arxiv.org/abs/1710.10345}, 
}

@inproceedings{marcotte2025transformativeconservativeconservationlaws,
title={Transformative or Conservative? Conservation laws for ResNets and Transformers},
author={Sibylle Marcotte and R{\'e}mi Gribonval and Gabriel Peyr{\'e}},
booktitle={Forty-second International Conference on Machine Learning},
year={2025},
url={https://openreview.net/forum?id=aTBwCSkPxv}
}

@inproceedings{Jin2017HowTE,
  title={How to Escape Saddle Points Efficiently},
  author={Chi Jin and Rong Ge and Praneeth Netrapalli and Sham M. Kakade and Michael I. Jordan},
  booktitle={International Conference on Machine Learning},
  year={2017},
  url={https://api.semanticscholar.org/CorpusID:14198632}
}

@misc{marcotte2024momentumconservationlawseuclidean,
      title={Keep the Momentum: Conservation Laws beyond {E}uclidean Gradient Flows}, 
      author={Sibylle Marcotte and Rémi Gribonval and Gabriel Peyré},
      year={2024},
      eprint={2405.12888},
      archivePrefix={arXiv},
      primaryClass={cs.LG},
      url={https://arxiv.org/abs/2405.12888}, 
}

@inproceedings{
cai2025implicit,
title={Implicit Bias of Gradient Descent for Non-Homogeneous Deep Networks},
author={Yuhang Cai and Kangjie Zhou and Jingfeng Wu and Song Mei and Michael Lindsey and Peter Bartlett},
booktitle={Forty-second International Conference on Machine Learning},
year={2025},
url={https://openreview.net/forum?id=ruSU7xtH6v}
}

@article{epub126385,
           pages = {1437--1460},
          number = {3},
           title = {More is less: inducing sparsity via overparameterization},
            year = {2023},
          author = {Hung-Hsu Chou and Johannes Maly and Holger Rauhut},
         journal = {Information and Inference},
       publisher = {Oxford University Press (OUP)},
             url = {},
        volume ={12}
}

@article{Chou2023RobustIR,
    author = {Chou, Hung-Hsu and Rauhut, Holger and Ward, Rachel},
    title = {Robust implicit regularization via weight normalization},
    journal = {Information and Inference: A Journal of the IMA},
    volume = {13},
    number = {3},
    pages = {iaae022},
    year = {2024},
    month = {09},
    issn = {2049-8772},
    doi = {10.1093/imaiai/iaae022},
    url = {https://doi.org/10.1093/imaiai/iaae022},
    eprint = {https://academic.oup.com/imaiai/article-pdf/13/3/iaae022/59066595/iaae022.pdf},
}

@inproceedings{Savarese2019HowDI,
  title={How do infinite width bounded norm networks look in function space?},
  author={Pedro H. P. Savarese and Itay Evron and Daniel Soudry and Nathan Srebro},
  booktitle={Annual Conference Computational Learning Theory},
  year={2019},
  url={https://api.semanticscholar.org/CorpusID:61153558}
}

@article{10.5555/902667,
title = {Exponentiated Gradient versus Gradient Descent for Linear Predictors},
journal = {Information and Computation},
volume = {132},
number = {1},
pages = {1--63},
year = {1997},
issn = {0890-5401},
doi = {https://doi.org/10.1006/inco.1996.2612},
url = {https://www.sciencedirect.com/science/article/pii/S0890540196926127},
author = {Jyrki Kivinen and Manfred K. Warmuth}
}

@InProceedings{pmlr-v235-tsoy24a,
  title = 	 {Simplicity Bias of Two-Layer Networks beyond Linearly Separable Data},
  author =       {Tsoy, Nikita and Konstantinov, Nikola},
  booktitle = 	 {Proceedings of the 41st International Conference on Machine Learning},
  pages = 	 {48728--48767},
  year = 	 {2024},
  noeditor = 	 {Salakhutdinov, Ruslan and Kolter, Zico and Heller, Katherine and Weller, Adrian and Oliver, Nuria and Scarlett, Jonathan and Berkenkamp, Felix},
  volume = 	 {235},
  series = 	 {Proceedings of Machine Learning Research},
  month = 	 {21--27 Jul},
  publisher =    {PMLR},
  url = 	 {https://proceedings.mlr.press/v235/tsoy24a.html}
}

@article{Simonyan2014VeryDC,
  title={Very Deep Convolutional Networks for Large-Scale Image Recognition},
  author={Karen Simonyan and Andrew Zisserman},
  journal={CoRR},
  year={2014},
  volume={abs/1409.1556},
  url={https://api.semanticscholar.org/CorpusID:14124313}
}

@article{Majidi2021ExponentiatedGR,
  title={Exponentiated Gradient Reweighting for Robust Training Under Label Noise and Beyond},
  author={Negin Majidi and Ehsan Amid and Hossein Talebi and Manfred K. Warmuth},
  journal={ArXiv},
  year={2021},
  volume={abs/2104.01493},
  url={https://api.semanticscholar.org/CorpusID:233024829}
}

@inproceedings{NEURIPS2023_eb189151,
 author = {Wu, Jingfeng and Braverman, Vladimir and Lee, Jason D},
 booktitle = {Advances in Neural Information Processing Systems},
 noeditor = {A. Oh and T. Naumann and A. Globerson and K. Saenko and M. Hardt and S. Levine},
 pages = {74229--74256},
 publisher = {Curran Associates, Inc.},
 title = {Implicit Bias of Gradient Descent for Logistic Regression at the Edge of Stability},
 url = {https://proceedings.neurips.cc/paper_files/paper/2023/file/eb189151ced0ff808abafd16a51fec92-Paper-Conference.pdf},
 volume = {36},
 year = {2023}
}

@inproceedings{
jacobs2026never,
title={Never Saddle: Reparameterized Steepest Descent as Mirror Flow},
author={Tom Jacobs and Chao Zhou and Rebekka Burkholz},
booktitle={The Fourteenth International Conference on Learning Representations},
year={2026},
url={https://openreview.net/forum?id=YgudIlQ9nC}
}

@inproceedings{
jacobs2026hyperbolic,
title={Hyperbolic Aware Minimization: Implicit Bias for Sparsity},
author={Tom Jacobs and Advait Gadhikar and Celia Rubio-Madrigal and Rebekka Burkholz},
booktitle={The Fourteenth International Conference on Learning Representations},
year={2026},
url={https://openreview.net/forum?id=XKB5Hu0ACY}
}

@misc{Lunk2026SparseTO,
  title={Sparse Training of Neural Networks based on Multilevel Mirror Descent},
  author={Yannick Lunk and Sebastian James Scott and Leon Bungert},
  year={2026},
  url={https://api.semanticscholar.org/CorpusID:285276107}
}

@misc{vaskevicius2019implicitregularizationoptimalsparse,
      title={Implicit Regularization for Optimal Sparse Recovery}, 
      author={Tomas Vaškevičius and Varun Kanade and Patrick Rebeschini},
      year={2019},
      eprint={1909.05122},
      archivePrefix={arXiv},
      primaryClass={stat.ML},
      url={https://arxiv.org/abs/1909.05122}, 
}

@InProceedings{azulay2021implicitbiasinitializationshape,
  title = 	 {On the Implicit Bias of Initialization Shape: Beyond Infinitesimal Mirror Descent},
  author =       {Azulay, Shahar and Moroshko, Edward and Nacson, Mor Shpigel and Woodworth, Blake E and Srebro, Nathan and Globerson, Amir and Soudry, Daniel},
  booktitle = 	 {Proceedings of the 38th International Conference on Machine Learning},
  pages = 	 {468--477},
  year = 	 {2021},
  noeditor = 	 {Meila, Marina and Zhang, Tong},
  volume = 	 {139},
  series = 	 {Proceedings of Machine Learning Research},
  month = 	 {18--24 Jul},
  publisher =    {PMLR},
  url = 	 {https://proceedings.mlr.press/v139/azulay21a.html}
}

@inproceedings{Wu2021ImplicitRI,
  title={Implicit Regularization in Matrix Sensing via Mirror Descent},
  author={Fan Wu and Patrick Rebeschini},
  booktitle={Neural Information Processing Systems},
  year={2021},
  url={https://api.semanticscholar.org/CorpusID:235248126}
}

@inproceedings{gunasekar2017implicitregularizationmatrixfactorization,
 author = {Gunasekar, Suriya and Woodworth, Blake E and Bhojanapalli, Srinadh and Neyshabur, Behnam and Srebro, Nati},
 booktitle = {Advances in Neural Information Processing Systems},
 noeditor = {I. Guyon and U. Von Luxburg and S. Bengio and H. Wallach and R. Fergus and S. Vishwanathan and R. Garnett},
 pages = {},
 publisher = {Curran Associates, Inc.},
 title = {Implicit Regularization in Matrix Factorization},
 url = {https://proceedings.neurips.cc/paper_files/paper/2017/file/58191d2a914c6dae66371c9dcdc91b41-Paper.pdf},
 volume = {30},
 year = {2017}
}

@inproceedings{
domin2024from,
title={From Lazy to Rich: Exact Learning Dynamics in Deep Linear Networks},
author={Cl{\'e}mentine Carla Juliette Domin{\'e} and Nicolas Anguita and Alexandra Maria Proca and Lukas Braun and Daniel Kunin and Pedro A. M. Mediano and Andrew M Saxe},
booktitle={UniReps: 2nd Edition of the Workshop on Unifying Representations in Neural Models},
year={2024},
url={https://openreview.net/forum?id=yAE3LOjgA4}
}

@inproceedings{Karhadkar2024BenignOI,
title={Benign overfitting in leaky {ReLU} networks with moderate input dimension},
author={Kedar Karhadkar and Erin George and Michael Murray and Guido Mont\'ufar and Deanna Needell},
booktitle={The Thirty-eighth Annual Conference on Neural Information Processing Systems},
year={2024},
url={https://openreview.net/forum?id=88TzdGyPT6}
}

@article{DBLP:journals/corr/abs-2405-14813,
  publtype={informal},
  author={Tim Large and Yang Liu and Minyoung Huh and Hyojin Bahng and Phillip Isola and Jeremy Bernstein},
  title={Scalable Optimization in the Modular Norm},
  year={2024},
  cdate={1704067200000},
  journal={CoRR},
  volume={abs/2405.14813},
  url={https://doi.org/10.48550/arXiv.2405.14813}
}

@inproceedings{
bernstein2025modular,
title={Modular Duality in Deep Learning},
author={Jeremy Bernstein and Laker Newhouse},
booktitle={Forty-second International Conference on Machine Learning},
year={2025},
url={https://openreview.net/forum?id=hErdffTsLu}
}

@article{Davis2018StochasticSM,
  title={Stochastic Subgradient Method Converges on Tame Functions},
  author={Damek Davis and Dmitriy Drusvyatskiy and Sham M. Kakade and J. Lee},
  journal={Foundations of Computational Mathematics},
  year={2018},
  volume={20},
  pages={119--154}, 
  url={https://api.semanticscholar.org/CorpusID:5025719}
}

@inproceedings{Marcotte2023AbideBT, 
title={Abide by the law and follow the flow: conservation laws for gradient flows},
author={Sibylle Marcotte and R{\'e}mi Gribonval and Gabriel Peyr{\'e}},
booktitle={Thirty-seventh Conference on Neural Information Processing Systems},
year={2023},
url={https://openreview.net/forum?id=kMueEV8Eyy}
}

@inproceedings{
tsilivis2024flavors,
title={Flavors of Margin: Implicit Bias of Steepest Descent in Homogeneous Neural Networks},
author={Nikolaos Tsilivis and Gal Vardi and Julia Kempe},
booktitle={NeurIPS 2024 Workshop on Mathematics of Modern Machine Learning},
year={2024},
url={https://openreview.net/forum?id=3eDGC0JtZt}
}

@misc{li2021implicitsparseregularizationimpact,
      title={Implicit Sparse Regularization: The Impact of Depth and Early Stopping}, 
      author={Jiangyuan Li and Thanh V. Nguyen and Chinmay Hegde and Raymond K. W. Wong},
      year={2021},
      eprint={2108.05574},
      archivePrefix={arXiv},
      primaryClass={stat.ML},
      url={https://arxiv.org/abs/2108.05574}, 
}

@inproceedings{
li2021resolvingimplicitbiasgradient,
title={Towards Resolving the Implicit Bias of Gradient Descent for Matrix Factorization: Greedy Low-Rank Learning},
author={Zhiyuan Li and Yuping Luo and Kaifeng Lyu},
booktitle={International Conference on Learning Representations},
year={2021},
url={https://openreview.net/forum?id=AHOs7Sm5H7R}
}

@inproceedings{
Marion2024DeepLN,
title={Deep linear networks for regression are implicitly regularized towards flat minima},
author={Pierre Marion and L{\'e}na{\"\i}c Chizat},
booktitle={The Thirty-eighth Annual Conference on Neural Information Processing Systems},
year={2024},
url={https://openreview.net/forum?id=F738WY1Xm4}
}

@inproceedings{arora2019implicitregularizationdeepmatrix,
 author = {Arora, Sanjeev and Cohen, Nadav and Hu, Wei and Luo, Yuping},
 booktitle = {Advances in Neural Information Processing Systems},
 noeditor = {H. Wallach and H. Larochelle and A. Beygelzimer and F. d\textquotesingle Alch\'{e}-Buc and E. Fox and R. Garnett},
 pages = {},
 publisher = {Curran Associates, Inc.},
 title = {Implicit Regularization in Deep Matrix Factorization},
 url = {https://proceedings.neurips.cc/paper_files/paper/2019/file/c0c783b5fc0d7d808f1d14a6e9c8280d-Paper.pdf},
 volume = {32},
 year = {2019}
}

@inproceedings{lyu2024dichotomyearlylatephase,
  author={Kaifeng Lyu and Jikai Jin and Zhiyuan Li and Simon Shaolei Du and Jason D. Lee and Wei Hu},
  title={Dichotomy of Early and Late Phase Implicit Biases Can Provably Induce Grokking},
  year={2024},
  cdate={1704067200000},
  url={https://openreview.net/forum?id=XsHqr9dEGH},
  booktitle={ICLR},
}

@article{Julistiono2024OptimizingAW,
  title={Optimizing Attention with Mirror Descent: Generalized Max-Margin Token Selection},
  author={Aaron Alvarado Kristanto Julistiono and Davoud Ataee Tarzanagh and Navid Azizan},
  journal={ArXiv},
  year={2024},
  volume={abs/2410.14581},
  url={https://api.semanticscholar.org/CorpusID:273482797}
}

@book{filippov1988differential,
  title={Differential Equations with Discontinuous Right-Hand Sides},
  author={Filippov, A. F.},
  publisher={Springer Dordrecht},
  year={1988}
}

@book{clarke1983optimization,
  title={Optimization and Nonsmooth Analysis},
  author={Clarke, Frank H.},
  publisher={Wiley-Interscience},
  year={1983}
}

@inproceedings{
xie2024implicit,
title={Implicit Bias of {AdamW}: $\ell_\infty$-Norm Constrained Optimization},
author={Shuo Xie and Zhiyuan Li},
booktitle={Forty-first International Conference on Machine Learning},
year={2024},
url={https://openreview.net/forum?id=CmXkdlO6JJ}
}

@inproceedings{
zhang2024the,
title={The Implicit Bias of Adam on Separable Data},
author={Chenyang Zhang and Difan Zou and Yuan Cao},
booktitle={The Thirty-eighth Annual Conference on Neural Information Processing Systems},
year={2024},
url={https://openreview.net/forum?id=xRQxan3WkM}
}

@techreport{cifar,
  title={Learning multiple layers of features from tiny images},
  author={Krizhevsky, Alex and Hinton, Geoffrey and others},
  year={2009},
  publisher={Toronto, ON, Canada},
  institution = {University of Toronto},
}

@inproceedings{
moon2026minor,
title={Minor First, Major Last: A Depth-Induced Implicit Bias of Sharpness-Aware Minimization},
author={Chaewon Moon and Dongkuk Si and Chulhee Yun},
booktitle={The Fourteenth International Conference on Learning Representations},
year={2026},
url={https://openreview.net/forum?id=ErnnE2UNI2}
}

@inproceedings{NEURIPS2021_f4661398,
 author = {Pesme, Scott and Pillaud-Vivien, Loucas and Flammarion, Nicolas},
 booktitle = {Advances in Neural Information Processing Systems},
 noeditor = {M. Ranzato and A. Beygelzimer and Y. Dauphin and P.S. Liang and J. Wortman Vaughan},
 pages = {29218--29230},
 publisher = {Curran Associates, Inc.},
 title = {Implicit Bias of SGD for Diagonal Linear Networks: a Provable Benefit of Stochasticity},
 url = {https://proceedings.neurips.cc/paper_files/paper/2021/file/f4661398cb1a3abd3ffe58600bf11322-Paper.pdf},
 volume = {34},
 year = {2021}
}

@inproceedings{
gadhikar2024masks,
title={Masks, Signs, And Learning Rate Rewinding},
author={Advait Harshal Gadhikar and Rebekka Burkholz},
booktitle={The Twelfth International Conference on Learning Representations},
year={2024},
url={https://openreview.net/forum?id=qODvxQ8TXW}
}

@inproceedings{
gadhikar2026signin,
title={Sign-In to the Lottery: Reparameterizing Sparse Training},
author={Advait Gadhikar and Tom Jacobs and Chao Zhou and Rebekka Burkholz},
booktitle={The Thirty-ninth Annual Conference on Neural Information Processing Systems},
year={2026},
url={https://openreview.net/forum?id=iwKT7MEZZw}
}

@inproceedings{
kolb2025deep,
title={Deep Weight Factorization: Sparse Learning Through the Lens of Artificial Symmetries},
author={Chris Kolb and Tobias Weber and Bernd Bischl and David R{\"u}gamer},
booktitle={The Thirteenth International Conference on Learning Representations},
year={2025},
url={https://openreview.net/forum?id=vNdOHr7mn5}
}

@misc{kolb2025differentiablesparsitydgatingsimple,
      title={Differentiable Sparsity via $D$-Gating: Simple and Versatile Structured Penalization}, 
      author={Chris Kolb and Laetitia Frost and Bernd Bischl and David Rügamer},
      year={2025},
      eprint={2509.23898},
      archivePrefix={arXiv},
      primaryClass={cs.LG},
      url={https://arxiv.org/abs/2509.23898}, 
}
\bibliographystyle{plain}

\newpage 

\appendix

\listofappendices

\newpage

\section{Extended related work} 
\label{appendix : extended related work}
\addcontentsline{apx}{section}{\protect\numberline{\thesection}Extended related work}

\paragraph{Implicit regularization results.} 

Implicit bias of gradient-based optimization has been widely studied as an explanation for generalization in overparameterized models. Prior work shows that gradient descent induces structured solutions such as low-rank or sparse representations even without explicit regularization in regression settings~\citep{gunasekar2017implicitregularizationmatrixfactorization,arora2019implicitregularizationdeepmatrix,vaskevicius2019implicitregularizationoptimalsparse,li2021implicitsparseregularizationimpact, woodworth2020kernel}. Extensions analyze the role of depth, initialization, gradient noise, and training dynamics, including phase transitions such a grokking and alternative optimizers ~\citep{li2021resolvingimplicitbiasgradient,Marion2024DeepLN,azulay2021implicitbiasinitializationshape,lyu2024dichotomyearlylatephase,xie2024implicit,zhang2024the, moon2026minor, NEURIPS2021_f4661398}. More specifically for $\alpha =2$, the optimization problem in Theorem \ref{theorem : reform KKT main} resembles the cost minimization problem for infinite width two layer neural networks studied in \citep{Savarese2019HowDI}.

From a theoretical perspective, nonsmooth analysis provides the natural framework for studying subgradient and gradient flow dynamics in modern machine learning. Classical results on generalized gradients and differential inclusions~\citep{clarke1983optimization,filippov1988differential} formalize the behavior of optimization algorithms in the presence of nonsmooth objectives, and are particularly relevant for architectures involving nonsmooth components such as ReLU activations.

\paragraph{Sparse training.}
Mirror flow has been connected to reparameterized gradient flow \citep{Li2022ImplicitBO} and sparse training using a time-varying mirror flow \citep{jacobs2025mirror}. Sparsity has also been accomplished using deeper reparameterizations \citep{kolb2025deep} and for group sparsity \cite{kolb2025differentiablesparsitydgatingsimple, li2021resolvingimplicitbiasgradient}. These structures however do not correspond anymore to a mirror flow as they would violate conditions in \citep{Li2022ImplicitBO}. Moreover, mirror flow can also be used to promote sign flips which have been shown to be important in sparse training \citep{gadhikar2024masks, gadhikar2026signin}.

\paragraph{Hyperparameter transfer.} 
Mirror descent introduces hyperparameters beyond the learning rate, such as the scaling of the mirror map \citep{jacobs2026hyperbolic, Wu2021ImplicitRI}. 
Principled transfer of hyperparameters across model scales has been studied in the tensor programs framework \citep{pmlr-v139-yang21c}, 
which provides conditions under which learning rates and weight decay remain stable as the width grows. 
We extend this perspective to mirror descent, deriving hyperparameter transfer rules that ensure the corresponding max-margin solution remains reachable, with scaling that depends on the network width.

\newpage

\section{Proof of Lemma \ref{lemma : invariance balance}}
\addcontentsline{apx}{section}{\protect\numberline{\thesection}Proof of Lemma \ref{lemma : invariance balance}}
\label{app:proof-lemma-31}

\begin{proof}[Proof of Lemma \ref{lemma : invariance balance}] 
First note that the homogeneity of $f$ implies 
\begin{equation}
    \langle W_i, \nabla_{i} \mathcal{L}(\theta) \rangle_F =  \langle W_{i +1}, \nabla_{i+1} \mathcal{L}(\theta) \rangle_F , \label{eq:homogeneity-induced}
\end{equation}
where $\langle \cdot , \cdot \rangle_F$ denotes the Frobenius inner product. 
Indeed, homogeneity of $f$ implies that the reparameterization 
	\[
	(W_i, W_{i+1}) \mapsto (e^s W_i, e^{-s} W_{i+1})
	\]
leaves $f_\theta$ and hence $\mathcal{L}(\theta)$ invariant. Thus differentiating at $s=0$ yields 
	\[
	0 = \frac{d}{ds} \mathcal{L}(\theta(t))\Big|_{s=0}
	= \langle \nabla_i \mathcal{L}(\theta), W_i \rangle_F
	- \langle \nabla_{i+1} \mathcal{L}(\theta), W_{i+1} \rangle_F,
	\]
	which implies
	\[
	\langle W_i, \nabla_i \mathcal{L}(\theta) \rangle_F
	=
	\langle W_{i+1}, \nabla_{i+1} \mathcal{L}(\theta) \rangle_F. 
	\]

Now let us consider the evolution of the proposed conserved quantity. 
Since $R$ is layer-wise separable, so is $Q$. We have
\begin{align*}
    d \Big( Q_i (\nabla_{i} R_i(W_i)) - Q_{i+1}(\nabla_{i+1} R_{i+1}(W_{i+1}))\Big)
    &= \left\langle \nabla Q_i(\nabla_i R_i(W_i)), \, d\big(\nabla_i R_i(W_i)\big) \right\rangle \\
    &\quad - \left\langle \nabla Q_{i+1}(\nabla_{i+1} R_{i+1}(W_{i+1})), \, d\big(\nabla_{i+1} R_{i+1}(W_{i+1})\big) \right\rangle.
\end{align*}
By the definition of $Q_i$ (as the Legendre dual of $R_i$), we have
\[
\nabla Q_i(\nabla_i R_i(W_i)) = W_i, 
\qquad
\nabla Q_{i+1}(\nabla_{i+1} R_{i+1}(W_{i+1})) = W_{i+1}.
\]
Moreover, along gradient flow,
\[
d\big(\nabla_i R_i(W_i)\big) = - \nabla_i \mathcal{L}(\theta)\, dt,
\qquad
d\big(\nabla_{i+1} R_{i+1}(W_{i+1})\big) = - \nabla_{i+1} \mathcal{L}(\theta)\, dt.
\]
Substituting, we obtain
\begin{align*}
    d \Big( Q_i (\nabla_{i} R_i(W_i)) - Q_{i+1}(\nabla_{i+1} R_{i+1}(W_{i+1}))\Big)
    &= - \langle W_i, \nabla_i \mathcal{L}(\theta) \rangle_F \, dt
    + \langle W_{i+1}, \nabla_{i+1} \mathcal{L}(\theta) \rangle_F \, dt \\
    &= \Big( -\langle W_i, \nabla_i \mathcal{L}(\theta) \rangle_F 
    + \langle W_{i+1}, \nabla_{i+1} \mathcal{L}(\theta) \rangle_F \Big) dt \\
    &= 0,
\end{align*}
where the last equality follows from Eq.~(\ref{eq:homogeneity-induced}). 
This concludes the proof.
\end{proof} 

\newpage

\newpage

\section{Verification of Assumption \ref{assumption : semihomogeneous}} 
\label{appendix : assumption mirrors} 
\addcontentsline{apx}{section}{\protect\numberline{\thesection}Verification of Assumption \ref{assumption : semihomogeneous}}

We provide explicit calculations for the considered mirror maps presented in Table \ref{table:mirror-potentials} to verify that they satisfy Assumption \ref{assumption : semihomogeneous}. 
Note that since all mirror potentials are separable, we can verify the assumption pointwise. 
Note that both types of potentials are asymptotically homogeneous, with $\alpha =1$ for the Hyperbolic entropy, i.e., $R(\theta) \sim \| \theta\|_1$ for $\|\theta\| \rightarrow \infty$, 
and $\alpha = p$ for the smoothened homogeneous potentials, as then $R(\theta) \sim \|\theta\|_p^p$ for $\|\theta\| \rightarrow \infty$. 

\paragraph{Hyperbolic entropy.} 
Recall the hyperbolic entropy with parameter $\lambda$: 
\begin{equation*}
    R(\theta) = \theta \, \text{arcsinh}\Big(\frac{\theta}{\sqrt{\lambda}}\Big) - \sqrt{\theta^2 + \lambda} . 
\end{equation*}
Then we have to show that for all $\lambda \geq 0$: 
\begin{equation*}
    Q(\nabla R(\theta)) \geq \| \theta \|^2_{\nabla^2 R}.
\end{equation*}
Plugging in the exact quantities: 
\begin{equation*}
    \sqrt{\theta^2 + \lambda} \geq \frac{\theta^2}{\sqrt{\theta^2 + \lambda}} ,
\end{equation*}
where the denominator is always positive. 
Rearranging gives: 
\begin{equation*}
    \theta^2 + \lambda \geq \theta^2 , 
\end{equation*}
which is true for all $\lambda \geq 0$, verifying the assumption. 

\paragraph{Smoothened homogeneous potentials.}  
Recall the smoothened homogeneous potentials with parameter $\lambda \geq 0$:
\begin{equation*}
    R(\theta) = \frac{1}{p} | \theta|^p + \frac{\lambda}{2} \theta^2.
\end{equation*}
Then we have to show that for all $\lambda \geq 0$:
\begin{equation*}
    pQ(\nabla R(\theta)) \geq \| \theta \|^2_{\nabla^2 R}. 
\end{equation*}
Plugging in the exact quantities: 
\begin{equation*}
   (p-1) |\theta|^p +  \frac{p\lambda}{2} \theta^2 \geq (p-1)|\theta|^p + \lambda \theta^2 . 
\end{equation*}
Rearranging gives: 
\begin{equation*}
   \frac{p\lambda}{2} \theta^2  \geq \lambda \theta^2 , 
\end{equation*}
which is true for all $\lambda \geq 0$ and $p \geq 2$, verifying the assumption. 
This also includes homogeneous potentials with $\lambda = 0$ and therefore gradient flow. 

\newpage

\section{Proof of Theorem \ref{theorem : main result of margin convergence}} 
\addcontentsline{apx}{section}{\protect\numberline{\thesection}Proof of Theorem \ref{theorem : main result of margin convergence}}
\label{app:proof-Theorem45}

The proof of Theorem~\ref{theorem : main result of margin convergence} is divided into three lemmas. 
First we show that the soft-margin approximates the $Q$-margin as stated in the main text in Lemma \ref{lemma : soft margin equivalence}.

\begin{proof}[Proof of Lemma~\ref{lemma : soft margin equivalence}] 
Observe that $e^{a_{\text{max}}} \leq \sum_{i =1}^K e^{a_{i}}   \leq K e^{a_{\text{max}}}$ holds for $a_{\text{max}} = \max \{a_1 , \ldots , a_K \}$. 
Applying the log we have $a_{\text{max}}\leq \text{LSE}(a_1,  \ldots , a_K) \leq a_{\text{max}} + \text{log} K$. 
Now using the definitions of both $\tilde{\gamma}$ and $\gamma_Q$ gives $\gamma_Q(\theta) - \log K \ \left(\alpha Q(\nabla(R(\theta))\right)^{-L/\alpha}) \leq \tilde{\gamma}_Q(\theta) \leq \gamma_Q(\theta)$. Note that this proof does not rely on the particular definition of $Q$. 
\end{proof} 

Next, for Lemma \ref{lemma : main helper lemma} below 
we want to utilize the following inequality as in \citep{Lyu2020Gradient}: 
\begin{align}\label{equation : growth bound exp loss log}
    -\langle \theta_t , g_t \rangle_2 &=   \langle \theta_t , \sum_{i =1}^K e^{-q_i} \partial^{\circ} q_i \rangle_2 
    = L \sum_{i =1}^K e^{-q_i}  q_i 
    \geq  L \sum_{i =1}^K e^{-q_i}  q_{\text{min}}
     \geq L \mathcal{L}(\theta_t)\, \log\!\bigl(1/\mathcal{L}(\theta_t)\bigr) , 
\end{align}
where the last inequality follows from $e^{-q_{\text{min}}} \leq \mathcal{L}$. 

\begin{lemma}\label{lemma : main helper lemma}
    If the mirror potential $R$ satisfies Assumption \ref{assumption : semihomogeneous} and $\mathcal{L}(\theta_{t_0}) < 1$, then the $Q$-soft margin satisfies: 
    \begin{equation*}
        \frac{d}{dt} \log  Q(\nabla R(\theta_t)) > 0 \qquad \text{ and } \qquad \frac{d}{dt} \log  \tilde{\gamma} \geq L\left( \frac{d}{dt} \log  Q(\nabla R(\theta_t)) \right)^{-1}  E_{tan} \geq 0 , 
    \end{equation*}
    where $E_{tan} := \frac{1}{Q}\left( \|\frac{d}{dt} \theta_t \|_{\nabla^2 R}^2 - \frac{\langle \theta,\nabla^2 R(\theta_t ) \frac{d}{dt} \theta_t\rangle^2}{\|\theta \|_{\nabla^2 R}^2} \right) $. 
\end{lemma}
\begin{proof} 
We have 
\begin{align*}
    \frac{d}{dt} \log  \tilde{\gamma_t} & =  \frac{d}{dt} \log  \frac{\log (\frac{1}{\mathcal{L}(\theta_t)})}{\left(\alpha Q(\nabla(R(\theta_t))\right)^{L/\alpha}} \\
    &= \frac{d}{dt} \left( \text{log log}(1/\mathcal{L}(\theta_t)) - \frac{L}{\alpha} \log (\alpha Q(\nabla(R(\theta_t)))\right) \\
    &= \frac{-1}{\mathcal{L}(\theta_t)\log (1/\mathcal{L}(\theta_t))} \frac{d}{dt} \mathcal{L}(\theta_t) -\frac{L}{\alpha} \frac{1}{Q(\nabla(R(\theta_t))} \langle \theta_t, \frac{d \nabla R(\theta_t)}{dt}\rangle 
\end{align*}
Now define $\nu_t := \mathcal{L}(\theta_t)\log (1/\mathcal{L}(\theta_t))$, and note that $\nu_t \rightarrow 0$ if $\mathcal{L}(\theta_t) \rightarrow 0$. 
We use the same idea as in \citep[][Lemma 5.1]{Lyu2020Gradient} where the energy can be split into radial and tangential parts. 
We can use that $\frac{d}{dt} \mathcal{L}(\theta_t) = - \| \frac{d}{dt}\theta_t \|_{\nabla^2 R}^2$ and $\dot{Q} = \langle \theta_t, \frac{d \nabla R(\theta_t)}{dt}\rangle \geq L \nu_t > 0 $ to get:
 
\begin{align*}
    \frac{1}{\nu_t}\|\frac{d}{dt} \theta_t\|_{\nabla^2 R}^2  -\frac{L}{\alpha} \frac{1}{Q} \dot{Q} &= \frac{Q}{\nu_t} (E_{tan}) + \frac{1}{\nu_t} \frac{\langle \theta, \frac{d}{dt} \nabla R(\theta) \rangle^2}{\|\theta \|_{\nabla^2 R}^2} - \frac{L}{\alpha} \frac{\dot{Q}}{Q} \\
    &= \frac{Q}{\nu_t} (E_{tan}) + \frac{1}{\nu_t} \frac{\dot{Q}^2}{\|\theta \|_{\nabla^2 R}^2} - \frac{L}{\alpha} \frac{\dot{Q}}{Q} \\
    &=  \frac{Q}{\nu_t} (E_{tan}) + \frac{\dot{Q}}{\alpha Q} \left( \frac{\alpha Q\dot{Q}}{\nu_t  \|\theta \|^2_{\nabla^2 R}} - L \right) \\ &\geq  \frac{Q}{\nu_t}  (E_{tan}) + \frac{L\dot{Q}}{\alpha Q} \left( \frac{\alpha Q}{  \|\theta \|^2_{\nabla^2 R}} - 1 \right) \\
    &\geq \frac{Q}{\nu_t} (E_{tan}) \\
    &\geq \frac{LQ}{\dot{Q}}  (E_{tan})\\
    &= L\left( \frac{d}{dt} \log  Q(\nabla R(\theta_t)) \right)^{-1} E_{\text{tan}}\\
    &\geq 0 . 
\end{align*}
Here the main step was to use Assumption~\ref{assumption : semihomogeneous}. 
Note that the tangential part $E_{tan}$ is always non-negative. This can be seen from the fact that $Q$ is non-negative and non-decreasing and using the identity $\frac{d}{dt} \nabla R(\theta_t) = \nabla^2 R(\theta_t) \frac{d}{dt} \theta_t$ and the Cauchy-Schwarz inequality to bound: 
\begin{align*}
    QE_{tan} &:= \|\tfrac{d}{dt} \theta_t\|_{\nabla^2 R}^2 - \frac{\langle \theta, \frac{d}{dt} \nabla R(\theta_t) \rangle^2}{\|\theta \|_{\nabla^2 R}^2} \\
    &= \|\tfrac{d}{dt} \theta_t \|_{\nabla^2 R}^2 - \frac{\langle \theta,\nabla^2 R(\theta_t ) \frac{d}{dt} \theta_t\rangle^2}{ \|\theta \|_{\nabla^2 R}^2}
    \\
    &\geq \|\tfrac{d}{dt} \theta_t \|_{\nabla^2 R}^2 - \|\tfrac{d}{dt} \theta_t\|_{\nabla^2 R}^2 \\
    &= 0 .
\end{align*}
Moreover, the fact that $\frac{d}{dt}Q > 0$ implies that $\frac{d}{dt} \text{log} Q > 0$, which concludes the result.
\end{proof}

It remains to show that the loss goes to zero and consequently that the iterates diverge, as established in the next Lemma \ref{lemma : loss and Q growth}.

\begin{lemma}\label{lemma : loss and Q growth}
For all $t \geq 0$
\begin{equation*}
    G(\frac{1}{\mathcal{L}(\theta_t)}) \geq \frac{L^2}{\alpha} \tilde{\gamma}(t_0)^{\alpha/L} (t - t_0) \qquad \text{for } G(x) := \int^{x}_{1/\mathcal{L}(t_0)} \left(\log (u) \right)^{ \frac{\alpha}{L} -2}du . 
\end{equation*}
    The loss $\mathcal{L}(\theta_t) \rightarrow 0$ and $Q(\nabla R(\theta_t)) \rightarrow \infty$ as $t \rightarrow \infty$.
\end{lemma}
\begin{proof} 
We have 
\begin{align*}
    - \frac{d \mathcal{L} (\theta_t)}{dt} &= \| \frac{d}{dt}\theta_t\|_{\nabla^2 R}^2 \geq \langle \frac{\theta_t}{\| \theta_t \|_{\nabla^2 R}}, 
    \frac{d}{dt} \theta_t\rangle^2_{\nabla^2{R}} \\ &\geq \frac{1}{\alpha Q(\nabla R(\theta_t))} \langle\theta_t, 
    \frac{d}{dt} \theta_t\rangle^2_{\nabla^2{R}} \\ &\geq \frac{L^2}{\alpha}  \frac{\nu^2_t}{Q(\nabla R(\theta_t))} , 
\end{align*}
where we use the CS inequality, Assumption \ref{assumption : semihomogeneous} and the definition of $L \nu$.
Note this exactly matches the lower bound in Lemma B.6  of \citep{Lyu2020Gradient} in the case of gradient descent. 
We replace $\alpha Q(\nabla R(\theta_t))$ with $(\log (\frac{1}{L(\theta_t)})/\tilde{\gamma}_Q)^{\alpha/L}$ to get an expression in terms of the $Q$-soft margin and the loss:
\begin{align*}
    - \frac{d \mathcal{L} (\theta_t)}{dt} &\geq \frac{L^2}{\alpha} \left(\mathcal{L}(\theta_t) \log (\frac{1}{\mathcal{L}(\theta_t)}) \right)^2 \left(\tilde{\gamma}_Q(t)/\log (\frac{1}{\mathcal{L}(\theta_t)})\right)^{\alpha/L} \\ &\geq \frac{L^2}{\alpha} \mathcal{L}(\theta_t)^2 \left(\log (\frac{1}{\mathcal{L}(\theta_t)}) \right)^{2 - \frac{\alpha}{L}} \tilde{\gamma}_Q(t_0)^{\alpha/L} , 
\end{align*}
where the last inequality follows from monotonicity of the $Q-$soft margin. 
So the following holds for a.e.\ $t \geq t_0$: 
\begin{equation*}
    \left(\log (\frac{1}{\mathcal{L}(\theta_t)}) \right)^{ \frac{\alpha}{L} -2} \cdot \frac{d}{dt} \frac{1}{\mathcal{L}(\theta_t)} \geq \frac{L^2}{\alpha} \tilde{\gamma}(t_0)^{\alpha/L} .
\end{equation*}
Integrating on both sides from $t_0$ to $t$, we can conclude that
\begin{equation*}
  G(\frac{1}{\mathcal{L}(\theta_t)})  \geq L^2 \tilde{\gamma}(t_0)^{\alpha/L}(t - t_0) . 
\end{equation*}
Note that $1/\mathcal{L}(\theta_t)$ is non-decreasing. 
If $1/\mathcal{L}(\theta_t)$ does not grow to $+ \infty$, then neither does $G(1/\mathcal{L}(\theta_t))$. 
But the RHS grows to $+ \infty$, which leads to a contradiction. So $\mathcal{L}(\theta_t) \rightarrow 0$. 
To make $\mathcal{L}(\theta_t) \rightarrow 0$, $q_{\text{min}} : = \min_i y_i f(\theta, x_i)$ must converge to $+ \infty$. So $Q(\nabla R(\theta_t)) \rightarrow \infty$. 
%
\end{proof} 

\begin{proof}[Proof of Theorem~\ref{theorem : main result of margin convergence}]
The result now follows directly from Lemmas \ref{lemma : main helper lemma} and \ref{lemma : loss and Q growth} combined with Lemma~\ref{lemma : soft margin equivalence}. 
\end{proof}

\newpage

\section{Bounds on the normalized dual iterates}
\addcontentsline{apx}{section}{\protect\numberline{\thesection}Bounds on the normalized dual iterates}

The following two lemmas help us control the dual iterates, which in turn allows us to show convergence to KKT points in Appendix \ref{theorem : appendix KKT max margin result}. 
For Lemma \ref{lemma : tangential bound dual normalized}, we introduce the constants $B_0$ and $B_1$. As in Section C.3 of \citep{Lyu2020Gradient}, the function $q_n$ is locally Lipschitz and therefore Lipschitz on a compact set such that
\begin{equation*}
    B_0 := \sup \{\frac{q_i }{(\alpha Q(\nabla R(\theta)))^{L/\alpha}} : \theta \in \mathbb{R}^n \setminus \{0 \}, h \in \partial^{\circ} q_i , i \in [K]\} < \infty
\end{equation*}
and
\begin{equation*}
    B_1 := \sup \{\frac{\|h \|_2}{(\alpha Q(\nabla R(\theta)))^{(L-1)/\alpha}} : \theta \in \mathbb{R}^n \setminus \{0 \}, h \in \partial^{\circ} q_i , i \in [K]\} < \infty . 
\end{equation*}

\begin{lemma}\label{lemma : tangential bound dual normalized}
    The normalized dual derivative is bounded such that we have for a.e.\ $t > t_0$:
    \begin{equation*}
        \| \frac{d}{dt}\frac{\nabla R(\theta_t)}{\|\nabla R(\theta_t) \|_2}\|_2 \leq \frac{2B_1\langle \theta_t , \frac{d}{dt} \nabla R(\theta_t))\rangle}{L\tilde{\gamma}_Q C_2 \|\theta \|_2\| \nabla R(\theta_t)\|_2}  , 
    \end{equation*}
    where $C_2 > 0$ is the constant such that $(\alpha Q(\nabla R(\theta)))^{1/\alpha} \geq C_2 \| \theta \|_2$ for all $\theta \in \mathbb{R}^n$. 
\end{lemma}

\begin{proof} 
By Cauchy-Schwarz:
\begin{align*}
    \| \frac{d}{dt}\frac{\nabla R(\theta_t)}{\|\nabla R(\theta_t) \|_2}\|_2 &= \|\frac{\nabla^2 R(\theta_t)\frac{d}{dt}\theta_t}{\|\nabla R(\theta_t) \|}  - \frac{\nabla R(\theta_t) \langle\nabla R(\theta_t),\nabla^2 R(\theta_t)\frac{d}{dt}\theta_t\rangle}{\|\nabla R(\theta_t) \|^3}\|_2 \\
    &\leq 2\frac{\|\nabla^2 R(\theta_t) \frac{d}{dt} \theta_t \|}{\|\nabla R(\theta_t) \|_2} . 
\end{align*}
We now use the definition of the mirror flow and a bound for the gradient. By the chain rule there exists $h_1, \hdots, h_K : [0,\infty) \rightarrow \mathbb{R}^n$ satisfying that for a.e. $t > 0$ $h_{i,t} \in \partial^{\circ} q_n$ and $\nabla^2 R(\theta_t) \frac{d\theta_t}{dt} =\sum_{i =1}^K e^{-q_i} h_{i,t} $. Now using $B_1$:
\begin{align*}
    \|\nabla^2 R(\theta_t) \frac{d}{dt} \theta_t \|_2 &\leq \sum_{i =1}^K e^{-q_n}\|h_i\|_2 \\
    &\leq \sum_{i =1}^K e^{-q_i}q_i \frac{1}{q_i} B_1 (\alpha Q)^{\frac{L-1}{\alpha}}.
\end{align*}
We can use that $q_i \geq q_{\text{min}} \geq \log  \frac{1}{\mathcal{L}}$ and that $\sum_{i =1}^K e^{-q_i}q_i = \frac{1}{L} \langle \theta_t , \frac{d}{dt} \nabla R(\theta_t))\rangle$: 
\begin{align*}
    \|\nabla^2 R(\theta_t) \frac{d}{dt} \theta_t \|_2 &\leq \frac{B_1 \langle \theta_t , \frac{d}{dt} \nabla R(\theta_t))\rangle }{L \log 1/\mathcal{L}} (\alpha Q)^{\frac{L-1}{\alpha}}\\ &\leq \frac{B_1\langle \theta_t , \frac{d}{dt} \nabla R(\theta_t))\rangle}{L\tilde{\gamma}_Q (\alpha Q)^{1/\alpha}} . 
\end{align*}
Next we can put this together giving and using that there exists a constant $C_2$ such that $(\alpha Q)^{1/\alpha} \geq C_2\| \theta \|_2$ for all $\theta \in \mathbb{R}^n$:
\begin{equation*}
    \| \frac{d}{dt}\frac{\nabla R(\theta_t)}{\|\nabla R(\theta_t) \|_2}\|_2 \leq \frac{2B_1\langle \theta_t , \frac{d}{dt} \nabla R(\theta_t))\rangle}{L\tilde{\gamma}_Q (\alpha Q)^{1/\alpha}\| \nabla R(\theta_t)\|_2} \leq \frac{2B_1\langle \theta_t , \frac{d}{dt} \nabla R(\theta_t))\rangle}{L\tilde{\gamma}_Q C_2 \|\theta \|_2\| \nabla R(\theta_t)\|_2} 
\end{equation*}
This concludes the proof. 
\end{proof} 

\begin{lemma}\label{lemma : dual log Q bound}
There exists a time $t_1$ such that for a.e. $t > t_1$:
\begin{equation*}
    \frac{d}{dt}\log (Q(\nabla R(\theta_t))) \geq \frac{\langle \theta_t , \frac{d}{dt}\nabla R(\theta_t) \rangle}{\|\theta_t \|\| \nabla R(\theta_t)\|} . 
\end{equation*}
\end{lemma}

\begin{proof} 
Since $Q \rightarrow \infty$ we have that $R \rightarrow \infty$ by definition of the mirror map.
Therefore, there exists a time $t_1$ such that for all $t > t_1$ $R > 0$.
Therefore, we have for all $t > t_1$:
\begin{equation*}
    \frac{d}{d t} \log (Q(\nabla R(\theta_t))) = \frac{\langle \theta_t, \frac{d}{dt} \nabla R(\theta_t)\rangle}{Q} \geq \frac{\langle \theta_t, \frac{d}{dt} \nabla R(\theta_t)\rangle}{\langle \theta_t, \nabla R(\theta_t)\rangle} \geq \frac{\langle \theta_t , \frac{d}{dt}\nabla R(\theta_t) \rangle}{\|\theta_t \| \| \nabla R(\theta_t)\|}.
\end{equation*}
Note from Lemma~\ref{lemma : tangential bound dual normalized} we know that the right hand side is greater than zero. 
\end{proof}

\section{Proof of Theorem \ref{theorem : convergence and growth rates}}
\addcontentsline{apx}{section}{\protect\numberline{\thesection}Proof of Theorem \ref{theorem : convergence and growth rates}}\label{section : growth rates}

This proof relies on the characterization of $G$ as in Lemma~\ref{lemma : loss and Q growth}. 

\begin{lemma}
\label{lemma : growth of G and G inverse}
    For $G$ defined in Lemma~\ref{lemma : loss and Q growth} and its inverse, we have the following bounds: 
    \begin{equation*}
        G(x) = \Theta \left((\text{log}\ x)^{\alpha/L-2}x\right) \text{ and } G^{-1}(y) = \Theta\left((\text{log} \ x)^{2-\alpha/L}x\right) , 
    \end{equation*}
    where $x = G^{-1}(y)$.
\end{lemma}
\begin{proof} 
Denote $\beta: = \frac{\alpha}{L}-2$
By $u = e^{z}$ substitution and iterative integration by parts after we get:
\begin{align*}
     \int^{x}_{1/\mathcal{L}(t_0)} \left(\text{log}(u) \right)^{ \beta}du &= \int^{\log(x)}  z^{\beta} e^{z} dz \\ &= \text{log}(x)^{\beta} x -\beta \int^{\text{log}x} z^{\beta -1} e^{z}dz 
     \\ &= \text{log}(x)^{\beta} x  \left(1 -  \frac{\beta}{\text{log}(x)} + ... \right) \\
     &= \Theta \left( (\text{log}\ x)^{\alpha/L-2}x\right)
\end{align*}

Now for $G^{-1}(y)$, let $x = G^{-1}(y)$ for $y \geq 0$. $G(x)$ is finite for $x$ finite, thus if $x \rightarrow \infty$ we have $y \rightarrow \infty$. From the first part we get: $y = \Theta \left( (\text{log} \ x)^{\alpha/L-2}x\right)$. Taking logarithms on each side gives: $\text{log} \ y = \Theta\left(\text{log} \ x\right)$. Therefore, we have $x =\Theta\left((\text{log}\ x)^{2-\alpha/L}x\right)$. 
\end{proof}

Now we can prove Theorem \ref{theorem : convergence and growth rates}. 

\begin{proof}[Proof of Theorem \ref{theorem : convergence and growth rates}]

\textbf{Upper bounding $\mathcal{L}$.} 
It follows from Lemma \ref{lemma : loss and Q growth} that we have $\frac{1}{\mathcal{L}} \geq  G^{-1}(\Omega(t))$. 
Using Lemma \ref{lemma : growth of G and G inverse} we have $\frac{1}{\mathcal{L}} = \Omega\left((\text{log}\ t)^{2-\alpha/L}t\right)$.

\textbf{Bounding $(\alpha Q)^{1/\alpha}$ in terms of $\mathcal{L}$.}
We have $\tilde{\gamma}_Q(t) \geq \tilde{\gamma}_Q(t_0)$, so $(\alpha Q)^{L/\alpha} \leq \frac{1}{\tilde{\gamma}_Q(t_0)}\text{log} \frac{1}{\mathcal{L}}$. On the other hand, $\text{log} \frac{1}{\mathcal{L}} \leq q_{\text{min}} \leq B_0 (\alpha Q)^{L/\alpha}$, so $(\alpha Q)^{L/\alpha} = \Omega(\text{log}\frac{1}{\mathcal{L}})$. This implies that we have:
\begin{equation}\label{equation : bound Q and L}
    (\alpha Q)^{L/\alpha} = \Theta (\text{log} \frac{1}{\mathcal{L}}) . 
\end{equation}

\textbf{Lower bounding $\mathcal{L}$.}
Let $h_1, \hdots , h_K$ be a set of vectors such that $h_i \in \frac{\partial q_i}{\partial \theta}$ and
\begin{equation*}
    \nabla^2 R(\theta) \frac{d \theta}{dt} = \sum_{i =1}^K e^{-q_i}h_i.
\end{equation*}
We have that $\|h_i\|_2 \leq B_1 Q^{\frac{L-1}{\alpha}} = O((\text{log} \frac{1}{\mathcal{L}})^{1-1/L})$.
Note that we have $|\partial_i^2 R^{-1}(\theta) |=  O(\max \{1,(\text{log} \frac{1}{\mathcal{L}})^{(2-\alpha)/L}\})$ by asymptotic $\alpha$-homogeneity and separability of the potential $R$. Thus: 
\begin{align*}
    -\frac{d \mathcal{L}}{dt} &= \|\frac{d\theta}{dt} \|_{\nabla^2 R}^2 
    = \langle \nabla^2 R(\theta)\frac{d\theta}{dt}  , \nabla^2 R^{-1}(\theta) \nabla^2 R(\theta) \frac{d\theta} {dt}  \rangle_2   
    \leq \max_{i \in [K]} |\nabla^2 R^{-1}(\theta)| \ \|\nabla^2 R(\theta)\frac{d\theta}{dt} \|_2^2\\ 
    &\leq O(\max \{1,(\text{log} \frac{1}{\mathcal{L}})^{(2-\alpha)/L}\})\| \sum_{i =1}^K e^{-q_i}h_i \|_2^2 
\\ &\leq  O(\max \{1,(\text{log} \frac{1}{\mathcal{L}})^{(2-\alpha)/L}\}) \left(\sum_{i =1}^K e^{-q_i}\max_{i \in [K]}\|h_i\|_2\right)^2 \\
    &\leq O(\max \{1,(\text{log} \frac{1}{\mathcal{L}})^{(2-\alpha)/L}\})\mathcal{L}^2 O((\text{log} \frac{1}{\mathcal{L}})^{2-2/L})) 
    = \mathcal{L}^2 O((\text{log} \frac{1}{\mathcal{L}})^{2-\min\{2,\alpha\}/L})) . 
\end{align*}
Now if $\alpha \leq 2$, rearranging and with the definition of $G$, there exist a constant $c$ such that $\frac{d}{dt} G(\frac{1}{\mathcal{L}}) \leq c$ for any $\mathcal{L}$ that is small enough. The proof follows directly from applying Lemma \ref{lemma : growth of G and G inverse}.

\paragraph{Bounding $(\alpha Q)^{1/\alpha}$ in terms of $t$.} By Eq.~(\ref{equation : bound Q and L}) and the tight bounds for $\mathcal{L}$ for $\alpha \leq 2$ we have that $(\alpha Q)^{L/\alpha} =  \Theta (\text{log} \frac{1}{\mathcal{L}}) = \Theta(\text{log}\ t)$.
For $\alpha > 2$ we have that the lowerbound on $\mathcal{L}$ fails, however, we still have an upperbound for $\mathcal{L}$ giving us $(\alpha Q)^{L/\alpha} = \Theta (\text{log} \frac{1}{\mathcal{L}}) = \Omega(\text{log}\ t)$. 
\end{proof}


\section{Proof of Theorem \ref{theorem : KKT max margin result}}

\addcontentsline{apx}{section}{\protect\numberline{\thesection}Proof of Theorem \ref{theorem : KKT max margin result}}  

Here we show the KKT conditions for the Homogeneous potentials with $\alpha = p \geq 2$. We can be more general and replace it by the following condition:

\begin{assumption}\label{assumption : approximation of horizon in norm}
    For $R$ and corresponding $\phi_{\alpha}$ there exists a fixed constant $C > 0$ such that we have when $\eta \sim \|\theta\| \rightarrow \infty$:
    \begin{equation*}
        \|\partial^\circ\phi_{\alpha}^2 (\theta)/\eta -C (\alpha Q(\nabla R(\theta)))^{2/\alpha - 1} \nabla^2 R(\theta)\theta /\eta \|_2^2 \rightarrow 0 . 
    \end{equation*}
\end{assumption}

Assumption \ref{assumption : approximation of horizon in norm} is satisfied by the homogeneous potentials but not by the hyperbolic entropy as detailed in Appendix \ref{kkt assumption}.

Now we want to construct similarly as in \citep{tsilivis2024flavors, Lyu2020Gradient} a limiting sequence of KKT points. 
For this we first recall the definitions of KKT points, approximate KKT points and constraint qualification. Next, we bound the approximate KKT points by the ``angle'' of the iterates with respect to the evolution and denoted by $\beta_t$ in Lemma \ref{lemma : bound angle approx constraints}. We can then control the angle decay in terms of the iterate growth in Corollary \ref{corollary : growth of beta}. Finally, we can construct the sequence of KKT points for which the angle decays. 

In this section, we will show the following result, which covers and extends Theorem~\ref{theorem : KKT max margin result}. 
\begin{theorem}\label{theorem : appendix KKT max margin result}
    Assume that $R$ satisfies Assumption \ref{assumption : approximation of horizon in norm} and $\alpha \geq 2$, under the same assumptions as in Theorem~\ref{theorem : main result of margin convergence}, the limit points $\bar{\theta}$ of  $\{\frac{\theta_t}{\| \theta_t \|_2}, t>0 \}$ converge in direction to a solution to the following optimization problem: 
    \begin{equation}\label{equation : max margin horizon}
        \min_{\theta \in \mathbb{R}^n} \, \frac{1}{2}\phi^2_{\alpha}(\theta) \qquad \text{such that} \qquad y_i f(\theta, x_i) \geq 1, \text{ for all } i \in [K],  
    \end{equation}
    where $\phi_{\alpha} : = \lim_{\eta \rightarrow 0} \eta (\alpha Q(\nabla R(\theta /\eta)))^{1/\alpha}$ is the horizon function.
\end{theorem}

\subsection{KKT conditions}

\addcontentsline{apx}{subsection}{\protect\numberline{\thesubsection}KKT conditions}  

To show the result we first need to state what the corresponding optimality conditions i.e. Karush Kuhn Tucker (KKT). The main idea is to show is to create a sequence of appproximate KKT points which under the  Mangasarian-Fromovitz
Constraint Qualification (MFCQ) implies the that limit is a KKT point.

\begin{definition}(KKT point)
    A feasible point $\theta$ of Eq.(\ref{equation : max margin horizon}) is a KKT point if there exists $\lambda_1, \hdots , \lambda_K \geq 0$ such that
    \begin{itemize}
        \item $\partial^{\circ} \frac{1}{2}\phi^2_{\alpha} (\theta) - \sum_{i =1}^K \lambda_i h_i$ for some $h_1, \hdots , h_K$ satisfying $h_i \in \partial^{\circ} q_i(\theta)$.
        \item For all $i \in [K]$: $\lambda_i(q_i(\theta)-1) = 0$. 
        
    \end{itemize}
\end{definition}

\begin{definition}(Approximate KKT point)
    A feasible point $\theta$ of Eq.~(\ref{equation : max margin horizon}) is an $(\epsilon,\delta)$-KKT point of Eq.~(\ref{equation : max margin horizon}) if there exists $\lambda_1, \hdots , \lambda_K \geq 0$ such that
    \begin{itemize}
        \item $\|\partial^{\circ} \frac{1}{2} \phi^2_{\alpha} (\theta) - \sum_{i =1}^K \lambda_i h_i \|_2 \leq \epsilon$ for some $h_1, \hdots , h_K$ satisfying $h_i \in \partial^{\circ}$.
        \item For all $i \in [K]$: $\lambda_i(q_i(\theta)-1) \leq \delta$.  
    \end{itemize}
\end{definition}

\begin{lemma}
    Eq.~(\ref{equation : max margin horizon}) satisfies the MFCQ at every feasible point $\theta$.
\end{lemma}
\begin{proof} 
See Lemma C.7 in \citep{Lyu2020Gradient}.
\end{proof}

\subsection{Key lemmas}

\addcontentsline{apx}{subsection}{\protect\numberline{\thesubsection}Key Lemmas}  

In order to show that the sequence of KKT points converges we need to show that the angle quantified by the constant $\beta_t$ goes to $1$:
\begin{equation*}
    \beta_t : = \frac{1}{\|\frac{d \theta_t}{dt} \|_{{\nabla}^2R}} \langle \frac{\theta_t}{\| \theta_t\|_{\nabla^2 R}},{\nabla}^2R (\theta) \frac{d \theta_t}{dt}\rangle \leq 1 . 
\end{equation*}

\begin{lemma}
\label{lemma : bound angle approx constraints} 
    Let $C_1$ and $C_2$ be two constants defined as:
    \begin{equation*}
        C_1 := \sqrt{ \frac{2}{\tilde{\gamma}_Q^{2/L}} \hat{\mu}} , \qquad C_2 : = \frac{\tilde{\gamma}^{-2/L}_Q}{L } K / e .
    \end{equation*}
    where $\hat{\mu}$ is a growth constant depending on $R$.
    Then $\tilde{\theta}:= \theta/q_{\text{min}}(\theta)^{1/L} $ is an $(\epsilon, \delta)-$KKT point of Eq.~(\ref{equation : max margin horizon}) where $\epsilon := C_1 \sqrt{1- \beta}$ and $\delta:= C_2/\log ( 1/\mathcal{L})$. 
\end{lemma}
\begin{proof} 
For the first condition we approximate the horizon function, then we show that the approximation can be bounded by the angle $\beta_t$. Notice that by construction the horizon function is $1$-homogeneous this gives us: 
\begin{equation*}
    \partial^{\circ} \frac{1}{2} \phi^2_{\alpha} (\theta/q_{\text{min}}(\theta)^{1/L})  = \partial^{\circ}\frac{1
    }{2} \phi_{\alpha}^2 (\theta)/q_{\text{min}}(\theta)^{1/L} .
\end{equation*}
This for $\alpha \geq 2$ then is well approximated by $(\alpha Q(\nabla R(\theta)))^{2/\alpha - 1} \nabla^2 R(\theta)\theta / q_{\text{min}}(\theta)^{1/L}$ as $\| \theta\| \rightarrow \infty$ by Assumption \ref{assumption : approximation of horizon in norm} and the fact that $q_{\text{min}} \sim \| \theta\|$ as $\| \theta\| \rightarrow\infty$.

Now, denote $h_t : = \frac{d \theta_t}{dt}$ for a.e. $t > 0$. Then, by the chain rule, there exists $h_1, \hdots, h_N$ such that $h_i \in \partial^{\circ} q_i$ and $h = \nabla^2 R^{-1}(\theta_t) \sum_{i=1}^N e^{-q_i} h_i$. Denote  $\tilde{h}_i := h_i / q_{\text{min}}^{1- 1/L} $ and $\zeta : = (\alpha Q(\nabla R(\theta)))^{2/\alpha - 1} \nabla^2 R(\theta) \theta / q_{\text{min}}^{1/{L}}$. Construct $\lambda_i = (\alpha Q(\nabla R(\theta)))^{2/\alpha - 1}  q_{\text{min}}^{1- 2/L} \| \theta\|_{\nabla^2 R} e^{-q_i} / \| h \|_{\nabla^2 R}$.
Then 
\begin{align*}
        \|\zeta - \sum_i \lambda_i \tilde{h}_i \|^2_2 &=  (\alpha Q(\nabla R(\theta)))^{4/\alpha - 2}/q_{\text{min}}^{2/L} \left\| \nabla^2 R(\theta) \theta - \frac{\| \theta\|_{\nabla^2 R}}{\| h\|_{\nabla^2 R}}  \nabla^2 R(\theta) h  \right\|_{2}^2\\
        &\leq (\alpha Q(\nabla R(\theta)))^{4/\alpha -1} /q_{\text{min}}^{2/L} \| \nabla^2 R(\theta) \theta /\| \theta\|_{\nabla^2 R}- \frac{1}{\| h\|_{\nabla^2 R}}  \nabla^2 R(\theta) h  \|_{2}^2 \\
        &\leq \frac{1}{\tilde{\gamma}_Q^{2/L}} / (\alpha Q(\nabla R(\theta)))^{1-2/\alpha} \left\| \nabla^2 R(\theta) \theta /\| \theta\|_{\nabla^2 R}- \frac{1}{\| h\|_{\nabla^2 R}}  \nabla^2 R(\theta) h  \right\|_{\nabla^2 R(\theta)}^2\\
        &\leq \frac{1}{\tilde{\gamma}_Q^{2/L}} \mu/ (\alpha Q(\nabla R(\theta)))^{1-2/\alpha} \left\|  \theta /\| \theta\|_{\nabla^2 R}- \frac{1}{\| h\|_{\nabla^2 R}}   h  \right\|_{\nabla^2 R(\theta)}^2 \\
        &\leq \frac{2}{\tilde{\gamma}_Q^{2/L}} \hat{\mu} (1-\beta) , 
\end{align*}
where we used the bound $\mu := (\max_i(1, |\nabla^2 R(\theta_i)|))$ and $\hat{\mu} = \mu Q^{\frac{2-\alpha}{\alpha}} < \infty$, when $\alpha \geq 2$.

Similarly we can bound the other quantity with our construction and homogeneity we have (note that $q_i = y_i f(\theta,x_i)$):
\begin{align*}
    \sum_{i =1}^K \lambda_i(q_i(\tilde{\theta}) - 1) &= \frac{q_{\text{min}}^{-2/L}(\alpha Q)^{2/\alpha -1}\|\theta\|_{\nabla^2 R}}{\|h\|_{\nabla^2 R}}  \sum_{i =1}^K e^{-q_i}(q_i - q_{\text{min}}) . 
\end{align*}
Now we can use that $\| h \|_{\nabla^2 R}\geq \langle h , \frac{\theta}{\|\theta\|_{\nabla^2 R}} \rangle_{\nabla^2 R} \geq L \nu/\|\theta\|_{\nabla^2 R}$ and note that
\begin{equation*}
    \nu = \mathcal{L}\  \log (1/\mathcal{L}) \geq e^{-q_{\text{min}}}  \log (1/\mathcal{L}) = e^{-q_{\text{min}}}  \tilde{\gamma}_Q (\alpha Q)^{L/\alpha } .
\end{equation*}
Combining this gives:
\begin{align*}
     \sum_{i =1}^K \lambda_i(q_i(\tilde{\theta}) - 1) &\leq \frac{q_{\text{min}}^{-2/L}(\alpha Q)^{2/\alpha}}{L\tilde{\gamma}_Q  (\alpha Q)^{(L)/\alpha}}  \sum_{i =1}^K e^{-(q_i-q_{\text{min}})}(q_i - q_{\text{min}}) \\ &\leq \frac{\tilde{\gamma}^{-2/L}_Q}{L \tilde{\gamma}_Q (\alpha Q)^{(L)/\alpha}}  \sum_{i =1}^K e^{-(q_i-q_{\text{min}})}(q_i - q_{\text{min}}) \\
     &= \frac{\tilde{\gamma}^{-2/L}_Q}{L \tilde{\gamma}_Q (\alpha Q)^{(L)/\alpha}} K / e \\
     &= \frac{\tilde{\gamma}^{-2/L}_Q}{L } K / e \cdot \frac{1}{\log  1/\mathcal{L}}   , 
\end{align*}
where we used that $e^{-x} x$ has its maximum $e^{-1}$ at $x = 1$.
This concludes the proof. 
\end{proof}


\begin{lemma}\label{lemma : beta bound}
For all $t_2 > t_1 \geq t_0$,
\begin{equation*}
    \int_{t_1}^{t_2} \left(\beta_{\tau}^{-2} -1 \right) \frac{d}{d\tau} \log  Q(\nabla R(\theta_{\tau})) d \tau \leq \frac{\alpha}{L} \log  \frac{\tilde{\gamma}_Q(t_2)}{\tilde{\gamma}_Q(t_1)}.
\end{equation*}
\end{lemma}
\begin{proof} 
From Lemma \ref{lemma : main helper lemma} we have that for all $t \in (t_1, t_2)$:
\begin{equation*}
    \frac{d}{dt} \log  \tilde{\gamma} \geq L\left( \frac{d}{dt} \log  Q(\nabla R(\theta_t)) \right) Q^2/\dot{Q}^2 E_{tan} \
\end{equation*}
We can now lowerbound $Q^2/\dot{Q}^2 E_{tan}$ by using the definitions of $E_{\text{tan}}$ and $\dot{Q} = $
\begin{align*}
    Q^2/\dot{Q}^2 E_{tan} 
    &= \frac{Q}{\langle \theta,\nabla^2 R(\theta_t ) \frac{d}{dt} \theta_t\rangle^2}\left(\|\frac{d}{dt} \theta_t\|_{\nabla^2 R}^2 - \frac{\langle \theta,\nabla^2 R(\theta_t ) \frac{d}{dt} \theta_t\rangle^2}{\|\theta \|_{\nabla^2 R}^2} \right) \\ 
    &\geq  \frac{1}{\alpha}\frac{\|\theta \|^2_{\nabla^2 R}}{\langle \theta,\nabla^2 R(\theta_t ) \frac{d}{dt} \theta_t\rangle^2}\left(\|\frac{d}{dt} \theta_t\|_{\nabla^2 R}^2 - \frac{\langle \theta,\nabla^2 R(\theta_t ) \frac{d}{dt} \theta_t\rangle^2}{\|\theta \|_{\nabla^2 R}^2} \right) \\ 
    &= \frac{1}{\alpha}\left( \beta^{-2} -1 \right) . 
\end{align*}
Together with Lemma \ref{lemma : main helper lemma} we have:
\begin{equation*}
    \frac{d}{dt} \log  \tilde{\gamma}_Q \geq \frac{L}{\alpha}\left( \beta^{-2} -1 \right)  \frac{d}{dt} \log  Q.
\end{equation*}
Integrating both sides concludes the result. 
\end{proof} 

\begin{corollary}\label{corollary : growth of beta}
    For all $t_2 > t_1 \geq t_0$, then there exists $t_* \in (t_1, t_2)$ such that:
    \begin{equation*}
        \beta_{t_*}^{-2} - 1 \leq \frac{\alpha}{L}\cdot \frac{\log\tilde{\gamma}_Q(t_2)-\log\tilde{\gamma}_Q(t_1)}{\log Q(\nabla R(\theta_{t_2})) - \log Q(\nabla R(\theta_{t_1}))}
    \end{equation*}
\end{corollary}
\begin{proof} 
We follow the exact same steps as \citep[Corollary C.10]{Lyu2020Gradient}. Denote the RHS as $C$. Assume the opposite is true i.e. $\beta_{\tau}^{-2} -1 > C$ for a.e. $\tau \in (t_1, t_2)$. From Lemma \ref{lemma : main helper lemma} we know $\log (Q) > 0$ for a.e.\ $\tau \in (t_1, t_2)$. 
Then, by Lemma \ref{lemma : beta bound}: 
\begin{equation*}
    \frac{\alpha}{L}  \log  \frac{\tilde{\gamma}_Q(t_2)}{\tilde{\gamma}_Q(t_1)} > \int_{t_1}^{t_2} C \frac{d}{d\tau} \log  Q(\nabla R(\theta_{\tau})) d\tau = C \ \log  Q(\nabla R(\theta_{t_2})) - \log  Q(\nabla R(\theta_{t_1})) = \frac{\alpha}{L} \log  \frac{\tilde{\gamma}_Q(t_2)}{\tilde{\gamma}_Q(t_1)},
\end{equation*}
which is a contradiction. 
\end{proof}

Now we are ready to construct the limiting sequence.

\begin{lemma}\label{lemma : limit point}
    For every limit point $(\bar{\theta}, \bar{g})$ of $\{(\frac{\theta_t}{\| \theta_t \|_2}, \frac{\nabla R(\theta_t)}{\|\nabla R(\theta_t)\|_2}) : t \geq 0 \}$, there exists a sequence of $\{ t_m : m \in \mathbb{N}\}$ such that $t_m \uparrow \infty$, $(\frac{\theta_t}{\| \theta_t \|_2}, \frac{\nabla R(\theta_t)}{\|\nabla R(\theta_t)\|_2}) \rightarrow (\bar{\theta}, \bar{g})$ and $\beta_{t_m} \rightarrow 1$, where $\bar{g} \in c\ \partial^{\circ} \phi_{\alpha}(\bar{\theta})$ for $c > 0$.
\end{lemma}
\begin{proof} 
Let $\{\epsilon_m : m \in \mathbb{N} \}$ be an arbitrary sequence with $\epsilon_m \rightarrow 0$. Now we construct $\{t_m \}$ by induction. Suppose $t_1 < t_2 < \hdots < t_{m-1}$ have already been constructed. Since $\bar{\theta}$ is a limit point and $\tilde{\gamma}_{Q,t} \uparrow \tilde{\gamma}_{Q,\infty}$, there exists a $s_m > t_{m-1}$ such that:
\begin{equation*}
    \|\frac{\theta_{s_m}}{\| \theta_{s_m}\|_2}  - \bar{\theta}\|_2 \leq \epsilon_m, \quad  \|\frac{\nabla(\theta_{s_m})}{\| \nabla R(\theta_{s_m})\|_2}  - \bar{g}\|_2 \leq \epsilon_m, \quad \text{and, } \quad \frac{\alpha}{L} \log  \frac{\tilde{\gamma}_{Q,\infty}}{\tilde{\gamma}_{Q,s_m}} \leq \epsilon_m^3.
\end{equation*}
The relationship between $\bar{g}$ and $\bar{\theta}$ through $\bar{g} \in c\ \partial^{\circ} \phi_{\alpha}(\bar{\theta})$ for $c > 0$ follows from Corollary 2 in \citep{pesme2024implicit}.
Now let $s_m' > s_m$ be a time such that $\log  Q(s'_m) = \log  Q(s_m) + \epsilon_m$. According to Theorem \ref{theorem : main result of margin convergence}, $\log  Q \rightarrow \infty$, so $s_m'$ must exist. We construct $t_m \in (s_m, s_m')$ to be the time that $\beta^{-2}_{t_m} - 1 \leq \epsilon_m^2$ where the existence follows from Corollary \ref{corollary : growth of beta}.
It follows that $\beta_{t_m} \geq 1/\sqrt{1+\epsilon_m^2} \rightarrow 1$. Moreover, by Lemmas \ref{lemma : dual log Q bound} and \ref{lemma : tangential bound dual normalized} we have:
\begin{align*}
    \|\frac{\nabla R(\theta_{t_m})}{\|\nabla R(\theta_{t_m})\|_2} - \bar{g} \|_2 &\leq  \|\frac{\nabla R(\theta_{t_m})}{\|\nabla R(\theta_{t_m})\|_2} - \frac{\nabla R(\theta_{s_m})}{\|\nabla R(\theta_{s_m})\|_2} \|_2 + \|\frac{\nabla R(\theta_{s_m})}{\|\nabla R(\theta_{s_m})\|_2} - \bar{g} \|_2 \\ &\leq \frac{2B_1}{LC_2 \tilde{\gamma}_{Q,t_0}} \epsilon_m + \epsilon_m \rightarrow 0.
\end{align*}
Since $\nabla Q$ is a diffeomorphism that has the same growth in all directions due to being separable and identical for each coordinate we also have $\theta_{t_m}/\|\theta_{t_m}\|_2 \rightarrow \bar{\theta}$. 
To see this, we can express the limit of $\frac{\theta_t}{\| \theta_t \|_2}$ as follows: 
\begin{align*}
    \lim_{t \rightarrow \infty} \frac{\theta_t}{\| \theta_t \|_2} &= \lim_{t \rightarrow \infty} \frac{\nabla Q( \nabla R(\theta_t))}{\|\nabla Q( \nabla R(\theta_t))\|_2} \\
    &=  \lim_{t \rightarrow \infty} \frac{\frac{1}{\|\nabla R(\theta_t)\|_2}\nabla Q( \nabla R(\theta_t))}{\|\frac{1}{\|\nabla R(\theta_t)\|_2}\nabla  Q( \nabla R(\theta_t))\|_2} \\
    &= \lim_{t \rightarrow \infty} \frac{\frac{1}{\|\nabla R(\theta_t)\|_2}\nabla Q(\|\nabla R(\theta_t)\|_2\frac{1}{\|\nabla R(\theta_t)\|_2} \nabla R(\theta_t))}{\|\frac{1}{\|\nabla R(\theta_t)\|_2}\nabla Q(\|\nabla R(\theta_t)\|_2 \frac{1}{\|\nabla R(\theta_t)\|_2}\nabla R(\theta_t)\|_2)}\\
    &= \lim_{t \rightarrow \infty} \frac{\frac{1}{\|\nabla R(\theta_t)\|_2}\nabla Q(\|\nabla R(\theta_t)\|_2\bar{g})}{\|\frac{1}{\|\nabla R(\theta_t)\|_2}\nabla Q(\|\nabla R(\theta_t)\|_2 \bar{g})\|_2} \\
    &= \lim_{t \rightarrow \infty} \frac{\frac{Q^{\frac{1}{\alpha}-1}}{\|\nabla R(\theta_t)\|_2}\nabla Q(\|\nabla R(\theta_t))\|_2\bar{g})}{\|\frac{Q^{\frac{1}{\alpha}-1}}{\|\nabla R(\theta_t))\|_2}\nabla Q(\|\nabla R(\theta_t)\|_2 \bar{g})\|_2} \\
    &=\lim_{t \rightarrow \infty} \frac{\bar{\theta}}{\|\bar{\theta} \|_2} \\
    &= \bar{\theta} . 
\end{align*}
Here we used the definition of the horizon function and the fact that $\| \nabla R(\theta_t)\|_2 \rightarrow \infty$ as $\|\theta_t\| \rightarrow \infty$ by the coercivity property of the mirror map. 
\end{proof} 

\begin{proof}[Proof of Theorem~\ref{theorem : KKT max margin result}]
Combining Lemmas \ref{lemma : bound angle approx constraints} and \ref{lemma : limit point} implies Theorem \ref{theorem : appendix KKT max margin result} and with that concluding the result as a direct consequence. 
\end{proof}


\subsection{Hyperbolic entropy and homogeneous reparameterization}
\label{appendix : hyperbolic}

\addcontentsline{apx}{subsection}{\protect\numberline{\thesubsection}Hyperbolic entropy and homogeneous reparameterization}

The current proof strategy of the KKT conditions relies on homogeneity $\alpha \geq 2$. This we believe is a technicality of the analysis and not a real obstruction. To highlight this we consider the hyperbolic entropy.
We use the observation that the hyperbolic entropy $R_{\lambda}(\theta)$ corresponds to training with the two-homogeneous parameterization $u \odot v$ under gradient flow or differential inclusion as shown in \citep{jacobs2025mirror, Li2022ImplicitBO} for differentiable objectives. However the result only relies on the parameteric invariance hence the mirror flow-reparameterization correspondence also holds in our setting as also substantiated by \citep{marcotte2025transformativeconservativeconservationlaws, Marcotte2023AbideBT} and moreover it follows as a corollary from Lemma \ref{lemma : invariance balance}. This allows us to apply the result by \cite{Lyu2020Gradient} to characterize the max-margin. Together with the fact that for all $t\geq 0$ $u_t^2 - v^2_t = u_{in}^2 - v^2_{in} =\sqrt{\lambda}$ we can characterize the max margin in terms of $\theta = u \odot v$ giving the $L_1$-max margin. This follows directly from noticing that the normalized iterates $\bar{u}$ and $\bar{v}$ become balanced under the vector level constraint on $u$ and $v$: 
\begin{equation*}
    \bar{u}_t^2 - \bar{v}_t^2 = (u_t^2 - v^2_t)/(\|u_t,v_t \|_2^2) = (u_{in}^2 - v^2_{in})/(\|u_t,v_t \|_2^2  ) \rightarrow 0 . 
\end{equation*}
Therefore the max-margin solution has additional constraint $\bar{u}^2 = \bar{v}^2$.
Changing the objective from $\frac{1}{2}\|\bar{u},\bar{v} \|_{2}^2 =  \| \bar{\theta}\|_1$ where $\bar{\theta} : = \bar{u} \odot \bar{v}$. By the same argument we have that there exist a positive constant $b > 0 $ such that $\|u_t, v_t \|_2^2 / \|\theta_t\|_1 \rightarrow b$, thus $\bar{\theta} \simeq \frac{\theta}{\|\theta\|_1}$, which concludes the result. This implies that the normalized iterates $\theta_t/\| \theta_t\|_1$ (and therefore also $\theta_t/\| \theta_t\|_2$ ) corresponding to hyperbolic entropy mirror flow converge to the direction of the following optimization problem:
\begin{equation*}
   \min_{\theta \in \mathbb{R}^n} \|\theta\|_1\qquad \text{such that} \qquad y_i f(\theta, x_i) \geq 0.
\end{equation*}
This shows that other strategies may exist to show the KKT conditions for $1\leq \alpha \leq 2$. 

\paragraph{Relation between the reparameterization and hyperbolic entropy.} 

We provide some more details on the relation between the mirror flow and the reparameterization.
We want to study the trajectory of $\theta = u \odot v$, this according to the chain rule can be written as:
\begin{equation*}
    d \theta_t = - (u_t^2 + v_t^2) \nabla_{\theta} \mathcal{L}(\theta_t) dt.
\end{equation*}
We now can connect $u^2 + v^2$ to the metric tensor of the hyperbolic entropy, by using the invariance and the fact that $\theta = u \odot v$ we can solve a system of equations:
\begin{equation*}
\begin{cases}
    u^2 - v^2 = \sqrt{\lambda}\\
    u \odot v  = \theta . 
\end{cases}
\end{equation*}
This can be solved by applying the quadratic formula giving $u^2 + v^2 = \sqrt{4\theta^2 + \lambda}$. Note that the constant $4$ in front of the parameter can be dealt with by rescaling.


\subsection{KKT assumption}
\label{kkt assumption}

\addcontentsline{apx}{subsection}{\protect\numberline{\thesubsection}KKT assumption}  

To show the KKT conditions for the max margin problem \cite{tsilivis2024flavors, Lyu2020Gradient} rely on the 1-homogeneity of the derivative of the norm squared. 
We would have $\partial^{\circ}\frac{1}{2}\|\theta \|^2 = \| \theta\| \partial^{\circ}\|\theta \|$. This allows them to extract the normalization term $1/q_{\text{min}}^{1/L}$ and having the object $\| \theta\| \partial^{\circ}\|\theta \| /q_{\text{min}}^{1/L}$.
We can apply the same principle for the corresponding mirror descent objective in terms of the horizon function $\partial^{\circ} \frac{1}{2} \phi^2_{\alpha}(\theta)$. By construction we have that the horizon function is equivalent to an $L_p$ norm. 
This was formalized in Assumption \ref{assumption : approximation of horizon in norm}. 

\begin{lemma}
    For all $\lambda > 0$ and $\alpha = p \geq 2$ there exists a constant $C > 0$ such that for $\eta(\theta) \sim \|\theta\| \rightarrow \infty$
    \begin{equation*}
    \|C(\alpha Q(\nabla R(\theta)))^{2/\alpha-1} \nabla^2 R(\theta) \theta /\eta(\theta)  - \phi_{\alpha}(\theta) \partial^\circ \phi_{\alpha}(\theta)/\eta(\theta)  \|_2 \rightarrow 0 . 
\end{equation*}
\end{lemma}

\begin{proof} 

We have: 
\begin{equation*}
    (\alpha Q(\nabla R(\theta)))^{2/\alpha-1} \nabla^2 R(\theta) \theta = \frac{1}{((p-1)\|\theta\|_p^p + \frac{p}{2} \lambda \|\theta\|_2^2)^{1-2/p}}((p-1)|\theta|^{p-2} + \lambda)\theta
\end{equation*}
and 
\begin{equation*}
    \phi_{\alpha}(\theta) = (p-1)^{1/p} \|\theta\|_p , 
\end{equation*}
giving
\begin{equation*}
     \phi_{\alpha}(\theta) \partial^\circ \phi_{\alpha}(\theta) = (p-1)^{2/p}
     \| \theta \|_p^{2-p} |\theta|^{p-2}\theta.
\end{equation*}
Note that both terms constant normalization decay at rate $\|\theta \|_p^{2-p}$ as the term depending on $\lambda$ grows less fast. Moreover the terms $|\theta|^{p-2} \theta$ match upto a multiplicative constant.
Knowing this we can bound the difference:
\begin{align*}
    & \leq \|(p-1)^{2/p}|\theta|^{p-2}\theta \|_2^2 \left| \frac{C}{((p-1)\|\theta\|_p^p + \frac{p}{2} \lambda \|\theta \|_2^2)^{1-2/p}} - \|\theta \|_p^{2-p}\right|^2 /\eta^2(\theta) \\
    &\sim \|\theta\|^{2p-2} \|\theta\|^{8-4p} /\eta^2(\theta) \rightarrow 0
\end{align*}
iff $p > 2$, where we choose $C = (p-1)^{1-2/p}$ such that we can cancel the leading term in the expansion of:
\begin{equation*}
    \left| \frac{C}{((p-1) \|\theta \|_p^p + \frac{p}{2} \lambda \|\theta\|_2^2)^{1-2/p}} - \|\theta\|_p^{2-p}\right|^2 \sim \|\theta\|^{8-4p}. 
\end{equation*}
Note that if we would set $\lambda = 0$ we get exactly zero for this particular choice of $C$. Note that the case $p =2$ can be directly matched with $ \phi_{\alpha}(\theta) \partial^\circ \phi_{\alpha}(\theta) =\theta$ and $\alpha Q(\nabla R(\theta)))^{2/\alpha-1} \nabla^2 R(\theta) \theta = (1 + \lambda) \theta$, so $C = (1 + \lambda)$ yields the result.
\end{proof}

\paragraph{Hyperbolic entropy.} 
We now show that such an approximation can not hold or the hyperbolic entropy. We first have the expression:
\begin{align*}
    (\alpha Q(\nabla R(\theta)))^{2/\alpha-1} \nabla^2 R(\theta) \theta /q_{\text{min}}^{1/L} =  \left(\sum_i \sqrt{\theta_i^2 +\lambda}\right) \frac{\theta}{\sqrt{\theta^2 + \lambda}}/q_{\text{min}}^{1/L} . 
\end{align*}
Comparing this to the subgradient of $\frac{1}{2}\|\theta \|_1^2$ we get:
\begin{equation*}
   \|\theta \|_1\text{sign}(\theta)/q_{\text{min}}^{1/L} . 
\end{equation*}
These do not approach each other in the Euclidean distance as we can not bound
\begin{equation*}
    \| \frac{\theta}{\sqrt{\theta^2 + \lambda}} - \text{sign}(\theta) \| , 
\end{equation*}
which needs to go to zero for a good approximation. This is not possible in any metric.

\subsection{Two-layer optimization problem reformulation}\label{subsection : reform}

\addcontentsline{apx}{subsection}{\protect\numberline{\thesubsection}Two-layer optimization problem reformulation}

Here we provide the details of the two-layer neural network $f(a,w,x) : = \sum_{j = 1}^N a_j \ \sigma(w_j^T x)$ setting with ReLU activation $\sigma (\cdot )  := \max \{\cdot , 0 \}$. 

\begin{proof}[Proof of Theorem \ref{theorem : reform KKT main}]
    The result follows from the balance equation for the unnormalized iterates and the rescale invariance of the two layer neural network, i.e., for $c > 0$ we have $a \sigma (w^T x) = a/c \ \sigma(c w^Tx)$ by homogeneity of the ReLU activation. 

    The balance equation for the hyperbolic entropy gives us for each neuron $j \in [N]$ and $t\geq 0$:
    \begin{equation*}
        \sqrt{a_{j,t}^2 + \lambda} - \sum_{i =1}^d \sqrt{w_{j,i,t}^2 + \lambda}= \text{constant}
    \end{equation*}
    Now divide both sides by the norm $\|a_t,w_t\|_2$ which gives for $t \rightarrow \infty$:
    \begin{equation*}
        |\bar{a}_j| - \|\bar{w}_{j}\|_1 = 0.
    \end{equation*}
    Using this additional constrained,  the objective $\phi_1 \sim \|\theta\|_1$ where $\theta = (a,w) $ can be rewritten as:
    \begin{equation*}
        \|\theta\|_1 = 2 \| a \|_1 = 2\sum_{j=1}^N \sqrt{|a_j| \|w_j\|_1}.
    \end{equation*}
    We can use a change of variable $\tilde{a}_j = a_j \|w_j\|_1$ and $\tilde{w_j}/\|w_j \|_1$ and the rescale invariance to get the objective and constrained in Eq.~(\ref{equation : reformulation of KKT problem}). The proof for the homogenous potentials is analogous but with dividing the balance equation with $\|a_t, w_t\|_p^p$. 
\end{proof}

\newpage

\section{Margin alignment}\label{appendix : alignment}

\addcontentsline{apx}{section}{\protect\numberline{\thesection}Margin alignment between $Q$ and $\phi$}  

Here we show how fast reaching the margin depends on the hyperparameter $\lambda$.
We first provide a general result under Assumption \ref{assumption : bounded Q and horizon} and then apply it to our specific cases.

\begin{assumption}\label{assumption : bounded Q and horizon}
    Assume there exists independent constants $a > 0$ and $c \geq 1$ such that:
    \begin{equation*}
        \phi_{\alpha} \leq (\alpha Q)^{1/\alpha}\leq c \ \phi_{\alpha} + a . 
    \end{equation*}
\end{assumption}

\begin{lemma}\label{lemma : Q and horizon alignment}
    If $Q$ and $\phi_{\alpha}$ satisfy Assumption~\ref{assumption : bounded Q and horizon}, then after $\Omega(\exp (a^L))$ time the relative difference is $\mathcal{O}(1)$. 
\end{lemma}

\begin{proof} 
From Theorem \ref{theorem : convergence and growth rates} we know for a general $\alpha$ that $(\alpha Q)^{1/\alpha} = \Omega(\text{log}(t)^{1/L}) $. %
Under Assumption~\ref{assumption : bounded Q and horizon} we can bound the relative difference: 
\begin{equation*}
    0 \leq\frac{(\alpha Q)^{1/\alpha} - \phi_{\alpha}}{(\alpha Q)^{1/\alpha}} \leq (c-1) + a O(1/\text{log}(t)^{1/L})%
\end{equation*}
Hence $t \geq \exp (a^L)$ the difference is $\mathcal{O}(1)$. 
\end{proof}

\begin{corollary}(Hyperbolic entropy)\label{corollary : hyperbolic entropy reachability}
   For the hyperbolic entropy mirror flow, after $\Omega(\exp ((\sqrt{\lambda} n)^L))$ time the relative difference between $Q_{\lambda}(\nabla R_{\lambda}(\theta))$ and $\phi_1(\theta)$ is $\mathcal{O}(1)$.
\end{corollary}
\begin{proof} 
We can verify Assumption \ref{assumption : bounded Q and horizon} using the definition of $Q_{\lambda
}(\nabla R_{\lambda}(\theta)) = \sum_{i =1}^n \sqrt{\theta^2_i + \lambda}$ and $\phi_1(\theta) = \|\theta\|_1$:

\textbf{Lower bound}: We know that $\sqrt{\theta^2 +\lambda} \geq |\theta|$ for all $\theta \in \mathbb{R}$ and $\lambda \geq 0$ summing over all entries implies the lower bound holds.

\textbf{Upper bound}: For $\lambda \geq 0$ we can bound $\sqrt{\theta^2 + \lambda} \leq |\theta| + \sqrt{\lambda}$ summing now gives the upper bound with $a = n \sqrt{\lambda}$ and $c =1 $.

Applying Lemma \ref{lemma : Q and horizon alignment} concludes the result.
\end{proof} 

\begin{corollary}(Smoothened Homogeneous potentials)\label{corollary : smoothened homogeneous potentials reachability}
    For the smoothened homogeneous mirror flow with $p \geq 2$, after $\Omega(\exp (\left(\frac{p \lambda n}{2} \right)^{L/p}))$ time the relative difference between $Q_{\lambda}(\nabla R_{\lambda}(\theta))$ and $\phi_p(\theta)$ is $\mathcal{O}(1)$.
\end{corollary}

\begin{proof} 
We can verify Assumption \ref{assumption : bounded Q and horizon} using the definition of $Q_{\lambda}(\nabla R_{\lambda}(\theta)) = 1/q \| \theta\|_p^p + \frac{\lambda}{2} \|\theta \|_2^2 $ and $\phi_p(\theta) = (p-1)^{1/p}\| \theta\|_p$. Recall that $q \ge 1$ is such that $\frac{1}{q} + \frac{1}{p} =1$:

\textbf{Lower bound}: We can use that $p Q_{\lambda}(\nabla R_{\lambda}(\theta))  \geq pQ_{0}(\nabla R_{0}(\theta)) =(p-1)||\theta||_p^p $ taking the power $1/p$ on both sides gives the lower bound.

\textbf{Upper bound}: We can use that for all $z  \in \mathbb{R}$ we have that $z^2 \leq |z|^p + 1$. This gives the upper bound:
\begin{align*}
   p Q_{\lambda}(\nabla R_{\lambda}(\theta)) &= (p-1) \| \theta\|_p^p + \frac{p \lambda}{2} \|\theta\|_2^2 \\
   &\leq (p-1 + \frac{p\lambda}{2})\|\theta\|_p^p + \frac{p \lambda n}{2} 
\end{align*}
Now taking both sides to the power $1/p$ gives:
\begin{align*}
    (p Q_{\lambda}(\nabla R_{\lambda}(\theta)))^{1/p} &\leq \left((p-1 + \frac{p\lambda}{2})\|\theta\|_p^p + \frac{p \lambda n}{2}  \right)^{1/p} \\
    &\leq \left((p-1 + \frac{p\lambda}{2})/(p-1)\right)^{1/p} (1-p)^{1/p}\|\theta\|_p + \left(\frac{p \lambda n}{2} \right)^{1/p} \\
    &= c \phi_p (\theta) + a
\end{align*}
where $c = \left((p-1 + \frac{p\lambda}{2})/(p-1)\right)^{1/p} $ and $a = \left(\frac{p \lambda n}{2} \right)^{1/p}$, concluding the upper bound.

Applying Lemma~\ref{lemma : Q and horizon alignment} with the given $a$ and $c$ concludes the result. 
\end{proof} 

\begin{proof}[Proof of Lemma \ref{lemma : hyperbolic entropy alignment horizon} ]
    The result now follows from combining Corollaries \ref{corollary : hyperbolic entropy reachability} and \ref{corollary : smoothened homogeneous potentials reachability}.
\end{proof}

\section{NTK and modularity implications}
\label{section : additional implications} 
\addcontentsline{apx}{section}{\protect\numberline{\thesection}NTK and modularity implications}

Here we provide two additional implications of the main theorems. 
We recover the support vector machine with a kernel characterization as in \citep{Lyu2020Gradient}. Moreover, we can extend the max-margin result to the case where we use different mirror maps for each layer (i.e., modularity \citep{bernstein2025modular}). 

\paragraph{Neural tangent kernel.} 
The Neural Tangent Kernel (NTK) is a helpful tool to characterize the training dynamics of deep learning models. 

\begin{corollary}[Corollary of Theorem \ref{theorem : KKT max margin result}] 
\label{corollary : NTK mirror}
Assume $f \in C^2$-smooth on $\mathbb{R}^n \setminus \{0 \}$. 
Then for mirror flow under the same assumptions as Theorem \ref{theorem : KKT max margin result} or \ref{theorem : KKT max margin result hyperbolic}, 
any limit point $\bar{\theta}$ of $\{\frac{\theta_t}{\| \theta_t\|_2}, t \geq 0 \}$ 
points in the 
max-margin direction for the hard-margin SVM with kernel $K_{\bar{\theta}}(x,x') = \langle\nabla f_x(\bar{\theta}),\nabla f_{x'}(\bar{\theta})\rangle_2$, 
where $f_x(\theta) := f(\theta, x)$. 
That is, there exists some $\beta> 0$ such that $\beta \bar{\theta}$ is the optimal solution for the following constrained optimization problem for $\alpha \geq 2$ or $\alpha =1$: 
\begin{equation*}
        \min_{\theta \in \mathbb{R}^n} \, \frac{1}{2}\phi^2_{\alpha}(\theta) \qquad \text{such that} \qquad y_i \langle \theta, \nabla f_{x_i}  (\bar{\theta}) \rangle_2  \geq 1, \text{ for all } i \in [K].
    \end{equation*}
If we do not assume $f \in C^2$, for mirror flow, then there exists a mapping $h(x) \in \partial^{\circ} f_x(\bar{\theta})$
such that the same conclusion holds for $K_{\bar{\theta}}(x,x') = \langle h(x),h(x') \rangle_2$.
\end{corollary}

\begin{remark}
    Corollary \ref{corollary : NTK mirror} indicates that, in general, the max-margin characterization cannot be captured by a Reproducing Kernel Hilbert Space (RKHS), as 
    the underlying objective does not respect an inner product structure. 
    This highlights a fundamental distinction between mirror flow and gradient flow. 
    A natural framework for capturing the behavior of mirror flow is provided by 
    Reproducing Kernel Banach Spaces (RKBS) \citep{Heeringa2025DeepNA}.  
\end{remark}

By the homogeneity of $q_i$, we can characterize KKT points using kernel SVM.
\begin{lemma}\label{lemma : helper ntk mirror}
If $\theta^*$ is KKT point of the optimization problem in either Theorem \ref{theorem : KKT max margin result} or \ref{theorem : KKT max margin result hyperbolic}, then there exists $h_i \in \partial^{\circ} f_{x_i}(\theta^*)$ for $i \in [K]$ such that $\frac{1}{L}\theta^*$ is an optimal solution for the following constrained optimization problem:
\begin{equation*}
     \min_{\theta \in \mathbb{R}^n} \, \frac{1}{2}\phi^2_{\alpha}(\theta) \qquad \text{such that} \qquad y_i \langle \theta, h_i \rangle_2  \geq 1, \text{ for all } i \in [K].
\end{equation*}
\end{lemma}

\begin{proof} For $\theta = \frac{2}{L}\theta^*$, from homogeneity,
we can see that $y_i \langle \theta, h_i\rangle_2 = 2q_i(\theta^*) \geq 2 > 1$, which implies Slater’s condition so there is a feasible solution. Thus, we only
need to show that $\frac{1}{L} \theta^*$ satisfies KKT conditions for of the optimization problem. By the KKT conditions for the optimization problem in either Theorem \ref{theorem : KKT max margin result} or \ref{theorem : KKT max margin result hyperbolic}, we can construct $h_i \in \partial^{\circ} q_i(\theta^*)$ for $i \in [K]$ such that $\partial^{\circ} \frac{1}{2} \phi_{\alpha}^2(\theta^*) - \sum_{i = 1}^K \lambda_i y_i h_i = 0$ for some $\lambda_i \geq 0$ for $i \in [K]$ and $\lambda_i (q_i(\theta^*)-1) = 0$. Thus $\frac{1}{L} \theta^*$ satisfies by homogeneity of $\phi_{\alpha}$ and Eulers identity:
\begin{align*}
    \begin{cases}
        \frac{1}{L}\partial^{\circ} \frac{1}{2} \phi_{\alpha}^2(\theta) - \sum_{i = 1}^K \frac{1}{L} \lambda_i y_i h_i = 0 \\
        \frac{1}{L}\lambda_i ((\langle\frac{1}{L}\theta^*, h_i \rangle)-1) = \frac{1}{L}\lambda_i (q_i(\theta)-1) \geq 0
    \end{cases}
\end{align*}
So $\frac{1}{L}\theta^*$ satisfies the KKT conditions for the optimization problem. 
\end{proof}

Now we prove Corollary \ref{corollary : NTK mirror}. 
\begin{proof}[Proof of Corollary \ref{corollary : NTK mirror}]
The proof is analogous to \cite[Corollary 4.5]{Lyu2020Gradient}. %
By Theorem \ref{theorem : KKT max margin result} or \ref{theorem : KKT max margin result hyperbolic}, every limit point $\bar{\theta}$ is along the direction of a KKT point of the optimization problem. Combining
this with Lemma \ref{lemma : helper ntk mirror}, we know that every limit point $\bar{\theta}$ is also along the max-margin direction
of the new optimization problem.

For smooth models, $h_i$ is exactly the gradient $\nabla f_{x_i}(\bar{\theta})$ So, it becomes the optimization problem
for SVM with kernel $K_{\bar{\theta}}(x,x') = \langle \nabla f_{x}(\bar{\theta}),\nabla f_{x'}(\bar{\theta}) \rangle_2$. For non-smooth models, we can construct
an arbitrary function $h(x) \in \partial^{\circ}f_x(\bar{\theta})$ that ensures $h(x_i) = h_i$. Then, the optimization
problem for SVM with kernel $K_{\bar{\theta}}(x,x') =\langle h(x), h(x')\rangle_2$. 
\end{proof}

\paragraph{Modularity result.}
Motivated by the modular perspective on modern optimization \cite{bernstein2025modular, DBLP:journals/corr/abs-2405-14813}, 
we consider applying different mirror maps to different layers and aim to characterize the resulting max-margin solution. 
The proof follows the arguments of \cite[Lemma H.2]{Lyu2020Gradient}. 
We specify the notion of margin and the associated constrained optimization problem. 
Let the mirror map potential be layer-wise separable, i.e., 
$R(\theta) := \sum_{j = 1}^L R_j(W_j)$ 
with corresponding homogeneity degrees $\alpha_i \geq 1$, and assume $f$ is $(1,1, \hdots , 1)$-homogeneous, i.e., 
$f(c_1 W_1, c_2 W_2, \hdots, c_L W_L) = (\Pi_{j =1}^L c_j) f(W_1, \hdots, W_L)$.

\begin{definition}[Multi-$Q$-margin]
Let $f$ be $(1,\ldots,1)$-homogeneous and let $R$ be an asymptotically $\alpha$-positively homogeneous mirror map that is layer-wise separable, i.e.,
$R(\theta) := \sum_{j=1}^L R_j(W_j)$, 
with associated functions $Q_j$ and degrees $\alpha_j \geq 1$. 
The multi-$Q$-margin corresponding to $R$ is defined as
\begin{align*}
\gamma_Q 
&:= \min_i \, y_i \, f\!\left(
\frac{W_1}{\big(\alpha_1 Q_1(\nabla R_1(W_1))\big)^{1/\alpha_1}}, \,\ldots,\,
\frac{W_L}{\big(\alpha_L Q_L(\nabla R_L(W_L))\big)^{1/\alpha_L}}, \, x_i
\right) \\
&= \min_i \frac{y_i f(\theta, x_i)}{\prod_{j=1}^L \big(\alpha_j Q_j(\nabla R_j(W_j))\big)^{1/\alpha_j}}.
\end{align*}
\end{definition}

This then %
leads to the constrained optimization problem: 
\begin{equation*}
        \min_{\theta \in \mathbb{R}^n} \, \frac{1}{2}\sum_{j = 1}^L\phi^2_{\alpha_j}(W_j) \qquad \text{such that} \qquad y_i f(W_1, \hdots, W_L, x_i)  \geq 1, \text{ for all } i \in [K].
\end{equation*}
with directional vector $\left(W_1/\|W_1\|_2, \hdots, W_L/\|W_L\|_2\right)$.

\newpage

\section{Experimental details and ablations}
\label{appendix : experiments}

\addcontentsline{apx}{section}{\protect\numberline{\thesection}Experimental details and ablations}  

Here we provide the details of all the experiments in the main text and additional ablations. For the toy example the precision used is float64, to prevent underflow.

\subsection{Student-teacher two-layer neural network}
Here we provide the details and ablations on the student teacher setting. We randomly initialize a teacher with $3$ neurons and no biases. Then we train a $100$ neuron two layer student neural network also without biases.
We consider $3$ mirror maps hyperbolic, gradient descent, and smoothened homogeneous with $p =3$.
The data is sampled from the unit circle, $N = 200$ points. For teacher network $f_t(\theta)$, the neurons are randomly generated such that $|a_j| || w_j||_2 = 1$ for each $j \in [3]$.
The labels are then generated by the indicator $\mathbb{I}_{f_t > 0} - \mathbb{I}_{f_t \leq 0}$.
The students weights are initialized in the mean field regime.
The experiments are all executed on a NVIDIA GeForce RTX 4090 Laptop GPU.

\paragraph{Experiment main text.}
For the first set of experiments for the main text we use time rescaling i.e. when the loss is small enough ($\mathcal{L} <0.1$) and all training data points have been classified we set the learning rate $0.1\eta /\mathcal{L}$.
We use starting learning rates $\eta \in \{0.001, 0.1, 1 \}$, for hyperbolic, gradient descent and $L_3$ in that order. We train until the training loss is below $1e-50$. We repeat the experiment over $6$ seeds.
Here $\lambda = 0.1$ for the hyperbolic entropy and $\lambda =1$ for the smoothened homogeneous potential.
The max margin values are reported in Table \ref{table:max_margin_student_teacher_time_rescaled}. Every mirror map maximizes their margin value.
In the main text this lead to learning different representations which all exhibit feature learning (Figure \ref{fig: feature learning}).

\begin{table}[ht!]
\caption{Reported max margin values for the student teacher setup with time rescaling over $6$ seeds.}
\label{table:max_margin_student_teacher_time_rescaled}
\centering
\resizebox{\textwidth}{!}{%
\begin{tabular}{l| c c c}
\hline
 & $L_1$ & $L_2$ & $L_3$\\
\hline
Hyp. 
    & $\bm{3.31\times10^{-5} \pm 2.00\times10^{-6}}$ 
    & $2.76\times10^{-4} \pm 8.80\times10^{-8}$ 
    & $4.55\times10^{-4} \pm 3.50\times10^{-6}$ \\
GD 
    & $1.83\times10^{-6} \pm 1.02\times10^{-7}$ 
    & $\bm{2.94\times10^{-4} \pm 4.42\times10^{-8}}$ 
    & $1.34\times10^{-3} \pm 4.59\times10^{-5}$ \\
$p=3$ 
    & $1.19\times10^{-6} \pm 1.27\times10^{-8}$ 
    & $2.91\times10^{-4} \pm 7.77\times10^{-7}$ 
    & $\bm{1.76\times10^{-3} \pm 5.18\times10^{-6}}$ \\
\hline
\end{tabular}%
}
\end{table}

\paragraph{The influence of $\lambda$.}
We now do not rescale time but consider the same fixed learning rate. We now want to investigate the role of $\lambda$ in both the hyperbolic and smoothened homogeneous potentials. For the hyperbolic entropy we consider $\lambda \in \{1, 0.1, 0.01\}$ and for the smoothened homogeneous we consider $\lambda \in \{ 10,1, 0.1\}$.
We train for $T =10000$ epochs and repeat for $6$ seeds.
We both report the final margin values and the representation reached. 

We illustrate this in Figures \ref{fig: circles hyperbolic plot} and \ref{fig: circles homogenous plot} for seed $42$.
The main observation is that feature learning still occurs but the representation found has more smeared out neurons as in the case for gradient descent. 
In Tables \ref{table:margin_l_values} and \ref{table:margin_l_values_L2_L3} we report the max margin values reached. Indeed the corresponding margin to each mirror map becomes smaller when using a larger $\lambda$. 

\paragraph{Reaching the margin.} 
Next we empirically verify that for different values of $\lambda > 0$, the time required to reach the corresponding max-margin solution can vary substantially. 
To this end, we train for a long but fixed duration without time rescaling. 
The results in Figure~\ref{fig: circles homogenous plot} show that, for large $\lambda$, the returned representation remains similar to that of gradient descent, with neurons more widely spread. 
This behavior persists even when the training time is increased tenfold, as shown in Figure~\ref{fig: ten times longer}. 
These findings indicate that, to obtain representations that differ from those of gradient descent, one must either train for extremely long times or appropriately tune the hyperparameter $\lambda$.

\begin{table}[ht!]
\caption{Reported margin values for different $\lambda$ values for the hyperbolic entropy.}
\label{table:margin_l_values}
\centering
\begin{tabular}{l| c c}
\hline
$\lambda$ & $L_1$ & $L_2$ \\
\hline
$1$ 
& $9.37\times10^{-6} \pm 1.14\times10^{-6}$ 
& $\bm{2.48\times10^{-4} \pm 1.54\times10^{-5}}$ \\

$0.1$ 
& $1.21\times10^{-5} \pm 7.51\times10^{-7}$ 
& $2.29\times10^{-4} \pm 2.09\times10^{-5}$ \\

$0.01$ 
& $\bm{1.55\times10^{-5} \pm 1.23\times10^{-6}}$ 
& $2.20\times10^{-4} \pm 2.58\times10^{-5}$ \\
\hline
\end{tabular}
\end{table}

\begin{table}[ht!]
\caption{Reported margin values for different $\lambda$ values for the smoothened homogeneous potential $p =3$.}
\label{table:margin_l_values_L2_L3}
\centering
\begin{tabular}{l| c c}
\hline
$\lambda$ & $L_2$ & $L_3$ \\
\hline
$10$
& $2.24\times10^{-4} \pm 3.54\times10^{-6}$ 
& $1.21\times10^{-3} \pm 1.83\times10^{-5}$ \\

$1$ 
& $2.32\times10^{-4} \pm 2.53\times10^{-6}$ 
& $1.43\times10^{-3} \pm 1.35\times10^{-5}$ \\

$0.1$ 
& $\bm{2.34\times10^{-4} \pm 2.43\times10^{-6}}$ 
& $\bm{1.43\times10^{-3} \pm 1.46\times10^{-5}}$ \\
\hline
\end{tabular}
\end{table}

\begin{figure}
    \centering
    \includegraphics[width=0.95\linewidth]{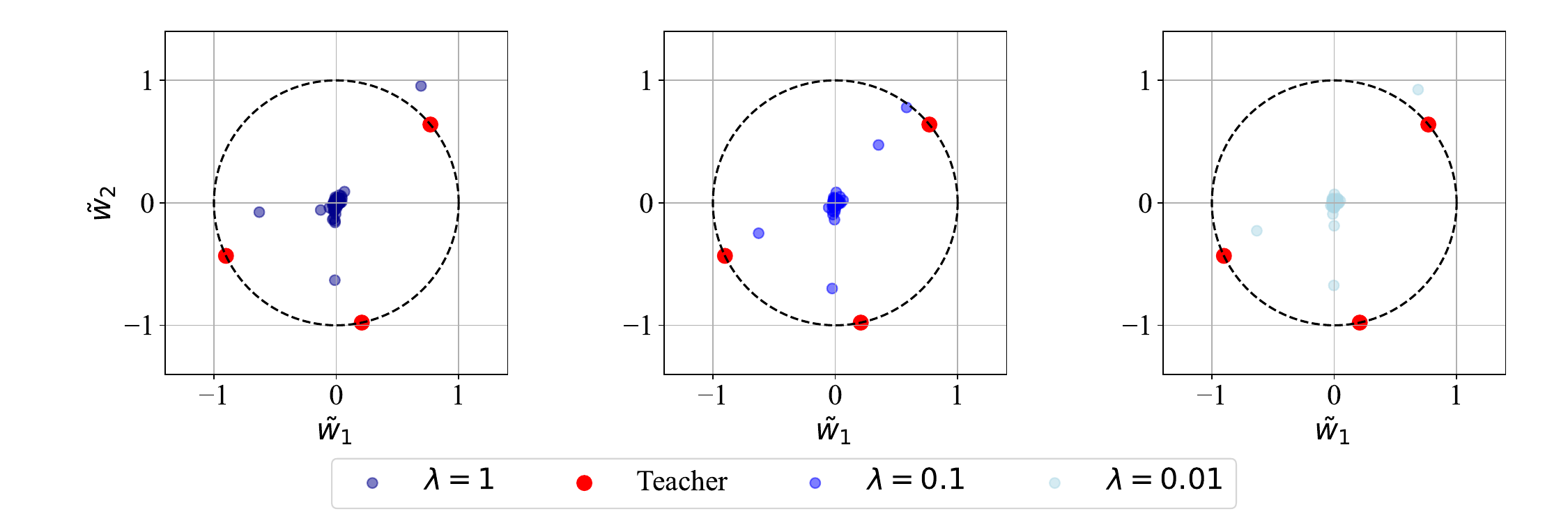}
    \caption{Learned representation by mirror descent with hyperbolic entropy for various $\lambda$. Larger $\lambda$ leads to a different representation that looks more like gradient descent solution with more neurons spread out.}
    \label{fig: circles hyperbolic plot}
\end{figure}

\begin{figure}
    \centering
    \includegraphics[width=0.95\linewidth]{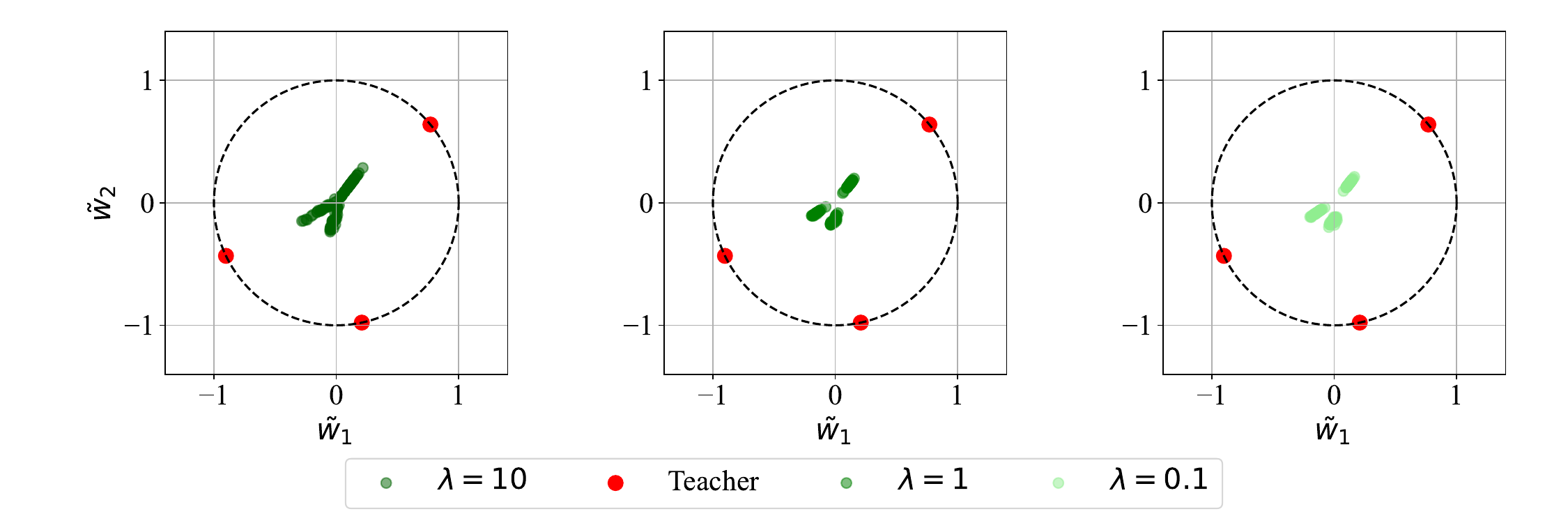}
    \caption{Learned representation by mirror descent with smoothened homogeneous ($p =3$) for various $\lambda$. Larger $\lambda$ leads to a different representation that looks more like gradient descent solution with more neurons spread out.}
    \label{fig: circles homogenous plot}
\end{figure}

\paragraph{Long training.}
Here we investigate the exponential separation for $\lambda =10$ for $T \in \{1000, 10000, 100000 \}$.
We observe that margin grows but the representation does not change indicating it is really hard to reach the time rescaled solution or small $\lambda$ solution when starting with large $\lambda$. 

\begin{table}[ht!]
\caption{Reported $L_3$ margin values for different epoch times $T$ for the smoothened homogeneous mirror.}
\label{table:margin_time_scales_L3}
\centering
\begin{tabular}{l| c}
\hline
$T$ & $L_3$ \\
\hline
$100000$  & $\bm{1.39\times10^{-3} \pm 1.84\times10^{-5}}$ \\
$10000$   & $1.21\times10^{-3} \pm 1.83\times10^{-5}$ \\
$1000$    & $8.03\times10^{-4} \pm 2.35\times10^{-5}$ \\
\hline
\end{tabular}
\end{table}

\begin{figure}
    \centering
    \includegraphics[width=0.95\linewidth]{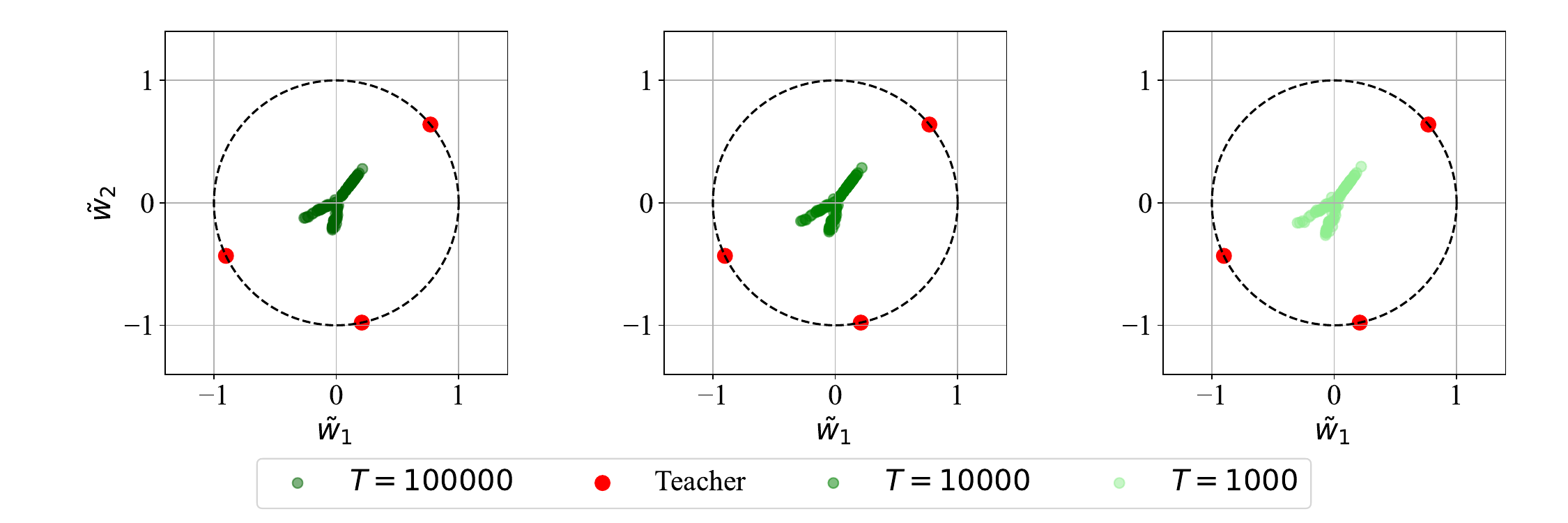}
    \caption{Training ten times longer does not change the representation much when $\lambda = 10$ is large.}
    \label{fig: ten times longer}
\end{figure}

\subsection{Three-layer neural network student-teacher setup}
We conduct an additional experiment on a $3$ layer student network with the same $2$ layer teacher as in the previous section.
The $3$ layer neural network has hidden dimension $100$.
We again train with the hyperbolic mirror map, gradient descent, and smoothened homogeneous $(p =3)$. The students weights are initialized in the mean field regime.

\paragraph{Time rescaling.} 
To reach closer to the margin we again employ time rescaling.
In order to ensure stable time rescaling we start rescaling when the loss is below $0.1$ and then we set the learning rate to $\eta \ 0.1/\mathcal{L}$ instead.
Moreover, we set the learning rates to $\eta \in \{ 0.001, 0.01, 0.1\}$ for the hyperbolic, GD, smoothened homogeneous mirror maps in that order.
For the hyperbolic mirror map we use $\lambda =0.1$ and for the smoothened homogeneous we use $\lambda =1$.
We report the weight distributions of all $3$ layers in Figures \ref{fig: multi first layer}, \ref{fig: multi second layer}, and \ref{fig: multi third layer}.
We observe in all layers that the weight magnitude distribution has shifted accordingly to the corresponding KKT problem.
To illustrate the effect in function space we also plot the input-output map highlighting the decision boundary in Figure \ref{fig: input output map 3 layer}.
Moreover, we report the final margin value found in Table \ref{table : margins 3 layer time rescale}.
We observe that the margin values in case of $L_1$ and $L_3$ have different orders of magnitude, which is in line with the change in representation.

\begin{table}[ht!]
\caption{Max margins reached by the $3$ mirror descent algorithms after time rescaling for the $3$ layer student and $2$ layer teacher setup.}\label{table : margins 3 layer time rescale}
\centering
\begin{tabular}{l| c c c}
\hline
 & $L_1$ & $L_2$ & $L_3$ \\
\hline
Hyp. & 
$\bm{(4.01 \pm 0.34)\times10^{-8}}$ & 
$(6.21 \pm 0.72)\times10^{-5}$ & 
$(1.03 \pm 0.26)\times10^{-4}$ \\

GD & 
$(1.13 \pm 0.07)\times10^{-9}$ & 
$\bm{(1.10 \pm 0.00)\times10^{-4}}$ & 
$(1.73 \pm 0.07)\times10^{-3}$ \\

Hom. & 
$(2.39 \pm 0.08)\times10^{-10}$ & 
$(9.36 \pm 0.04)\times10^{-5}$ & 
$\bm{(3.42 \pm 0.06)\times10^{-3}}$ \\
\hline
\end{tabular}
\end{table}

\begin{figure}
    \centering
    \includegraphics[width=0.95\linewidth]{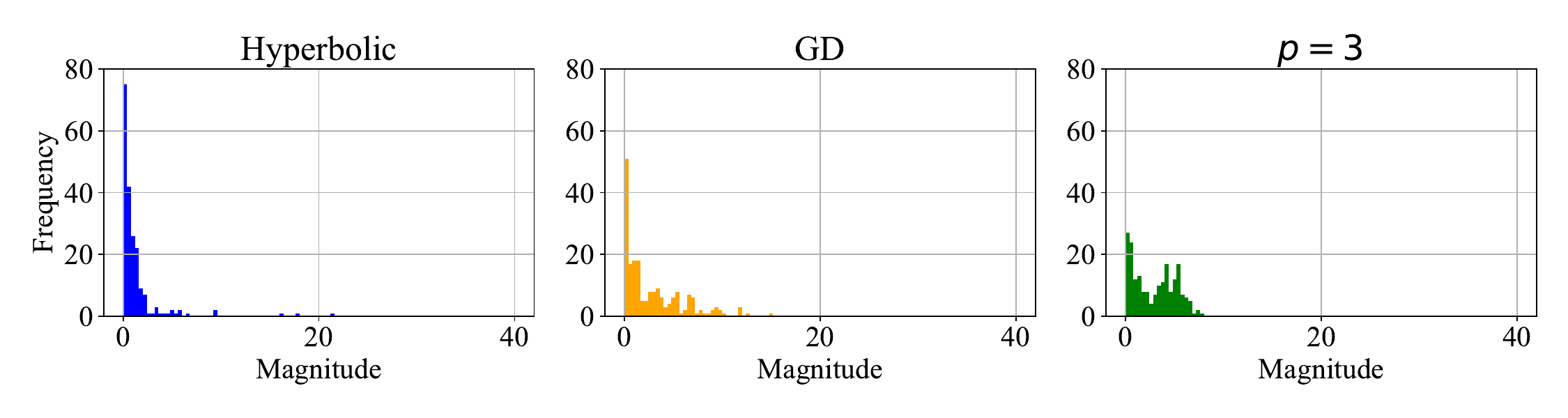}
    \caption{Weight magnitude distribution of the first layer after time rescaling for all mirror maps.}
    \label{fig: multi first layer}
\end{figure}

\begin{figure}
    \centering
    \includegraphics[width=0.95\linewidth]{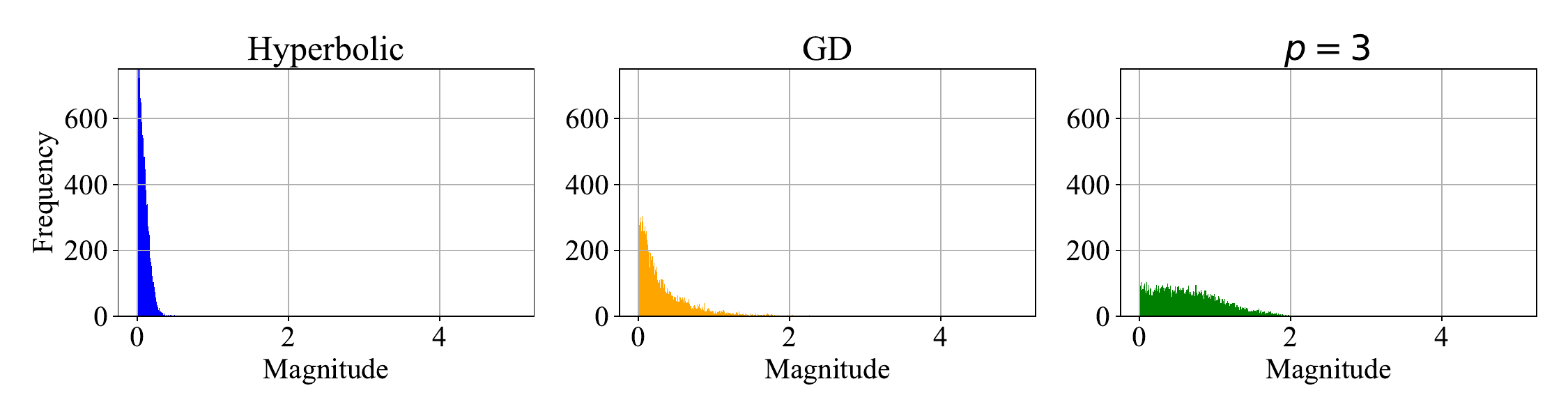}
    \caption{Weight magnitude distribution of the second layer after time rescaling for all mirror maps.}
    \label{fig: multi second layer}
\end{figure}

\begin{figure}
    \centering
    \includegraphics[width=0.95\linewidth]{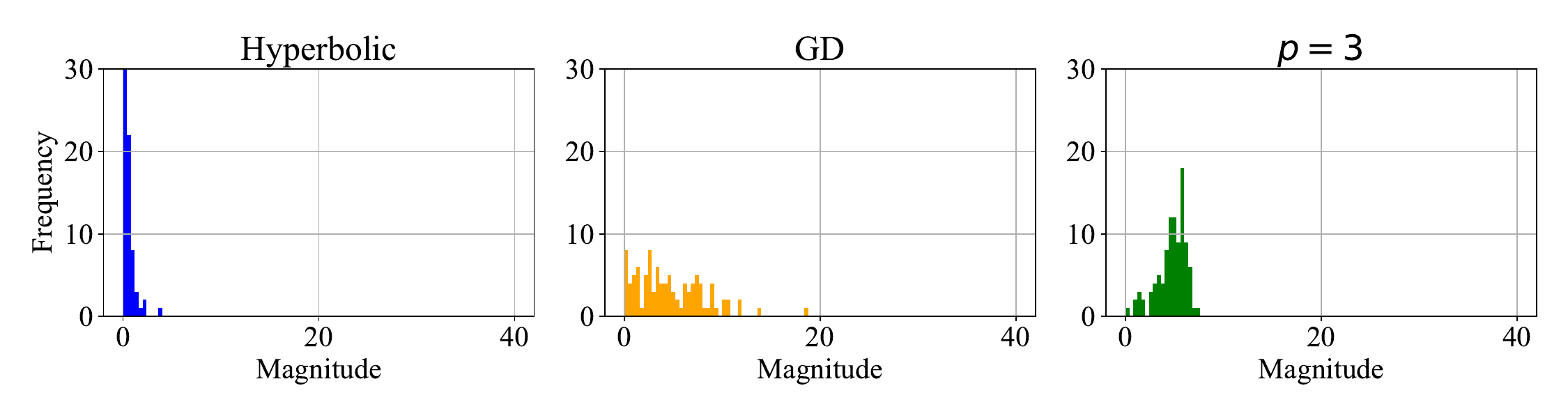}
    \caption{Weight magnitude distribution of the third layer after time rescaling for all mirror maps.}
    \label{fig: multi third layer}
\end{figure}

\begin{figure}
    \centering
    \includegraphics[width=0.95\linewidth]{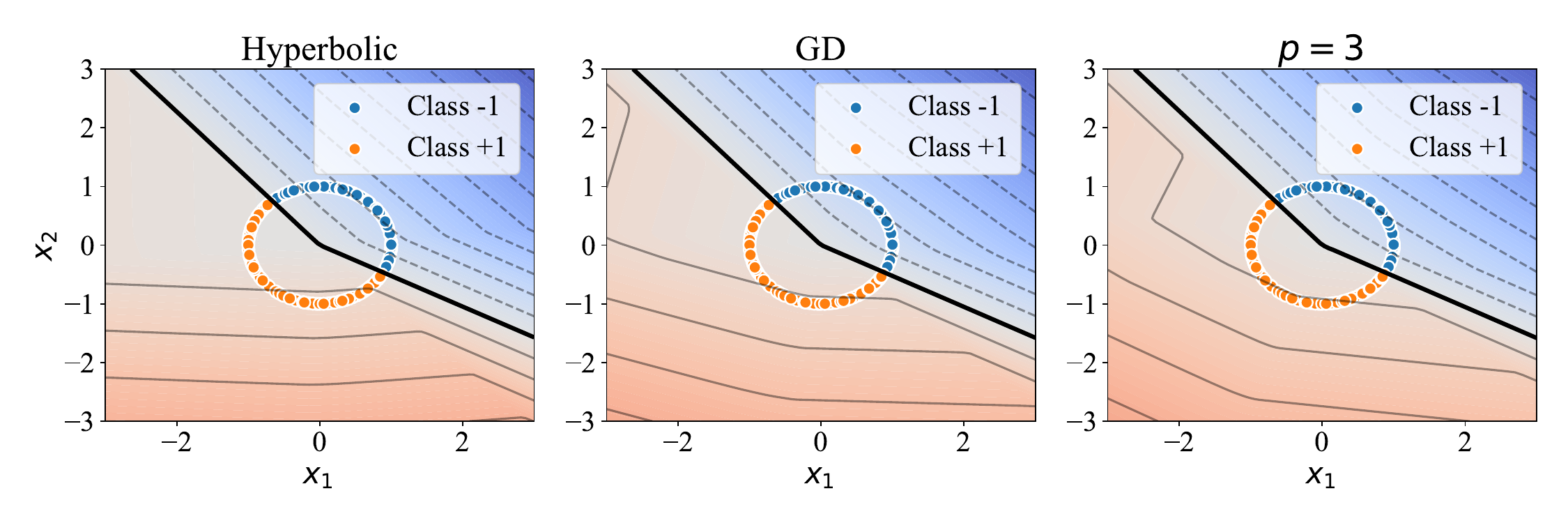}
    \caption{Illustration of the decision boundary and function activation value reached for the $3$ layer neural network experiment.}
    \label{fig: input output map 3 layer}
\end{figure}

\paragraph{Additional experiments.} 
We conduct the same $2$ experiments as in case of the $2$ layer student: influence of $\lambda$ in both the case of hyperbolic entropy and smoothened homogeneous. In the first experiment, we use $\eta = 0.01$ and $T =10000$ for the hyperbolic entropy. We report the margins reached in Tables~\ref{table : 3 layer hyperbolic lambda} and \ref{table:lambda_smooth_homogeneous_p3}.  Moreover, we report the changed input representation in Figures~\ref{fig: circles hyperbolic plot 3 layer} and~\ref{fig: circles L3 plot 3 layer}.

\begin{table}[ht!]
\caption{Max margins for different $\lambda$ values with the hyperbolic entropy mirror in the $3$ layer student setting.}\label{table : 3 layer hyperbolic lambda}
\centering
\begin{tabular}{l| c c}
\hline
 & $L_1$ & $L_2$ \\
\hline
$\lambda=1$ & 
$(1.67 \pm 0.24)\times10^{-9}$ & 
$\bm{(6.43 \pm 0.51)\times10^{-5}}$ \\

$\lambda=0.1$ & 
$(3.62 \pm 0.30)\times10^{-9}$ & 
$(4.76 \pm 0.60)\times10^{-5}$ \\

$\lambda=0.01$ & 
$\bm{(4.87 \pm 0.23)\times10^{-9}}$ & 
$(4.44 \pm 0.76)\times10^{-5}$ \\
\hline
\end{tabular}
\end{table}

\begin{figure}
    \centering
    \includegraphics[width=0.95\linewidth]{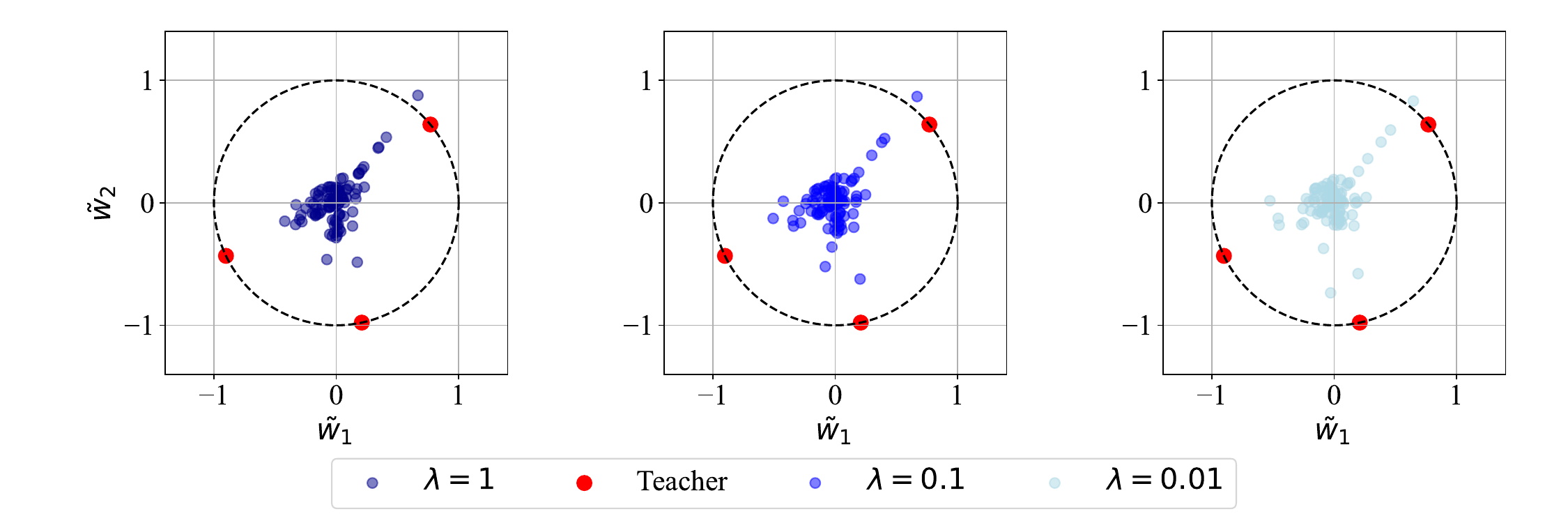}
    \caption{Learned representation in the input layer $w$ by mirror descent with hyperbolic entropy for various $\lambda$ in the $3$ layer student setting. Larger $\lambda$ leads to a different representation that looks more like gradient descent solution with more neurons spread out.}
    \label{fig: circles hyperbolic plot 3 layer}
\end{figure}

\begin{table}[ht!]
\caption{Max margins for different $\lambda$ values with the smoothened homogeneous mirror ($p=3$).}
\label{table:lambda_smooth_homogeneous_p3}
\centering
\begin{tabular}{l| c c}
\hline
 & $L_2$ & $L_3$ \\
\hline
$\lambda=10$ & 
$\bm{(8.76 \pm 0.09)\times10^{-5}}$ & 
$(1.75 \pm 0.04)\times10^{-3}$ \\

$\lambda=1$ & 
$(8.69 \pm 0.06)\times10^{-5}$ & 
$(2.62 \pm 0.03)\times10^{-3}$ \\

$\lambda=0.1$ & 
$(7.47 \pm 0.04)\times10^{-5}$ & 
$\bm{(3.44 \pm 0.03)\times10^{-3}}$ \\
\hline
\end{tabular}
\end{table}

\begin{figure}
    \centering
    \includegraphics[width=0.95\linewidth]{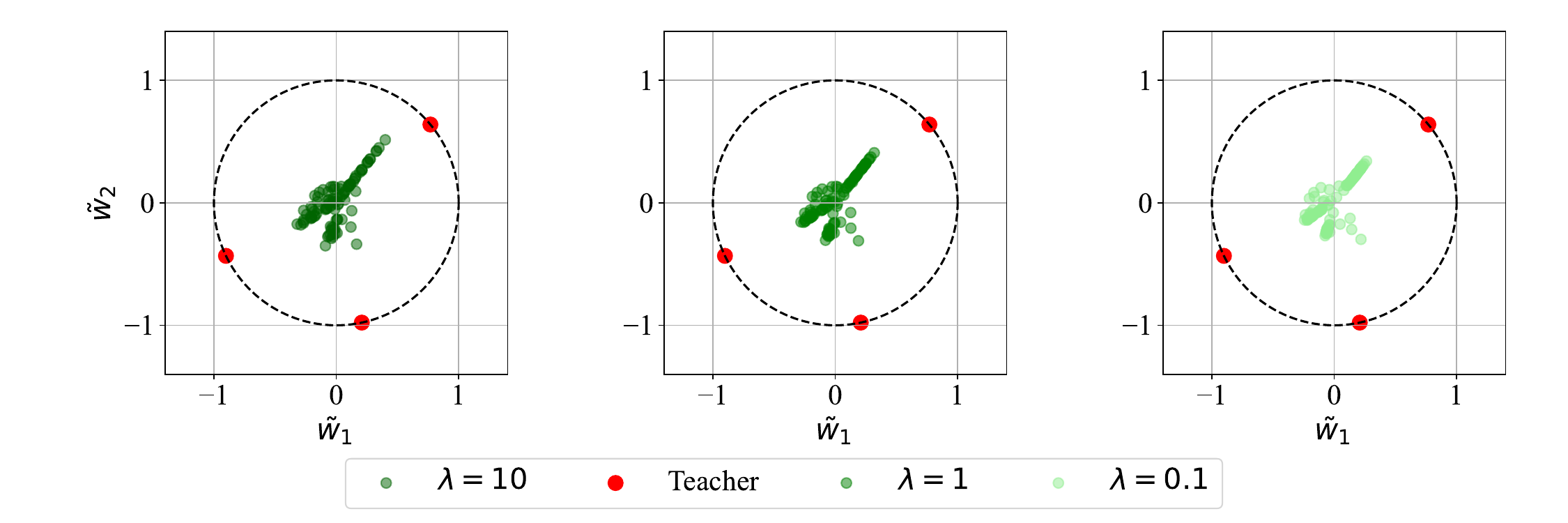}
    \caption{Learned representation in the input layer $w$ by mirror descent with smoothened homogeneous ($p =3$) for various $\lambda$ in the $3$ layer student setting. Larger $\lambda$ leads to a different representation that looks more like gradient descent solution with more neurons spread out.}
    \label{fig: circles L3 plot 3 layer}
\end{figure}

\subsection{Classification on vision tasks}
We run mirror descent with stochastic gradient estimates on single A100 GPU. We do not use weight decay or momentum. The inverse metric tensor is created in the following way to ensure comparability and layerwise stability:
\begin{equation*}
   \nabla^2R_{\text{hyp}}^{-1}(\theta) = \sqrt{\theta^2 * \text{width} + \lambda}  /(\sqrt{1+\lambda}) \text{ and } \nabla^2R_{\text{hom}}^{-1}(\theta) = \frac{\lambda}{|\theta|^{p-2} *\text{width}^{(2-p)/2} + \lambda}.
\end{equation*}
This ensures the average update for each element in the same layer gets a similar update. Note we do not do learning rescaling with the loss or multiplicative coefficient as for the toy example.
For tracking the margin we use the layerwise product of the norms and the logits similar as in \citep{Lyu2020Gradient}.
We use He initialization for the weights and zero initialization for the biases of the first layer. In other words we are using the standard parameterization (SP).
Note that this setting is not in the feature learning regime, therefore, we can not expect all layers to move equally as much. For the experiment we sweep the hyper parameters $\lambda \in \{1e-05, 0.01, 0.1, 1\}$ and learning rate $\eta \in \{0.1, 0.2\}$. We use batch-size $100$, we do not use weight decay or momentum. 
We train a VGG16 \citep{Simonyan2014VeryDC} without biases (except the first layer) for $1000$ epochs. The first layer having biases does not violate homogeneity as it can be seen as augmented data input. The numerical values are over $3$ seeds, the illustrations are for one seed: $42$. Furthermore, we note that training either GD, or Smoothened homogeneous becomes unstable and diverges when trained with learning rate $0.2$.

We report the best accuracy values for each method in Table \ref{tab:results}. We observe that the hyperbolic entropy leads to an improvement over the other methods. Note that the standard baseline accuracy is higher due to the use of momentum, batch norm and weight decay.
Besides the first layer magnitude plot in the main text we plot a histogram of the last layer in Figure \ref{fig: hist last layer}. We observe that the weight distributions are very similar. This is due to the standard parameterization (SP) \citep{pmlr-v139-yang21c}. This highlights the importance of knowing in which regime we are training, irrespective of validation accuracy performance, which representation we end up with. Note that in \citep{Sun2023AUA, jacobs2026hyperbolic} the mirror maps lead to very different representations and with that different magnitude distributions per layer in modern architectures using residual connections and batch/layer norm. Moreover, it also leads to improved performance in these more general settings notably for using a mirror map associated with the hyperbolic entropy \citep{jacobs2026hyperbolic} and the homogeneous mirror map with ($p =3$) \citep{Sun2023AUA}.

\begin{figure}
    \centering
    \includegraphics[width=0.5\linewidth]{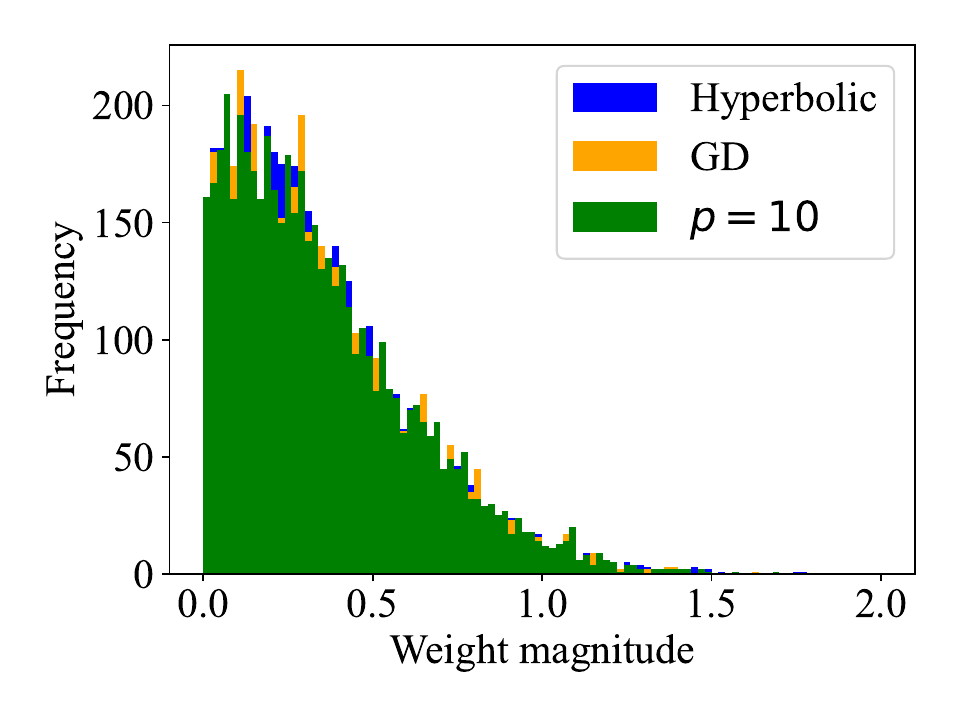}
    \caption{The weight magnitude distribution of a VGG16's last layer. We observe that all algorithms get a similar distribution due to the standard parameterization.}
    \label{fig: hist last layer}
\end{figure}

\begin{table}
\caption{Validation accuracy for a VGG16 on CIFAR 10 without using batch normalization, weight decay or momentum.}
\label{tab:results}
\centering
\begin{tabular}{lc}
\hline
Mirror & Val. Acc. \\
\hline
GD         & $81.26 \pm 0.15$ \\
$p = 10$   & $81.26 \pm 0.12$ \\
Hyperbolic & $81.96 \pm 0.41$ \\
\hline
\end{tabular}
\end{table}

\subsection{First step analysis}  
We furthermore conduct a first step analysis.
In order for the parameter norm to grow and reach the max-margin direction we need the parameters to move from initialization. 
For gradient descent this has been studied under the NTK and mean-field regimes. 
Under the mean-field regime we have that $\nabla_i \mathcal{L} = O(1)$ for all $i \in [L]$. 
This means that the weights of all layers are moving. 
Now for mirror flow we have an additional contribution trough the metric tensor. 
To account for this, we can study the random variable transform $\nabla^2_i R^{-1}(\theta_0)$ with the standard initialization $\theta_0 \sim N(0, I_n \sigma^2)$. If we treat this pointwise and together with the learning rate we would need to have
\begin{equation*}
    \mathbb{E}[\eta\nabla^2_i R^{-1}(\theta_0)] = O(1) 
\end{equation*}
to keep the maximal update like in gradient descent. We will refer to this as mirror-$\mu$P, which is needed to preserve the update magnitude.
We now know what mirror descent has to satisfy in order to give a maximal parameter update at initialization and reach the margin. 
\begin{table}[ht!]
\caption{Training recipes for mirror descent with learning rate $\eta$ and hyperparameter $\lambda$, denoted as $(\eta, \lambda)$.}
\label{tab:mirror_properties}
\centering
\begin{tabular}{lccc}
\hline
Mirror potential & Mirror-$\mu$P & Margin reachable \\
\hline
GD & $(1,-)$ &  $(1,-)$ \\
 Hyp & $(1,1), (\sigma^{-1},0),  (\sigma^{-1},\sigma^{2})$ &  $(\sigma^{-1},0),(\sigma^{-1},\sigma^{4})$  \\
$L_{p,\lambda}^p$, $p > 2$ & $(1,1), (\sigma^{p-2}, \sigma^{p-2})$    & $ (\sigma^{2}, \sigma^{2})$ \\
\hline
\end{tabular} 
\end{table}

This allows us to summarize the hyperparameter for designing training recipes in Table~\ref{tab:mirror_properties}.
Observe that how Mirror-$\mu$P is decoupled from reaching a max-margin solution, implying that we can have large parameter changes at the start of training without eventually reaching the corresponding margin solution. 
\begin{example}
    Consider the hyperbolic entropy, which has $2$ mirror-$\mu$P solutions. 
    The expectation is given by $\mathbb{E}\left[\eta\sqrt{\theta_0^2 + \lambda}\right] = O(\eta(\sigma + \sqrt{\lambda})) $, so $\lambda = 1, \eta = 1$ is a solution for decaying $\sigma$. Moreover, $\eta = \sigma^{-1}, \lambda = \sigma^2$ is a solution for all $n \in \mathbb{N}^+$. 
\end{example}

\begin{remark}
    The random variables corresponding to the metric tensor $\nabla^2 R^{-1}(\theta_0)$ and gradient $\nabla f(\theta_0)$ are coupled. However we can justify using free probability theory to treat them as decoupled.
\end{remark}

\newpage

\end{document}